\RequirePackage{fix-cm}
\documentclass[smallextended]{svjour3}      
\smartqed  
\usepackage[utf8]{inputenc}
\usepackage{algorithmicx}
\usepackage{algpseudocode}
\usepackage{parskip}
\usepackage{amsmath,amssymb,amstext,amsfonts}
\usepackage{float}
\usepackage{graphicx}
\usepackage{xcolor}
\usepackage{comment}
\usepackage{mathrsfs}

\usepackage{mathtools}
\DeclarePairedDelimiter\ceil{\lceil}{\rceil}
\DeclarePairedDelimiter\floor{\lfloor}{\rfloor}

\newcommand{\W}{\mathrm{W}}
\newcommand{\OT}{\mathrm{OT}}

\newcommand{\ADM}{\mathrm{Joint}}
\newcommand{\Tr}{\mathrm{Tr}}
\newcommand{\E}{\mathbb{E}}
\newcommand{\Probs}{\mathcal{P}}
\newcommand{\N}{\mathcal{N}}

\newcommand{\KL}{{\mathrm{KL}}}
\newcommand{\trace}{\mathrm{Tr}}
\newcommand{\approach}{\ensuremath{\rightarrow}}

\renewcommand{\H}{\mathcal{H}}
\newcommand{\Sym}{\mathrm{Sym}}
\newcommand{\HS}{\mathrm{HS}}
\newcommand{\tr}{\mathrm{tr}}
\newcommand{\Xbf}{\mathbf{X}}
\newcommand{\Ybf}{\mathbf{Y}}
\newcommand{\Ncal}{\mathcal{N}}
\newcommand{\1}{\mathbf{1}}
\newcommand{\X}{\mathcal{X}}
\newcommand{\R}{\mathbb{R}}
\newcommand{\la}{\langle}
\newcommand{\ra}{\rangle}
\newcommand{\mapto}{\ensuremath{\rightarrow}}
\renewcommand{\b}{\mathbf{b}}

\newcommand{ \ep}{\epsilon}

\newcommand{\bP}{\mathbb{P}}

\newcommand{\restr}[1]{\lower3pt\hbox{$|_{#1}$}}

\newcommand{\MMD}{\mathrm{MMD}}

\newcommand{\logdet}{\mathrm{logdet}}

\newcommand{\Gauss}{\mathrm{Gauss}}

\newcommand{\bE}{\mathbb{E}}

\newcommand{\equivalent}{\ensuremath{\Longleftrightarrow}}
\newcommand{\imply}{\ensuremath{\Rightarrow}}

\newcommand{\Lcal}{\mathcal{L}}

\newcommand{\Pcal}{\mathcal{P}}
\newcommand{\Srm}{\mathrm{S}}
\newcommand{\Nbb}{\mathbb{N}}

\newcommand{\Csc}{\mathscr{C}}

\newcommand{\myspan}{\mathrm{span}}

\newcommand{\Cbf}{\mathbf{C}}

\title{
Convergence and finite sample approximations of 
entropic regularized Wasserstein distances in Gaussian and RKHS settings}
\titlerunning{
	Regularized Wasserstein distance in Gaussian and RKHS settings}
\author{H\`a Quang Minh}
\institute{
	\at
	RIKEN Center for Advanced Intelligence Project, 1-4-1 Nihonbashi, Chuo-ku,
	Tokyo 103-0027, JAPAN
	\\
	\email{minh.haquang@riken.jp}
}
\date{\today}

\begin{document}

\maketitle

\begin{abstract}
This work studies the convergence and finite sample approximations of entropic regularized Wasserstein distances 
in the Hilbert space setting. Our first main result is that for Gaussian measures
on an infinite-dimensional Hilbert space, convergence in the 2-Sinkhorn divergence is {\it strictly weaker} than 
convergence in the exact 2-Wasserstein distance. Specifically, a sequence of centered Gaussian measures
converges in the 2-Sinkhorn divergence if the corresponding covariance operators converge in the Hilbert-Schmidt norm.
This is in contrast to the previous known result that a sequence of centered Gaussian measures converges in the exact 2-Wasserstein distance if and only if the covariance operators converge in the trace class norm.
In the reproducing kernel Hilbert space (RKHS) setting, the {\it kernel Gaussian-Sinkhorn divergence}, which is the Sinkhorn divergence between Gaussian measures defined on an RKHS, defines a semi-metric on the set of Borel probability measures on a Polish space, given a characteristic kernel on that space.
With the Hilbert-Schmidt norm convergence,
we obtain {\it dimension-independent} convergence rates for 
finite sample approximations of the kernel Gaussian-Sinkhorn divergence, with the same order as the Maximum Mean Discrepancy.
These convergence rates apply in particular to Sinkhorn divergence between Gaussian measures on Euclidean and infinite-dimensional Hilbert spaces. The sample complexity for the 2-Wasserstein distance between Gaussian measures on Euclidean space, while dimension-dependent and larger than that of the Sinkhorn divergence, is exponentially faster than the worst case scenario in the literature.
\end{abstract}

\section{Introduction}

This work studies the entropic regularization formulation of the $2$-Wasserstein distance in the infinite-dimensional Hilbert space setting, with a focus on convergence properties
and finite sample approximations of the entropic Wasserstein distances/divergences. It  
builds upon the
previous work \cite{Minh2020:EntropicHilbert}, which formulated these distances/divergences in the infinite-dimensional Gaussian setting, along with the corresponding barycenter problems.

Entropic regularization in optimal transport has recently attracted much attention in various research fields, including 
in particular machine learning and statistics
\cite{cuturi13,Sommerfeld2017WassersteinDO,feydy18,genevay16,genevay17,GigTam18,MalMonGer19,ramdas2017,RipThesis,mena2019samplecomplexityEntropicOT}. It has found applications fields ranging from computer vision to density functional theory and inverse problems (e.g.~\cite{genevay17,GerGroGor19,Lunz18,patrini18}). This line of research is also closely connected
with the {\it Schr\"odinger bridge problem}~\cite{Schr31}, which has been studied extensively
~\cite{BorLewNus94,Csi75,peyre17,FraLor89,RusIPFP,Zam15,galsal,LeoSurvey,rus93,rus98}.

In \cite{Mallasto2020entropyregularized,Janati2020entropicOT,barrio2020entropic}, the authors
studied the entropic regularized 2-Wasserstein distance for Gaussian measures on Euclidean space, 
providing explicit formulas for the entropic Wasserstein distance and Sinkhorn divergence, along with
the fixed point equations for the corresponding barycenter problems, and many other properties.
These studies exploit in particular the Maximum Entropy property of Gaussian densities in 
$\R^n$ and the connection between the entropic regularization framework with the {Schr\"odinger bridge problem}~\cite{Schr31}.

The above results were subsequently generalized to the setting of infinite-dimensional  
Gaussian measures and covariance operators on Hilbert spaces in \cite{Minh2020:EntropicHilbert},
where the entropic formulation is shown to be valid for both settings of {\it singular} and {\it nonsingular} covariance operators. The infinite-dimensional Gaussian setting reveals several important properties of the entropic regularization formulation, including
(i) the failed generalization of the finite-dimensional entropic Wasserstein distance barycenter equation; (ii)
the uniqueness of solution of the Sinkhorn barycenter equation, in contrast to both the 
entropic and exact 2-Wasserstein barycenter equations; (iii) the Fr\'echet differentiability of
the entropic Wasserstein distance and Sinkhorn divergence, in contrast to the exact $2$-Wasserstein distance, which is 
{\it not} Fr\'echet differentiable.
 
In the current work, we study the convergence and finite sample approximations of the entropic regularized Wasserstein distances/divergences, with a particular focus on the Sinkhorn divergence. It is well-known that the exact Wasserstein distances metrize the weak convergence of probability measures on Polish spaces (see e.g. \cite{villani2016}). In \cite{feydy18}, it is shown that the Sinkhorn divergence metrizes weak convergence
of probability measures on compact metric spaces. In this work, we show that
for infinite-dimensional Hilbert space Gaussian measures, the Sinkhorn divergence gives {\it weaker} convergence
than the exact Wasserstein distances, thus allowing a {\it broader } set of converging sequences.

Another well-known but undesirable property of the Wasserstein distances is that their sample complexity
can grow exponentially in the dimension of the underlying space $\R^d$, with the worst case being $O(n^{-1/d})$ \cite{dudley1969speed}, see also \cite{weed17,fournier2015rate,horowitz1994mean}. 
In \cite{genevay18sample}, the authors show that, as a consequence of the entropic regularization,
the Sinkhorn divergence between two probability measures with {\it bounded support} on $\R^d$ achieves sample complexity
$O((1+\ep^{-\floor{d/2}})n^{-1/2})$, that is the same as the Maximum Mean Discrepancy (MMD) for a fixed $\ep>0$.
However, 
the constant factor in 
the sample complexity in \cite{genevay18sample} depends exponentially on the diameter of the support.
In \cite{mena2019samplecomplexityEntropicOT}, the rate of convergence 
$O\left(\ep\left(1+\frac{\sigma^{\ceil{5d/2} + 6}}{\ep^{\ceil{5d/4} + 3}}\right)n^{-1/2}\right)$
was obtained for $\sigma^2$-subgaussian measures on $\R^d$.
In this work, we show that the Sinkhorn divergence in the RKHS setting achieves the rate of convergence
$O\left((1+\frac{1}{\ep})n^{-1/2}\right)$ for all $\ep > 0$, which is thus {\it dimension-independent}.
In particular, this applies to Sinkhorn divergence between Gaussian measures on Euclidean space as well as on an infinite-dimensional Hilbert space.

{\bf Contributions of this work}
\begin{enumerate}
	\item We show that on an infinite-dimensional Hilbert space,
	 a sequence of centered Gaussian measures
	converges in the Sinkhorn divergence if the corresponding covariance operators converge in the {\it Hilbert-Schmidt norm}.
    This is in contrast to the previous known result \cite{Bogachev:Gaussian,masarotto2019procrustes} that a sequence of centered Gaussian measures converges in the exact Wasserstein distance if and only if the covariance operators converge in the {\it trace class norm}.
    Thus convergence in the Sinkhorn divergence is {\it strictly weaker} than 
	convergence in the Wasserstein distance, i.e.
	convergence in the Wasserstein distance automatically implies convergence in the Sinkhorn divergence but the converse is false. This phenomenon is a distinctive feature of the infinite-dimensional setting,
	since for Gaussian measures on Euclidean space, convergences in exact Wasserstein distance and Sinkhorn divergence are equivalent.
	
	\item 
	In the  reproducing kernel Hilbert space (RKHS) setting, the {\it kernel Gaussian-Sinkhorn divergence},
	which is the Sinkhorn divergence between Gaussian measures defined on an RKHS,
	 is an interpolation between
	the MMD and kernel Wasserstein distance and admits closed forms via kernel Gram matrices.
	 For characteristic kernels on a Polish space $\X$,
	the RKHS Sinkhorn divergence defines a {\it semi-metric} on the set of Borel probability measures on $\X$.
	With the convergence in the Hilbert-Schmidt norm, we obtain {\it dimension-independent} convergence rates for 
	finite sample approximations of the kernel Gaussian-Sinkhorn divergence, of $O((1+\frac{1}{\ep})n^{-1/2})$.
	This is done via the application of the laws of large numbers for random variables with values in the Hilbert space of Hilbert-Schmidt operators. With the linear kernel, we then obtain sample complexity bounds for the Sinkhorn divergence
	between Gaussian measures on any separable Hilbert space, of finite or infinite dimension, with 
	$\R^d$ being a special case.
	
	\item As a consequence of the analysis in the RKHS setting, we obtain sample complexity for
	the exact $2$-Wasserstein distance between two Gaussian measures on $\R^d$,
	of the form $O\left((\frac{d}{n})^{1/4}\right)$. This is {\it exponentially faster}  than
	the worst case scenario $O(n^{-1/d})$ in \cite{dudley1969speed,weed2019sharp}.
   This analysis does {\it not}, however, extend to an infinite-dimensional Hilbert space. 
\end{enumerate}

\section{Background and previous work}
\label{section:background}

Let $(X,d)$ be a Polish (complete separable metric) space equipped with a lower semi-continuous \emph{cost function} $c:X\times X \to \mathbb{R}_{\geq 0}$. 
Let $\Pcal(X)$ denote the set of all probability measures on $X$.
The {\it optimal transport} (OT) problem between two probability measures $\nu_0, \nu_1 \in \Probs(X)$ is  
(see e.g. \cite{villani2016})
\begin{equation}
\OT(\nu_0, \nu_1) = \min_{\gamma\in \ADM(\nu_0,\nu_1)}\E_\gamma[c] = \min_{\gamma \in \ADM(\nu_0, \nu_1)}\int_{X \times X}c(x,y)d\gamma(x,y)
\label{equation:OT-exact}
\end{equation}
where $\ADM(\nu_0,\nu_1)$ is the set of joint probabilities with marginals $\nu_0$ and $\nu_1$.
%
For $1 \leq p < \infty$, let $\Pcal_p(X)$ denote the set of all probability measures $\mu$ on $X$ of finite moment of order $p$,
i.e. 
$\int_{X}d^p(x_0,x)d\mu(x) < \infty$ for some (and hence any) $x_0 \in X$.
The {\it $p$-Wasserstein distance} $W_p$ between $\nu_0$ and $\nu_1$ is defined as
\begin{equation}
W_p(\nu_0,\nu_1) = \OT_{d^p}(\nu_0, \nu_1)^{\frac{1}{p}}.
\end{equation}
This distance defines a metric on $\Pcal_p(X)$ (Theorem 7.3, \cite{villani2016}),
which metrizes the {\it weak convergence} on $\Pcal_p(X)$ (Theorem 7.12, \cite{villani2016}).
For two Gaussian measures on a separable Hilbert space 
$\nu_i=\N(m_i,C_i)$, $i=0,1$, 
$W_2(\nu_0, \nu_1)$ admits the following closed form \cite{Gelbrich:1990Wasserstein,cuesta1996:WassersteinHilbert} 
\begin{equation}
\label{equation:Gaussian-Wass-finite}
W_2^2(\nu_0, \nu_1) = \|m_0-m_1\|^2 + \Tr(C_0) + \Tr(C_1) - 2 \Tr\left(C_0^\frac{1}{2} C_1 C_0^\frac{1}{2}\right)^\frac{1}{2}.
\end{equation}
For $\H=\R^n$, this formula was obtained in \cite{givens84,dowson82,olkin82,knott84}.

{\bf Entropic regularization and Sinkhorn divergence.}
The exact OT problem \eqref{equation:OT-exact} is often computationally challenging and it is more numerically efficient to solve the following regularized optimization problem, for a given $\ep > 0$, 
\begin{equation}
\label{equation:OT-entropic}
\OT_c^\epsilon(\mu, \nu) = \min_{\gamma\in \ADM(\mu,\nu)}\left\lbrace\E_\gamma[c]
+ \epsilon \KL(\gamma || \mu \otimes \nu) \right\rbrace,
\end{equation}
where $\KL(\nu || \mu)$ denotes the Kullback-Leibler divergence between $\nu$ and $\mu$.
The KL
in \eqref{equation:OT-entropic} acts as a bias \cite{feydy18}, with the consequence that in general $\OT_c^\epsilon(\mu, \mu) \neq 0$. The following $p$-Sinkhorn divergence \cite{genevay17,feydy18} removes this bias
\begin{equation}
\label{equation:sinkhorn}
S_{d^p}^\epsilon(\mu, \nu) = \OT_{d^p}^\epsilon(\mu, \nu) - \frac{1}{2}(\OT_{d^p}^\epsilon(\mu,\mu) + \OT_{d^p}^\epsilon(\nu,\nu) ).
\end{equation}
In the case $X=\H$ is a separable Hilbert space and $\mu,\nu$ are Gaussian measures on $\H$,
both $\OT^{\ep}_{d^2}$ and $\Srm^{\ep}_{d^2}$ admit closed form expressions, as follows.
\begin{theorem}
	[\textbf{Entropic 2-Wasserstein distance and Sinkhorn divergence between Gaussian measures on Hilbert space},
	\cite{Minh2020:EntropicHilbert}, Theorems 3, 4, and 7]
	\label{theorem:OT-regularized-Gaussian}
	Let $\H$ be a separable Hilbert space and
	$\mu_0 = \Ncal(m_0, C_0)$, $\mu_1 = \Ncal(m_1, C_1)$ be two Gaussian measures on $\H$. For each fixed $\ep > 0$,
	\begin{align}
	\OT^{\ep}_{d^2}(\mu_0, \mu_1) &= ||m_0-m_1||^2 + \trace(C_0) + \trace(C_1) - \frac{\ep}{2}\trace(M^{\ep}_{01})
	\nonumber
	\\
	& \quad +\frac{\ep}{2}\log\det\left(I + \frac{1}{2}M^{\ep}_{01}\right).
	\label{equation:gauss-entropic-2-Wasserstein-infinite}
	\end{align}
	\begin{equation}
	\begin{aligned}
	\Srm^{\ep}_{d^2}(\mu_0, \mu_1) &= ||m_0 - m_1||^2 + \frac{\ep}{4}\trace\left[M^{\ep}_{00} - 2M^{\ep}_{01} + M^{\ep}_{11}\right] 
	\\
	& \quad + \frac{\ep}{4}\log\det\left[\frac{\left(I + \frac{1}{2}M^{\ep}_{01}\right)^2}{\left(I + \frac{1}{2}M^{\ep}_{00}\right)\left(I + \frac{1}{2}M^{\ep}_{11}\right)}\right].
	\end{aligned}
	\label{equation:gauss-sinkhorn-infinite}
	\end{equation}
	The optimal joint measure $\gamma^{\ep}$ is unique and is the Gaussian measure	
	\begin{align}
	\label{equation:minimizing-Gaussian-measure}
	\gamma^{\ep} &= \Ncal\left(\begin{pmatrix} m_0 \\ m_1 \end{pmatrix},
	\begin{pmatrix} C_0 & C_{XY}
	\\
	C_{XY}^{*} & C_1\end{pmatrix}
	\right),
	\\
	\text{where  } C_{XY} &= \frac{2}{\ep}C_0^{1/2}\left(I+\frac{1}{2}M^{\ep}_{01}\right)^{-1}
	C_0^{1/2}C_1.
	\end{align}
	Here $\det$ is the Fredholm determinant and
	$M^{\ep}_{ij}: \H \mapto \H$, 
	are defined by
	\begin{align}
	\label{equation:M-ep}
	M^{\ep}_{ij} = -I + \left(I + \frac{16}{\epsilon^2}C_i^{1/2}C_jC_i^{1/2}\right)^{1/2}, \;\;\; i,j=0,1.
	\end{align}
\end{theorem}
In particular, when $\ep \approach 0$ and $\ep \approach \infty$,
\begin{align}
\lim_{\ep \approach 0}\OT^{\ep}_{d^2}(\mu_0, \mu_1) &= \lim_{\ep \approach 0}\Srm^{\ep}_{d^2}(\mu_0,\mu_1) = \W_2(\mu_0,\mu_1),
\\
\lim_{\ep \approach \infty}\Srm^{\ep}_{d^2}(\mu_0,\mu_1) &= ||m_0 - m_1||^2.
\end{align}
When $\dim(\H) < \infty$, we recover the finite-dimensional results in \cite{Mallasto2020entropyregularized,Janati2020entropicOT,barrio2020entropic}. 

{\bf Properties of the Sinkhorn divergence $\Srm^{\ep}_{d^2}$}.
\begin{enumerate}
 
\item $\Srm^{\ep}_{d^2}$ is a divergence function on $\Gauss(\H)$, the set of all Gaussian measures on $\H$, while $\OT^{\ep}_{d^2}$ is neither a distance nor a divergence, since generally $\OT^{\ep}_{d^2}(\mu, \mu) \neq 0$, as noted before.
\item  $\Srm^{\ep}_{d^2}(\mu_0, \mu_1)$ is Fr\'echet differentiable in each argument, whereas
the Wasserstein distance $\W_2(\mu_0, \mu_1)$ is {\it not } Fr\'echet differentiable when $\dim(\H) = \infty$.

\item $\Srm^{\ep}_{d^2}(\mu_0, \mu_1)$ is {\it strictly convex} in each argument and the corresponding barycenter equation always 
has a unique solution for any $\ep > 0$ fixed, valid in both singular and nonsingular settings.

\item In this work, we study the {\it convergence properties} of Gaussian measures in
Hilbert space with respect to the Sinkhorn divergence $\Srm^{\ep}_{d^2}$ and the {\it finite-dimensional
approximations} of $\Srm^{\ep}_{d^2}(\mu_0, \mu_1)$.
\end{enumerate}

{\bf Notation}.
Throughout the following, let $\H$ be a real, separable Hilbert space, with $\dim(\H) = \infty$ unless explicitly stated otherwise.
Let $\Lcal(\H)$ denote the set of bounded linear operators on $\H$, 
with norm $||A|| = \sup_{||x||\leq 1}||Ax||$.
%
Let $\Sym(\H) \subset \Lcal(\H)$ be the set of bounded, self-adjoint linear operators on $\H$. Let $\Sym^{+}(\H) \subset \Sym(\H)$ be the set of
self-adjoint, {\it positive} operators on $\H$, i.e. $A \in \Sym^{+}(\H) \equivalent \la Ax,x\ra \geq 0 \forall x \in \H$. 
Let $\Sym^{++}(\H)\subset \Sym^{+}(\H)$ be the set of self-adjoint, {\it strictly positive} operator on $\H$,
i.e $A \in \Sym^{++}(\H) \equivalent \la x, Ax\ra > 0$ $\forall x\in \H, x \neq 0$.
We write $A \geq 0$ for $A \in \Sym^{+}(\H)$ and $A > 0$ for $A \in \Sym^{++}(\H)$.
The Banach space $\Tr(\H)$  of trace class operators on $\H$ is defined by (see e.g. \cite{ReedSimon:Functional})
$\Tr(\H) = \{A \in \Lcal(\H): ||A||_{\tr} = \sum_{k=1}^{\infty}\la e_k, (A^{*}A)^{1/2}e_k\ra < \infty\}$,
for any orthonormal basis $\{e_k\}_{k \in \Nbb} \in \H$.
For $A \in \Tr(\H)$, its trace is defined by $\trace(A) = \sum_{k=1}^{\infty}\la e_k, Ae_k\ra$, which is independent of choice of $\{e_k\}_{k\in \Nbb}$. 
The Hilbert space $\HS(\H)$ of Hilbert-Schmidt operators on $\H$ is defined by 
$\HS(\H) = \{A \in \Lcal(\H):||A||^2_{\HS} = \trace(A^{*}A) =\sum_{k=1}^{\infty}||Ae_k||^2 < \infty\}$,
for any orthonormal basis $\{e_k\}_{k \in \Nbb}$ in $\H$,
with inner product $\la A,B\ra_{\HS}=\trace(A^{*}B)$. 
We have $\trace(\H) \subsetneq \HS(\H) \subsetneq \Lcal(\H)$ when $\dim(\H) = \infty$, with $||A||\leq ||A||_{\HS}\leq ||A||_{\tr}$.

\section{Convergence of entropic regularized distances}
\label{section:main}

Let $(X,d)$ be a Polish space (i.e. a complete, separable metric space). It is a well-established fact (see e.g. \cite{villani2016}) that the Wasserstein distance
$\W_p$ metrizes the {\it weak convergence} in $\Pcal_p(X)$. We recall that (see e.g. \cite{Bogachev:Gaussian}, Definition 3.8.1) a sequence $\{\mu_n\}_{n \in \Nbb}$ of Radon measures on a topological space $X$ is said to be {\it weakly convergent}
to a Radon measure $\mu$ if
\begin{align}
\lim_{n \approach \infty}\int_{X}f(x)d\mu_n(x) = \int_{X}f(x)d\mu(x), \;\;\;\forall f \in C_b(X),
\end{align}
where $C_b(X)$ is the set of all bounded, continuous functions on $X$. By Theorem 7.12 in \cite{villani2016},
\begin{align}
\label{equation:convergence-general-p-Wass}
\lim_{n \approach \infty}\W_p(\mu_n, \mu) = 0 \equivalent
\left\{
\begin{matrix}
\mu_n \xrightarrow[n \approach \infty]{}\mu \;\;\;\text{weakly},
\\
\int_{X}d(x_0,x)^pd\mu_n(x) \xrightarrow[n \approach \infty]{}\int_{X}d(x_0,x)d\mu(x) < \infty,
\end{matrix}
\right.
\end{align}
with the second limit holding for some (and hence any) $x_0 \in X$.
For Gaussian measures on a Hilbert space $X = \H$, the convergence
of $\W_2(m_n, \mu)$, where $\mu_n = \Ncal(m_n, C_n)$, $\mu = \Ncal(m, C)$, is completely characterized
by the convergence of the corresponding means and covariance operators. Specifically,
(\cite{Bogachev:Gaussian}, Examples 3.8.13 and 3.8.15, see also \cite{masarotto2019procrustes}, Proposition 4),
\begin{align}
\label{equation:convergence-exact-2-Wass}
\lim_{n \approach \infty}\W_2(\mu_n, \mu) = 0 &\equivalent 
\left\{
\begin{matrix}
\lim_{n \approach \infty}||m_n - m|| = 0,
\\
\lim_{n \approach \infty}||C_n - C||_{\tr} = 0,
\end{matrix}
\right.
\nonumber
\\
&\equivalent \mu_n \xrightarrow[n \approach \infty]{}\mu \;\;\text{weakly}.
\end{align}
In particular, for {\it centered} Gaussian measures, i.e. $\mu_n = \Ncal(0, C_n)$, $\mu = \Ncal(0,C)$,
\begin{align}
\label{equation:convergence-centered-Gaussian}
\lim_{n \approach \infty}\W_2(\mu_n, \mu) = 0 \equivalent \lim_{n \approach \infty}||C_n - C||_{\tr} = 0.
\end{align}
The convergence in the trace norm implies in particular the convergence of the trace, which 
is precisely convergence of the second moment, that is
\begin{align}
\label{equation:convergence-2nd-moment}
\lim_{n \approach \infty}\left[\trace(C_n) = \int_{\H}||x||^2d\mu_n(x)\right] = \trace(C) = \int_{\H}||x||^2d\mu(x).
\end{align}
The following shows quantitatively the continuity of $\W_2$ in trace norm.
\begin{proposition}
	\label{proposition:2Wasserstein-gaussian-upperbound}
	Let $\mu_1 = \Ncal(m_1, C_1)$, $\mu_2 = \Ncal(m_2, C_2)$ on $\H$. Then
	\begin{align}
	\W_2^2(\mu_1, \mu_2) &= ||m_1 - m_2||^2 + \trace(C_1) + \trace(C_2) - 2 \trace[(C_1^{1/2}C_2C_1^{1/2})^{1/2}]
	\\
	& \leq ||m_1-m_2||^2 + ||C_1^{1/2} - C_2^{1/2}||^2_{\HS}
	\label{equation:Wasserstein-upperbound-HS}
	\\
	& \leq ||m_1 - m_2||^2 + ||C_1 - C_2||_{\tr}.
	\label{equation:Wasserstein-upperbound-trace}
	\end{align}
	In particular, for $\H = \R^d$, $d \in \Nbb$,
	\begin{align}
	\W_2^2(\rho_1, \rho_2)& \leq ||m_1 - m_2||^2 + \sqrt{d}||C_1 - C_2||_{\HS}.
	\end{align}
	In \eqref{equation:Wasserstein-upperbound-HS}, equality occurs if $C_1$ and $C_2$ commute.
	In \eqref{equation:Wasserstein-upperbound-trace}, equality occurs if e.g. $C_1 = 0$ or $C_2 = 0$.
\end{proposition}

{\bf Convergence in Sinkhorn divergence}. In \cite{feydy18}, it was proved that convergence
in Sinkhorn divergence is equivalent to weak convergence if (i) $X$ is a compact metric space and 
$c(x,y)$ is a Lipschitz cost function; or (ii) $X = \R^D$, $c(x,y) = ||x-y||^p$, $p=1,2$, 
and the measures involved have {\it bounded support}.

Our first main result shows that for $X = \H$, with $\dim(\H) = \infty$, convergence of 
a sequence of Gaussian measures in Sinkhorn divergence is {\it strictly weaker} than convergence
in the exact $2$-Wasserstein distance.

\begin{theorem}
	[\textbf{Convergence in Sinkhorn divergence}]
	\label{theorem:convergence-Sinkhorn}
	Let $\{m_N\}_{N \in \Nbb}, m \in \H$, $\{A_N\}_{N\in \Nbb}, A \in \Sym^{+}(\H) \cap \Tr(\H)$. Then
	\begin{align}
	\Srm^{\ep}_{d^2}[\Ncal(0, A_N), \Ncal(0, A)] &\leq \frac{3}{\ep}[||A_N||_{\HS} + ||A||_{\HS}]||A_N-A||_{\HS}.
	\\
	\Srm^{\ep}_{d^2}[\Ncal(m_N, A_N), \Ncal(m, A)] &\leq ||m_N - m||^2 
	\nonumber
	\\
	&\quad + \frac{3}{\ep}[||A_N||_{\HS} + ||A||_{\HS}]||A_N-A||_{\HS}.
	\end{align}
	In particular,
	\begin{align}
	\lim_{N \approach \infty}||A_N - A||_{\HS} = 0 \imply \lim_{N \approach \infty}\Srm^{\ep}_{d^2}[\Ncal(0, A_N), \Ncal(0, A)] = 0.
	\\
	\left\{\begin{matrix} \lim_{N \approach \infty}||m_N - m|| = 0
	\\
	\lim_{N \approach \infty}||A_N - A||_{\HS} = 0
	\end{matrix}
	\right. \imply \lim_{N \approach \infty}\Srm^{\ep}_{d^2}[\Ncal(m_N, A_N), \Ncal(m, A)] = 0.
	\end{align}
\end{theorem}

\begin{theorem}
	[\textbf{Convergence in Sinkhorn vs. exact Wasserstein}]
	\label{theorem:Sinkhorn-vs-Wasserstein}
	For Gaussian measures on $\H$, $\dim(\H) = \infty$, convergence in Sinkhorn divergence $\Srm^{\ep}_{d^2}$ is {\bf strictly weaker}
	than convergence in $2$-Wasserstein distance. 
	On the one hand, $\forall \{A_N\}_{N \in \Nbb}, A \in \Sym^{+}(\H)\cap \Tr(\H)$, $\forall \{m_N\}_{N \in \Nbb}, m \in \H$,
	\begin{align}
	\lim_{N \approach \infty}\W_2[\Ncal(m_N,A_N), \Ncal(m,A)] = 0 \imply \lim_{N \approach \infty}\Srm^{\ep}_{d^2}[\Ncal(m_N,A_N), \Ncal(m,A)] = 0.
	\end{align}
	On the other hand,
	$\exists \{A_N\}_{N \in \Nbb} , A\in \Sym^{+}(\H) \cap \Tr(\H)$ such that
	\begin{align}
	\lim_{N \approach \infty}||A_N - A||_{\HS} &= 0, \;\; \lim_{N \approach \infty}||A_N - A||_{\tr} \neq 0,
	\\
	\lim_{N \approach \infty}\Srm^{\ep}_{d^2}[\Ncal(0,A_N), \Ncal(0,A)] &= 0, \; \lim_{N \approach \infty}\W_2[\Ncal(0,A_N), \Ncal(0,A)] \neq 0.
	\end{align}
	The Hilbert-Schmidt norm convergence of $\Srm^{\ep}_{d^2}$ {\bf cannot} be weakened to operator norm convergence, i.e.
	$\exists \{A_N\}_{N \in \Nbb} , A\in \Sym^{+}(\H) \cap \Tr(\H)$ such that
	\begin{align}
	\lim_{N \approach \infty}||A_N - A||_{\HS} & \neq 0, \;\; \lim_{N \approach \infty}||A_N - A|| = 0,
	\\
	\lim_{N \approach \infty}\Srm^{\ep}_{d^2}[\Ncal(0,A_N), \Ncal(0,A)] & \neq 0.
	\end{align}
\end{theorem}
{\bf Discussion of results}. 
Theorem \ref{theorem:convergence-Sinkhorn} and the preceding discussion shows that
convergence in $2$-Wasserstein distance automatically leads to convergence in Sinkhorn divergence.
However, Theorem \ref{theorem:Sinkhorn-vs-Wasserstein} shows that
the set of converging Gaussian measures under the Sinkhorn divergence is {\it strictly larger} than that
under $2$-Wasserstein distance. As we show the proof of Theorem \ref{theorem:Sinkhorn-vs-Wasserstein},
let $A = 0$, $A_N = \frac{1}{N}\sum_{i=1}^Ne_k \otimes e_k$, where $\{e_k\}_{k \in \Nbb}$ is any orthonormal basis in $\H$,
then 
$\trace(A) = 0$ but $\trace(A_N) = 1$ $\forall N \in \Nbb$, and for $\mu_N = \Ncal(0,A_N), \mu = \Ncal(0,A)$,
\begin{align}
||A_N-A||_{\tr} =1 \; \forall N \in \Nbb, \;\; ||A_N - A||_{\HS}^2 = \frac{1}{N} \approach 0 \text{ as $N \approach \infty$},
\\
\W_2(\mu_N,\mu) = 1, \forall N \in \Nbb, \; \text{while }\;\lim_{N \approach \infty}\Srm^{\ep}_{d^2}(\mu_N, \mu)  = 0.
\end{align}
In this case, convergence in Sinkhorn divergence happens {\it without} convergence of the second moment as in Eq.\eqref{equation:convergence-2nd-moment}. 
On the other hand, if we let $A_N = \frac{1}{\sqrt{N}}\sum_{k=1}^Ne_k \otimes e_k$,$A =0$, then
\begin{align}
&||A_N - A||_{\HS} = 1 \;\forall N \in \Nbb, \;\;||A_N - A||=\frac{1}{\sqrt{N}} \approach 0 \;\text{as $N \approach \infty$},
\\
&\lim_{N \approach \infty} \Srm^{\ep}_{d^2}(\mu_N, \mu) = \frac{4}{\ep} > 0 \;\forall \ep > 0.
\end{align}
Thus operator norm convergence is {\it not} sufficient for convergence of $\Srm^{\ep}_{d^2}$.
Theorem \ref{theorem:Sinkhorn-vs-Wasserstein} is a distinctive feature of the infinite-dimensional setting. 
When $\dim(\H) < \infty$,
the norms $||\;||$, $||\;||_{\HS}$, and $||\;||_{\tr}$ are equivalent and
convergence in Sinkhorn divergence is equivalent to convergence in exact $2$-Wasserstein distance.

{\bf Finite-dimensional/finite-rank approximations}.
In practice, it is often necessary to compute finite-dimensional/finite-rank approximations 
of infinite-dimensional covariance operators. The following result shows that
the entropic $2$-Wasserstein distance $\OT^{\ep}_{d^2}[\Ncal(0,A), \Ncal(0,B)]$ can be approximated via 
two sequences of approximate covariance operators $\{A_N\}_{N \in \Nbb}$, $\{B_N\}_{N \in \Nbb}$ converging in trace class norm, i.e. 
$\lim_{N \approach \infty}||A_N-A||_{\tr}= \lim_{N \approach \infty}||B_N-B||_{\tr} = 0$.
\begin{theorem}
	[\textbf{Continuity of entropic Wasserstein distance in trace class norm}]
	\label{theorem:OT-entropic-approx}
	Let $m_{A,N}, m_{B,N}, m_A, m_B \in \H$ and 
	$A,B,A_N,B_N\in \Sym^{+}(\H) \cap \Tr(\H)$. Then
	\begin{align}
	&\left|\OT^{\ep}_{d^2}[\Ncal(m_{A,N}, A_N), \Ncal(m_{B,N},B_N)] - \OT^{\ep}_{d^2}[\Ncal(m_A, A), \Ncal(m_B,B)]\right|
	\nonumber
	\\
	&\leq [||m_{A,N}|| + ||m_{B,N}||+||m_A|| + ||m_B||][||m_{A,N}-m_A|| + ||m_{B,N} - m_B||]
	\nonumber
	\\
	&\quad + ||A_N - A||_{\tr} + ||B_N - B||_{\tr}
	\nonumber
	\\
	&\quad +\frac{6}{\ep}
	\left(||A_N||_{\HS}||B_N-B||_{\HS} + ||B||_{\HS}||A_N-A||_{\HS}\right).
	\end{align}
	In particular, let $\{m_{A,N}\}_{N \in \Nbb}, \{m_{B,N}\}_{N \in \Nbb} \in \H$ and  $\{A_N\}_{N \in \Nbb}, \{B_N\}_{N \in \Nbb} \in \Sym^{+}(\H)\cap \Tr(\H)$ be such that $\lim_{N \approach 0}||m_{N,A} - m_A|| = \lim_{N \approach \infty}||m_{B,N} - m_B|| = 0$ and $\lim_{N \approach \infty}||A_N-A||_{\tr} = \lim_{N \approach \infty}||B_N - B||_{\tr} = 0$, then
	\begin{align}
	\lim_{N \approach \infty}\OT^{\ep}_{d^2}[\Ncal(m_{A,N}, A_N), \Ncal(m_{B,N},B_N)] = \OT^{\ep}_{d^2}[\Ncal(m_A, A), \Ncal(m_B,B)].
	\end{align}
\end{theorem}

The Sinkhorn divergence  $\Srm^{\ep}_{d^2}[\Ncal(0,A),\Ncal(0,B)]$, on the other hand, can be approximated via two sequences of covariance operators converging in Hilbert-Schmidt norm, 
i.e.
$\lim_{N \approach \infty}||A_N-A||_{\HS}= \lim_{N \approach \infty}||B_N-B||_{\HS} = 0$.
\begin{theorem}
	[\textbf{Continuity of Sinkhorn divergence in Hilbert-Schmidt norm}]
	\label{theorem:Sinkhorn-entropic-approx}
	Let $m_{A,N}, m_{B,N}, m_A, m_B \in \H$ and $A,B,A_N,B_N\in \Sym^{+}(\H) \cap \Tr(\H)$. Then
	\begin{align}
	&\left|\Srm^{\ep}_{d^2}[\Ncal(m_{A,N}, A_N), \Ncal(m_{B_N},B_N)] - \Srm^{\ep}_{d^2}[\Ncal(m_A, A), \Ncal(m_B,B)]\right|
	\nonumber
	\\
	&\leq [||m_{A,N}|| + ||m_{B,N}||+||m_A|| + ||m_B||][||m_{A,N}-m_A|| + ||m_{B,N} - m_B||]
	\nonumber
	\\
	& \quad + \frac{3}{\ep}\left[||A_N||_{\HS} + ||A||_{\HS} + 2||B||_{\HS}||\right]||A_N-A||_{\HS}
	\nonumber
	\\
	&\quad + \frac{3}{\ep}\left[2||A_N||_{\HS} + ||B_N||_{\HS} + ||B||_{\HS}\right]||B_N-B||_{\HS}. 
	\end{align}
	In particular, let $\{m_{A,N}\}_{N \in \Nbb}, \{m_{B,N}\}_{N \in \Nbb} \in \H$ and $\{A_N\}_{N \in \Nbb}, \{B_N\}_{N \in \Nbb} \in \Sym^{+}(\H)\cap \Tr(\H)$ be such that $\lim_{N \approach 0}||m_{N,A} - m_A|| = \lim_{N \approach \infty}||m_{B,N} - m_B|| = 0$ and $\lim_{N \approach \infty}||A_N-A||_{\HS} = \lim_{N \approach \infty}||B_N-B||_{\HS} = 0$, then
	\begin{align}
	\lim_{N \approach \infty}\Srm^{\ep}_{d^2}[\Ncal(m_{A,N}, A_N), \Ncal(m_{B,N},B_N)] = \Srm^{\ep}_{d^2}[\Ncal(m_A, A), \Ncal(m_B,B)].
	\end{align}
\end{theorem}
In the next section, we apply this result to obtain sample complexity bounds for 
the finite-dimensional approximations of $\Srm^{\ep}_{d^2}$.

\section{The RKHS setting: Kernel Gaussian-Sinkhorn divergence}
\label{section:RKHS}

We now consider the setting of Gaussian measures defined on reproducing kernel Hilbert spaces (RKHS), induced by positive definite kernels on a metric space $\X$. In this case, we obtain the nonlinear generalizations of the Wasserstein distance and Sinkhorn divergence
between Gaussian measures on Euclidean space. Furthermore, the Wasserstein distance/Sinkhorn divergence between Gaussian measures defined on finite samples
admits explicit expressions in terms of the corresponding kernel Gram matrices, which are readily computable in practice.
In particular, the kernel Gaussian-Sinkhorn divergence is an interpolation between the {\it Maximum Mean Discrepancy} (MMD) \cite{Gretton:MMD12a}
and the {\it Kernel Wasserstein distance} \cite{zhang2019:OTRKHS,Minh:2019AlphaProcrustes}.
If the kernel is characteristic, then the kernel Gaussian-Sinkhorn divergence 
is a semi-metric on the set of Borel probability measures on $\X$.
As we discuss below, if the kernel is non-characteristic, then the Sinkhorn divergence is generally
{\it more informative} than the MMD.

By virtue of the Hilbert-Schmidt norm convergence of Sinkhorn divergence between Gaussian measures,
we then apply laws of large numbers for Hilbert space-valued random variables to obtain
{\it dimension-independent} sample complexity for finite sample approximations of the Sinkhorn divergence between infinite-dimensional 
Gaussian measures. 

Throughout this section, we assume the following.
\begin{enumerate}
	\item {\bf Assumption 1}: $\X$ is a complete, separable metric space.
	\item {\bf Assumption 2}: $\rho,\rho_1,\rho_2$ are  Borel probability measures on $\X$.
	
	\item {\bf Assumption 3}: $K:\X \times \X \mapto \R$ is a continuous, positive definite kernel and $\exists \kappa > 0$ such that $K$ and $\rho$ ($\rho_1, \rho_2$) satisfy
	\begin{align}
	\int_{\X}K(x,x)d\rho(x) \leq \kappa^2 < \infty. 
	\end{align}
\end{enumerate}
The reproducing kernel Hilbert space (RKHS) {$\H_K$} of functions on $\X$
induced by {$K$} is then separable (\cite{Steinwart:SVM2008}, Lemma 4.33).
Let {$\Phi: \X \mapto \H_K$} be the corresponding canonical feature map, defined by
\begin{align}
&\Phi(x) = K_x, \;\;\;\text{with } \Phi(x)(y) = K_x(y) = K(x,y), \forall (x,y) \in \X \times \X,
\\
&K(x,y) = \la \Phi(x), \Phi(y)\ra_{\H_K} \;\; \forall (x,y) \in \X \times \X,
\\
&\la \Phi(x), f\ra_{\H_K} = \la K_x, f\ra_{\H_K} = f(x), \;\;\;\forall f \in \H_K, \forall x \in \X.
\end{align}
The Borel probability measure $\rho$ on $\X$ in Assumption 2 then satisfies
\begin{align}
\int_{\X}||\Phi(x)||_{\H_K}^2d\rho(x) = \int_{\X}K(x,x)d\rho(x)\leq \kappa^2 < \infty.
\end{align}
Thus the RKHS mean vector $\mu_{\Phi} \in \H_K$ and covariance operator {$C_{\Phi}:\H_K \mapto \H_K$} induced by the feature map $\Phi$ are both well-defined and are given by
\begin{align}
\mu_{\Phi} &= \mu_{\Phi,\rho} =  \int_{\X}\Phi(x)d\rho(x) \in \H_K, 
\\\;\;\;
C_{\Phi} &= C_{\Phi,\rho} = \int_{\X}(\Phi(x)-\mu_{\Phi})\otimes (\Phi(x)-\mu_{\Phi})d\rho(x).
\\
&= \int_{\X}\Phi(x)\otimes \Phi(x)d\rho(x) - \mu_{\Phi} \otimes \mu_{\Phi} = L_K - \mu_{\Phi} \otimes \mu_{\Phi}.
\end{align}
Here the rank-one operator $u \otimes v$ is defined by $(u\otimes v)w = \la v,w\ra_{\H_K}u$, $u,v,w \in \H_K$.
The operator $L_K: \H_K \mapto \H_K$ is given by
\begin{align}
\label{equation:LK}
L_Kf(x) = \int_{\X}\la \Phi(t), f\ra_{\H_K} \Phi(t)(x)d\rho(t) = \int_{\X}K(x,t)f(t)d\rho(t).
\end{align}
The integral operator $L_K$ is self-adjoint, positive, and trace class, 
and has been studied extensively in the literature, see e.g. \cite{CuckerSmale,SmaleZhou2007}, with
(Lemma \ref{lemma:trace-LK})
\begin{align}
\trace(L_K) &= \int_{\X}K(x,x)d\rho(x) \leq \kappa^2 < \infty.
\end{align}
Thus $C_{\Phi}$ is also a positive trace class operator on $\H_K$ (see e.g. \cite{Minh:Covariance2017}).

Let $\Xbf =(x_i)_{i=1}^m$, $m \in \Nbb$, be independently sampled from $(\X, \rho)$. 
The feature map {$\Phi$} on {$\Xbf$} 
defines the following bounded linear operator
\begin{align}
\Phi(\Xbf): \R^m \mapto \H_K, \;\;\; \Phi(\Xbf)\b = \sum_{j=1}^mb_j\Phi(x_j) , \b \in \R^m.
\end{align}
The adjoint operator $\Phi(\Xbf)^{*}:\H_K \mapto \R^m$ is the {\it sampling operator} given by
\begin{align}
\Phi(\Xbf)^{*}f = (\la f, \Phi(x_j)\ra_{\H_K})_{j=1}^m = (f(x_j))_{j=1}^m.
\end{align}
Their composition is the operator $\Phi(\Xbf)\Phi(\Xbf)^{*}: \H_K \mapto \H_K$ given by 
\begin{align}
\Phi(\Xbf)\Phi(\Xbf)^{*}f &= \sum_{j=1}^m \Phi(x_j)f(x_j) = \sum_{j=1}^m K_{x_j}f(x_j),
\\
\frac{1}{m}[\Phi(\Xbf)\Phi(\Xbf)^{*}f](x) &= \frac{1}{m}\sum_{j=1}^m K(x,x_j)f(x_j).
\end{align}
Thus we have the following corresponding empirical mean vector for $\mu_{\Phi}$ and empirical covariance operator for $C_{\Phi}$,
associated with the sample $\Xbf$,
given by
\begin{align}
\mu_{\Phi(\Xbf)} &= \frac{1}{m}\sum_{j=1}^m\Phi(x_j) = \frac{1}{m}\Phi(\Xbf)\1_m,
\label{equation:empirical-mean-RKHS}
\\
C_{\Phi(\Xbf)} &= \frac{1}{m}\Phi(\Xbf)J_m\Phi(\Xbf)^{*}: \H_K \mapto \H_K,
\\
&= \frac{1}{m}\Phi(\Xbf)\Phi(\Xbf)^{*} - \mu_{\Phi(\Xbf)}\otimes \mu_{\Phi(\Xbf)}.
\label{equation:empirical-covariance-operator-RKHS}
\end{align}
Here $J_m = I_m -\frac{1}{m}\1_m\1_m^T,\1_m = (1, \ldots, 1)^T \in \R^m$, 
is the centering matrix.

The positive trace class operator $C_{\Phi}$ and positive, finite-rank operator $C_{\Phi(\Xbf)}$, together with the mean vectors, define the Gaussian measures $\Ncal(\mu_{\Phi}, C_{\Phi})$ and $\Ncal(\mu_{\Phi(\Xbf)}, C_{\Phi(\Xbf)})$, respectively, on $\H_K$, with $\mu_{\Phi(\Xbf)}$ and $C_{\Phi(\Xbf)}$ being the finite-sample approximations
of $\mu_{\Phi}$ and $C_{\Phi}$, respectively.

In particular, for $\X = \H$, with $\H$ being a separable Hilbert space, and $K(x,y) = \la x,y\ra$, we have $\H_K \cong \H$ and the canonical feature map $\Phi: \H \mapto \H$ is the identity map. With the Gaussian measure $\Ncal(\mu,C) = \Ncal(\mu_I, C_I)$ on $\H$, we have the empirical version $\Ncal(\mu_{\Xbf}, C_{\Xbf})$, with mean and covariance operator
\begin{align}
\mu_{\Xbf} = \frac{1}{m}\sum_{i=1}^mx_i, \;\;C_{\Xbf} = \frac{1}{m}\sum_{i=1}^m(x_i-\mu_{\Xbf})\otimes (x_i-\mu_{\Xbf}).
\end{align}
For $\X = \R^d$, we have $C_{\Xbf} = \frac{1}{m}\sum_{i=1}^m(x_i-\mu_{\Xbf})(x_i - \mu_{\Xbf})^T$,
the maximum likelihood estimate for the covariance matrix of $\Ncal(\mu, C)$ on $\R^d$.
\begin{remark}
	For our current purposes, we focus on the sample covariance operator $C_{\Phi(\Xbf)}$ and sample covariance matrix $C_{\Xbf}$. Further studies on optimal empirical covariance operators along the line of e.g. \cite{ledoit2004Shrinkage,bickel2008regularizedCovariance,cai2010optimalCovariance} will be considered in a future work.
\end{remark}

{\bf Kernel Gaussian-Sinkhorn divergence between Borel probability measures}.
Let $\rho_1,\rho_2$ be two Borel probability measures on $\X$. 
Let $\mu_{\Phi,\rho_1}, \mu_{\Phi,\rho_2} \in \H_K$ and $C_{\Phi,\rho_1}, C_{\Phi,\rho_2}: \H_K \mapto \H_K$
denote the corresponding RKHS mean vectors and covariance operators, respectively.
Then we have two Gaussian measures $\Ncal(\mu_{\Phi,\rho_i}, C_{\Phi, \rho_i})$, $i=1,2$, on the RKHS $\H_K$,
with a well-defined Sinkhorn divergence between them.

\begin{definition}
[\textbf{kernel Gaussian-Sinkhorn divergence}]
\label{definition:kernel-Sinkhorn}
The Sinkhorn divergence between 
RKHS Gaussian measures
$\Srm^{\ep}_{d^2}[\Ncal(\mu_{\Phi,\rho_1}, C_{\Phi, \rho_1}), \Ncal(\mu_{\Phi,\rho_2}, C_{\Phi, \rho_2})]$ , as well as its finite sample approximation,
is called the {\it kernel Gaussian-Sinkhorn divergence}.
\end{definition}
We have the following decomposition of the Sinkhorn divergence
\begin{align}
&\Srm^{\ep}_{d^2}[\Ncal(\mu_{\Phi,\rho_1}, C_{\Phi, \rho_1}), \Ncal(\mu_{\Phi,\rho_2}, C_{\Phi, \rho_2})] 
\nonumber
\\
&= ||\mu_{\Phi,\rho_1}-\mu_{\Phi,\rho_2}||^2_{\H_K} 
+ \Srm^{\ep}_{d^2}[\Ncal(0, C_{\Phi, \rho_1}), \Ncal(0, C_{\Phi, \rho_2})],
\end{align}
where the first term, $\MMD^2_K(\rho_1, \rho_2) = ||\mu_{\Phi,\rho_1}-\mu_{\Phi,\rho_2}||^2_{\H_K}$
is the squared {\it Maximum Mean Discrepancy} \cite{Gretton:MMD12a} between $\rho_1$ and $\rho_2$,
with
\begin{align}
\MMD_K(\rho_1,\rho_2) = \sup_{f \in \H_K, ||f||_{\H_K} \leq 1}\left[\int_{\X}fd\rho_1(x) - \int_{\X}fd\rho_2(x)\right].
\end{align}
A  bounded, measurable kernel $K$ on $\X$ is said to be {\it characteristic} 
\cite{fukumizu2007:characteristickernel}
if
\begin{align}
\label{equation:characteristic-kernel}
\MMD_K(\rho_1, \rho_2) = 0 \equivalent \rho_1= \rho_2 \;\;\;\forall \rho_1,\rho_2 \in \Pcal(\X).
\end{align}
In this case, $\MMD_K$ is a metric on $\Pcal(\X)$, a so-called {\it integral probability metric}.
Examples of characteristic kernels are  Gaussian kernel $K(x,y) = \exp(-\frac{||x-y||^2}{\sigma^2})$,
$\sigma \neq 0$, $\X= \R^d$, and Laplacian kernel $K(x,y) = \exp(-a||x-y||), a > 0, \X = \R^d$
(Theorem 2, \cite{fukumizu2007:characteristickernel}).


The following generalizes Eq.\eqref{equation:characteristic-kernel}
to the kernel Gaussian-Sinkhorn divergence.
\begin{theorem}
	[\textbf{Kernel Gaussian-Sinkhorn divergence between Borel probability measures}]
	\label{theorem:Sinkhorn-divergence-Borel-probability-measures}
	Assume Assumptions 1-3.
	Let $\rho_1,\rho_2$ be two Borel probability measures on $\X$.
	Let $K: \X \times \X \mapto \R$ be a characteristic kernel.
	Then
	\begin{align}
	\Srm^{\ep}_{d^2}[\Ncal(\mu_{\Phi,\rho_1}, C_{\Phi, \rho_1}), \Ncal(\mu_{\Phi,\rho_2}, C_{\Phi, \rho_2})] &= \Srm^{\ep}_{d^2}[\Ncal(\mu_{\Phi,\rho_2}, C_{\Phi, \rho_2}), \Ncal(\mu_{\Phi,\rho_1}, C_{\Phi, \rho_1})],
	\\
	\Srm^{\ep}_{d^2}[\Ncal(\mu_{\Phi,\rho_1}, C_{\Phi, \rho_1}), \Ncal(\mu_{\Phi,\rho_2}, C_{\Phi, \rho_2})] &\geq  0,
	\\
	\Srm^{\ep}_{d^2}[\Ncal(\mu_{\Phi,\rho_1}, C_{\Phi, \rho_1}), \Ncal(\mu_{\Phi,\rho_2}, C_{\Phi, \rho_2})] &= 0
	\equivalent \rho_1 = \rho_2 \;\;\forall \rho_1,\rho_2 \in \Pcal(\X).
	\end{align} 
	Here $0 \leq \ep \leq \infty$, with $S^{\infty}_{d^2} = \MMD_K^2$.
\end{theorem}
{\bf Kernel Gaussian-Sinkhorn divergence as a semi-metric}. We recall that $(E,d)$ is a {\it semi-metric space} (see e.g. \cite{wilson1931semimetric}) if $\forall x,y\in E$,
\begin{enumerate}
	\item $d(x,y) \geq 0$,
	\item $d(x,y) = d(y,x)$,
	\item $d(x,y) = 0 \equivalent x=y$.
\end{enumerate}
 
Thus if $K$ is characteristic, then $\Srm^{\ep}_{d^2}$, defined on $\Gauss(\H_K)$, is a {\it semi-metric} on $\Pcal(\X)$ for $0 < \ep < \infty$,
with $\sqrt{\Srm^{\ep}_{d^2}}$ being a metric when $\ep=0, \infty$. 

If $K$ is non-characteristic, then
$S^{\ep}_{d^2}$ is generally {\it more informative} than MMD. For example, if $K(x,y) = \la x,y\ra$ on
$\H$, a separable Hilbert space,
then $S^{\ep}_{d^2}$ is the Sinkhorn divergence between Gaussian measures on $\H$,
defining a semi-metric on $\Gauss(\H)$, 
whereas MMD is the distance between the mean vectors.

{\bf Finite sample approximations of kernel Gaussian-Sinkhorn divergence}.
Let $\rho_1, \rho_2$ be two Borel probability measures on $\X$.
Let $\Xbf = (x_i)_{i=1}^m$ and $\Ybf = (y_i)_{i=1}^n$ be independently sampled from
$(\X, \rho_1)$ and $(\X, \rho_2)$, respectively.
Together with the feature map $\Phi$, these define the RKHS mean vectors $\mu_{\Phi,\rho_1}$, $\mu_{\Phi,\rho_2}$, $\mu_{\Phi(\Xbf)}$, and $\mu_{\Phi(\Ybf)}$, and RKHS covariance operators
$C_{\Phi, \rho_1}$, $C_{\Phi, \rho_2}$, $C_{\Phi(\Xbf)}$, and $C_{\Phi(\Ybf)}$, along with 
the corresponding Gaussian measures $\Ncal(\mu_{\Phi,\rho_1}, C_{\Phi, \rho_1})$, $\Ncal(\mu_{\Phi,\rho_2}, C_{\Phi, \rho_2})$, $\Ncal(\mu_{\Phi(\Xbf)}, C_{\Phi(\Xbf)})$, and $\Ncal(\mu_{\Phi(\Ybf)}, C_{\Phi(\Ybf)})$ on $\H_K$.

One particular advantage of the RKHS setting is that 
for $\mu_1 = \Ncal(\mu_{\Phi_{\Xbf}}, C_{\Phi(\Xbf)})$ and $\mu_2 = \Ncal(\mu_{\Phi_{\Ybf}}, C_{\Phi(\Ybf)})$, both  $\OT^{\epsilon}_{d^2}(\mu_1, \mu_2)$ and $S^{\epsilon}_{d^2}(\mu_1, \mu_2)$
admit closed form expressions in terms of the kernel Gram matrices defined on the finite samples $\Xbf$ and $\Ybf$.
%
Define the following kernel Gram matrices
\begin{align}
K[\Xbf] &= \Phi(\Xbf)^{*}\Phi(\Xbf) \in \R^{m \times m}, (K[\Xbf])_{ij} = K(x_i,x_j),\; i,j=1,\ldots, m
\\
K[\Ybf] &= \Phi(\Ybf)^{*}\Phi(\Ybf) \in \R^{n \times n}, (K[\Ybf])_{ij} = K(y_i, y_j),\; i,j =1,\ldots, n
\\ 
K[\Xbf,\Ybf] &= \Phi(\Xbf)^{*}\Phi(\Ybf) \in \R^{m \times n}, (K[\Xbf, \Ybf])_{ij} = K(x_i,y_j), \; 1 \leq i \leq m, 1 \leq j \leq n
\end{align}
For completeness, the following is a generalization of Theorem 15 in \cite{Minh2020:EntropicHilbert}.
For our current purposes, we focus exclusively on the Sinkhorn divergence.

\begin{theorem}
	[\textbf{Sinkhorn divergences between Gaussian measures on RKHS - Finite samples}]
	\label{theorem:RKHS-distance}
	Let $\epsilon > 0$ be fixed.
	For $\mu_1 =\Ncal(\mu_{\Phi(\Xbf)}, C_{\Phi(\Xbf)})$, $\mu_2 = \Ncal(\mu_{\Phi(\Ybf)}, C_{\Phi(\Ybf)})$,
	\begin{align}
	S^{\epsilon}_{d^2}(\mu_1, \mu_2) & = \frac{1}{m^2}\1_m^TK[\Xbf]\1_m + \frac{1}{n^2}\1_n^TK[\Ybf]\1_n - \frac{2}{mn}\1_m^TK[\Xbf,\Ybf]\1_n
	\nonumber
	\\
	& 
	+\frac{\epsilon}{4}\trace\left[- I + \left(I + \frac{16}{\epsilon^2m^2} (J_mK[\Xbf]J_m)^2\right)^{1/2}\right]
	\nonumber
	\\
	&+ 
	\frac{\epsilon}{4}\trace\left[- I + \left(I + \frac{16}{\epsilon^2n^2} (J_nK[\Ybf]J_n)^2\right)^{1/2}\right]
	\nonumber
	\nonumber
	\\
	& -
	\frac{\epsilon}{2}\trace\left[- I + \left(I + \frac{16}{\epsilon^2mn} J_mK[\Xbf,\Ybf]J_nK[\Ybf,\Xbf]J_m\right)^{1/2}\right]
	\nonumber
	\nonumber
	\\
	& + \frac{\epsilon}{2}\log\det\left(\frac{1}{2}I + \frac{1}{2}\left(I + \frac{16}{\epsilon^2mn}J_mK[\Xbf,\Ybf]J_nK[\Ybf,\Xbf]J_m \right)^{1/2}\right)
	\nonumber
	\\
	& - \frac{\epsilon}{4}\log\det\left(\frac{1}{2}I + \frac{1}{2}\left(I + \frac{16}{\epsilon^2m^2}(J_mK[\Xbf]J_m)^2 \right)^{1/2}\right)
	\nonumber
	\\
	& - \frac{\epsilon}{4}\log\det\left(\frac{1}{2}I + \frac{1}{2}\left(I + \frac{16}{\epsilon^2n^2}(J_nK[\Ybf]J_n)^2 \right)^{1/2}\right).
	\end{align}
\end{theorem}
{\bf Limiting cases}. In particular, 	
as $\epsilon \approach \infty$,
\begin{align}
\lim_{\epsilon \approach \infty}S^{\epsilon}_{d^2}(\mu_0, \mu_1) 
&= 
||\mu_{\Phi(\Xbf)} - \mu_{\Phi(\Ybf)}||^2_{\H_K} 
\\
&= \frac{1}{m^2}\1_m^TK[\Xbf]\1_m + \frac{1}{n^2}\1_n^TK[\Ybf]\1_n - \frac{2}{mn}\1_m^TK[\Xbf,\Ybf]\1_n.
\nonumber 
\end{align}
This is the empirical squared Kernel MMD distance \cite{Gretton:MMD12a}. As $\ep \approach 0$,
\begin{equation}
\begin{aligned}
\lim_{\ep \approach 0}S^{\epsilon}_{d^2}(\mu_0, \mu_1) 
&= \frac{1}{m^2}\1_m^TK[\Xbf]\1_m + \frac{1}{n^2}\1_n^TK[\Ybf]\1_n - \frac{2}{mn}\1_m^TK[\Xbf,\Ybf]\1_n
\\
&\quad+ \frac{1}{m}\trace(K[\Xbf]J_m) +  \frac{1}{n}\trace(K[\Ybf]J_n)
\\
&\quad- \frac{2}{\sqrt{mn}}\trace[J_mK[\Xbf,\Ybf]J_nK[\Ybf,\Xbf]J_m]^{1/2}.
\end{aligned}
\end{equation}
This is the Kernelized Wasserstein Distance \cite{zhang2019:OTRKHS,Minh:2019AlphaProcrustes}.

For $\X = \H$, with $\H$ a separable Hilbert space, 
and $K(x,y) = \la x,y\ra$, we recover the finite sample approximation of the
Sinkhorn divergence between two Gaussian measures on $\H$.

\subsection{Sample complexity with bounded kernels}
\label{section:bounded-kernels}

We now show the convergence of $||\mu_{\Phi(\Xbf)} - \mu_{\Phi}||_{\H_K}$ and 
$||C_{\Phi(\Xbf)} - C_{\Phi}||_{\HS(\H_K)}$, which in turns lead to the convergence
of $\Srm^{\ep}_{d^2}[\Ncal(\mu_{\Phi(\Xbf)}, C_{\Phi(\Xbf)}), \Ncal(\mu_{\Phi}, C_{\Phi})]$, 
as $m \approach \infty$. We consider the cases the kernel $K$ is bounded and unbounded on $\X$
separately, with the former giving tighter bounds.

{\bf Assumption 4}. Throughout this section, we assume that
\begin{align}
\sup_{x \in \X}K(x,x) \leq \kappa^2.
\end{align}
This implies the condition $\int_{\X}K(x,x)d\rho(x) < \infty$ in Assumption 3.
Assumption 4 is automatically satisfied for translation-invariant kernels, such as Gaussian kernel,
on $\R^d$, but not for polynomial kernels unless $\X \subset \R^d$ is compact.

By Assumption 4, the random variables $\xi_1:(\X,\rho)\mapto \H_K$, defined by $\xi_1(x) = \Phi(x)$,
and $\xi_2:(\X,\rho) \mapto \HS(\H_K)$, defined by $\xi_2(x) = \Phi(x) \otimes \Phi(x)$, 
are both bounded.
We can then apply the following law of large numbers for Hilbert space-valued random variables, which
is a consequence of a general result due to Pinelis (\cite{Pinelis1994optimum}, Theorem 3.4).
The following version is Lemma 2 in \cite{SmaleZhou2007}. 
\begin{proposition}
	[\cite{SmaleZhou2007}]
	\label{proposition:Pinelis}
	Let $(\H, ||\;||)$ be a Hilbert space and $\xi$ be a random variable on 
	$(Z, \rho)$ with values in $\H$. Assume that $\exists M > 0$ such that $||\xi|| \leq M < \infty$ almost surely.
	Let $\sigma^2(\xi)= \bE||\xi||^2$. Let $(z_i)_{i=1}^m$ be independently sampled according to $\rho$.
	Then for any $0 < \delta < 1$, with probability at least $1-\delta$,
	\begin{align}
	\left\|\frac{1}{m}\sum_{i=1}^m\xi(z_i) - \bE\xi \right\| \leq \frac{2M\log\frac{2}{\delta}}{m} + \sqrt{\frac{2\sigma^2(\xi)\log\frac{2}{\delta}}{m}}.
	\end{align}
\end{proposition}
Applying Proposition \ref{proposition:Pinelis}, we obtain the following results.
\begin{theorem}
	[\textbf{Convergence of mean and covariance operators - bounded kernels}]
	\label{theorem:CPhi-concentration}
	Assume Assumptions 1-4.
	Let $\Xbf = (x_i)_{i=1}^m$, $m \in \Nbb$ be independently  sampled from $(\X, \rho)$. Then
	\begin{align}
	||\mu_{\Phi}||_{\H_K} &\leq \kappa, \;\;\; ||\mu_{\Phi(\Xbf)}||_{\H_K} \leq \kappa \;\;\forall \Xbf \in \X^m.
	\\
	||C_{\Phi}||_{\HS(\H_K)}|| &\leq 2\kappa^2,\;\;\;
	||C_{\Phi(\Xbf)}||_{\HS(\H_K)}  \leq 2\kappa^2, \;\;\forall \Xbf \in \X^m.
	\end{align}
	For any $0 < \delta <1$, let $U \subset \X^m$ be such that both of the following hold  
	\begin{align}
	||\mu_{\Phi(\Xbf)} - \mu_{\Phi}||_{\H_K} & \leq \kappa\left(\frac{2\log\frac{4}{\delta}}{m} + \sqrt{\frac{2\log\frac{4}{\delta}}{m}}\right),
	\\
	||C_{\Phi(\Xbf)} - C_{\Phi}||_{\HS(\H_K)} &\leq 3\kappa^2\left(\frac{2\log\frac{4}{\delta}}{m} + \sqrt{\frac{2\log\frac{4}{\delta}}{m}}\right),
	\end{align}
	$\forall \Xbf \in U$. Then $\rho^m(U) \geq 1-\delta$.
\end{theorem}
%
\begin{theorem}
	[\textbf{Convergence of empirical Gaussian measures on RKHS in Sinkhorn divergence - bounded kernels}]
	\label{theorem:Sinkhorn-RKHS-concentration}
	Assume Assumptions 1-4. 
	Let $\Xbf = (x_i)_{i=1}^m$, $m \in \Nbb$ be independently  sampled from $(\X, \rho)$.
	For any $0 < \delta < 1$, with probability at least $1-\delta$,
	\begin{align}
	\Srm^{\ep}_{d^2}[\Ncal(\mu_{\Phi(\Xbf)}, C_{\Phi(\Xbf)}), \Ncal(\mu_{\Phi}, C_{\Phi})] 
	&\leq
	\kappa^2\left(\frac{2\log\frac{4}{\delta}}{m} + \sqrt{\frac{2\log\frac{4}{\delta}}{m}}\right)^2
	\nonumber
	\\
	&\quad + \frac{36\kappa^4}{\ep}\left(\frac{2\log\frac{4}{\delta}}{m} + \sqrt{\frac{2\log\frac{4}{\delta}}{m}}\right). 
	\end{align}
\end{theorem}




{\bf Sample complexity bounds}. Having obtained the explicit expression of $\Srm^{\ep}_{d^2}[\Ncal(\mu_{\Phi(\Xbf)}, C_{\Phi(\Xbf)}),\Ncal(\mu_{\Phi(\Ybf)}, C_{\Phi(\Ybf)})]$ in terms of the kernel Gram matrices, we now show that
it converges to $\Srm^{\ep}_{d^2}[\Ncal(\mu_{\Phi,\rho_1}, C_{\Phi, \rho_1}), \Ncal(\mu_{\Phi,\rho_2}, C_{\Phi, \rho_2})]$, with high probability, as the sample sizes $m, n \approach \infty$.

\begin{theorem}
	[\textbf{Sample complexity for finite sample approximations of Sinkhorn divergence between Gaussian measures on RKHS -bounded kernels}]
	\label{theorem:Sinkhorn-RKHS-approximation}
	Assume Assumptions 1-4. Let $\Xbf = (x_i)_{i=1}^m$ and $\Ybf = (y_j)_{j=1}^n$ be independently sampled from  $(\X,\rho_1)$ and $(\X, \rho_2)$, respectively. Then for any $0 < \delta < 1$, with probability at least $1-\delta$, 
	\begin{align}
	&\left|\Srm^{\ep}_{d^2}[\Ncal(\mu_{\Phi(\Xbf)}, C_{\Phi(\Xbf)}), \Ncal(\mu_{\Phi(\Ybf)}, C_{\Phi(\Ybf)})]
	- \Srm^{\ep}_{d^2}[\Ncal(\mu_{\Phi, \rho_1}, C_{\Phi, \rho_1}), \Ncal(\mu_{\Phi, \rho_2}, C_{\Phi, \rho_2})]\right|
	\nonumber
	\\
	& \leq 4\kappa^2\left(\frac{2\log\frac{8}{\delta}}{m} + \sqrt{\frac{2\log\frac{8}{\delta}}{m}}+\frac{2\log\frac{8}{\delta}}{n} + \sqrt{\frac{2\log\frac{8}{\delta}}{n}}\right)
	\nonumber
	\\
	& \quad +\frac{72\kappa^4}{\ep}\left(\frac{2\log\frac{8}{\delta}}{m} + \sqrt{\frac{2\log\frac{8}{\delta}}{m}}\right)
	+
	\frac{72\kappa^4}{\ep}\left(\frac{2\log\frac{8}{\delta}}{n} + \sqrt{\frac{2\log\frac{8}{\delta}}{n}}\right). 
	\end{align}
\end{theorem}

{\bf Discussion of results}.
In Theorems \ref{theorem:CPhi-concentration}, \ref{theorem:Sinkhorn-RKHS-concentration} and  \ref{theorem:Sinkhorn-RKHS-approximation}, if $\kappa$ is an absolute constant, e.g. with 
translation-invariant kernels such as the Gaussian kernel, then 
the convergence rates are completely {\it dimension-independent}, with the results hold for $\X = \H$, $\H$ being an infinite-dimensional separable Hilbert space.

Furthermore, in Theorems \ref{theorem:Sinkhorn-RKHS-concentration} and  \ref{theorem:Sinkhorn-RKHS-approximation}, the convergence rates are inversely proportional to $\ep$ and thus accelerates as $\ep \approach \infty$.
When $\epsilon = \infty$, we simply have the convergence rate for the MMD \cite{Gretton:MMD12a}.
The convergence becomes slower when $\ep \approach 0$, i.e. when we are close to 
the exact Wasserstein distance, and the bound is vacuous when $\ep=0$, since the right hand side is infinite.

\subsection{General kernels}
\label{section:general-kernels}

Theorems \ref{theorem:CPhi-concentration}, \ref{theorem:Sinkhorn-RKHS-concentration}, and \ref{theorem:Sinkhorn-RKHS-approximation} do not apply to polynomial kernels on $\R^d$, which are unbounded. In particular, they do not apply to the linear kernel $K(x,y) = \la x,y\ra$,
in which case the kernel Gaussian-Sinkhorn divergence is precisely the Sinkhorn divergence between Gaussian measures on $\R^d$.
We now consider a more general setting without the boundedness assumption for the kernel $K$ in Assumption 4. Instead, we assume the following.

{\bf Assumption 5}. $K: \X \times \X \mapto \R$ is a continuous, positive definite kernel and there exists $\kappa >0$ such that 
$K$ and $\rho$ ($\rho_1, \rho_2$) satisfy
\begin{align}
\int_{\X}K(x,x)^2d\rho(x) \leq \kappa^4 <  \infty.
\end{align}
By H\"older inequality, Assumption 5 implies in particular Assumption 3, that is
$\int_{\X}K(x,x)d\rho(x) \leq \kappa^2$.
Apart from the translation-invariant kernels such as the Gaussian kernel, Assumption 5 is valid in particular
for polynomial kernels of any degree on $\R^d$ if $\rho$ is the Gaussian measure. 

Without Assumption 4, the random variable $\xi_1:(\X, \rho) \mapto \H_K$, defined by $\xi_1(x) = \Phi(x)$,
and $\xi_2:(X, \rho) \mapto \HS(\H_K)$, defined by $\xi_2(x) = \Phi(x) \otimes \Phi(x) \in \HS(\H_K)$,
are generally unbounded. Instead, by Assumption 5,
\begin{align}
\bE{||\xi_1||^2_{\H_K}} = \int_{\X}K(x,x)d\rho(x) \leq \kappa^2 < \infty,
\\
\bE{||\xi_2}||^2_{\HS(\H_K)} = \int_{\X}K(x,x)^2\rho(x) \leq \kappa^4 < \infty.
\end{align}
We exploit these bounded variances and Chebyshev inequality to obtain the following,
which correspond to Theorems \ref{theorem:CPhi-concentration}, \ref{theorem:Sinkhorn-RKHS-concentration}, and
\ref{theorem:Sinkhorn-RKHS-approximation} in Section \ref{section:bounded-kernels}.

\begin{proposition}
	[\textbf{Convergence of mean and covariance operator - general kernels}]
	\label{proposition:CPhi-concentration-unbounded}
	Assume Assumptions 1,2 and 5.
	For any $0 < \delta< 1$, let $U \in (\X^m, \rho^m)$ be such that
	\begin{align}
	||\mu_{\Phi(\Xbf)} - \mu_{\Phi}||_{\H_K} &\leq \frac{2\kappa}{\sqrt{m}\delta},
	\\
	||\mu_{\Phi(\Xbf)}|| & \leq \kappa\left(1+ \frac{2}{\sqrt{m}\delta}\right),
	\\
	||C_{\Phi(\Xbf)} - C_{\Phi}||_{\HS(\H_K)} &\leq 
	\frac{2\kappa^2}{\sqrt{m}\delta}\left(3 + \frac{2}{\sqrt{m}\delta}\right),
	\\
	||C_{\Phi(\Xbf)}||_{\HS(\H_K)}  &\leq 2\kappa^2 + \frac{2\kappa^2}{\sqrt{m}\delta}\left(3 + \frac{2}{\sqrt{m}\delta}\right),
	\end{align}
	for all $\Xbf \in U$. Then $\rho^m(U) \geq 1-\delta$.
\end{proposition}

\begin{theorem}
	[\textbf{Convergence of empirical Gaussian measures on RKHS in Sinkhorn divergence - general kernels}]
	\label{theorem:Sinkhorn-RKHS-concentration-unbounded}
	Assume Assumptions 1,2, and 5. Let $\Xbf = (x_i)_{i=1}^m$, $m \in \Nbb$, be independently sampled
	from $(\X, \rho)$. For any $0 < \delta <1$, with probability at least $1-\delta$,
	\begin{align}
	&\Srm^{\ep}_{d^2}[\Ncal(\mu_{\Phi(\Xbf)}, C_{\Phi(\Xbf)}), \Ncal(\mu_{\Phi}, C_{\Phi})]
	\nonumber
	\\
	&\quad \leq \frac{4\kappa^2}{m \delta^2} + \frac{12 \kappa^4}{\ep \sqrt{m}\delta}\left[2 + \frac{1}{\sqrt{m}\delta}\left(3 + \frac{2}{\sqrt{m}\delta}\right)\right]\left(3 + \frac{2}{\sqrt{m}\delta}\right).
	\end{align}
\end{theorem}

The convergence rate in Theorem \ref{theorem:Sinkhorn-RKHS-concentration-unbounded} is {\it dimension-independent} and 
of the form $O(\frac{1}{\sqrt{m}}(1+\frac{1}{\ep}))$.
For $\X = \H$ and $K(x,y) = \la x,y\ra$, Theorem \ref{theorem:Sinkhorn-RKHS-concentration}
gives the convergence rate in the Sinkhorn divergence between the empirical measure $\Ncal(\mu_{\Xbf}, C_{\Xbf})$ and the Gaussian measure $\Ncal(\mu, C)$ on $\H$. In this case, the smallest constant $\kappa$ satisfying Assumption 5 is given by the following.

\begin{lemma}
	\label{lemma:Gaussian-integral-norm-4}
	For the Gaussian measure $\Ncal(\mu,C)$ on $\H$,
	\begin{align}
	\int_{\H}||x||^4d\Ncal(\mu, C)(x) &=  2||C||^2_{\HS} + 4\la \mu, C\mu\ra + (\trace{C} +||\mu||^2)^2.
	\end{align}
\end{lemma}

Combining Theorem \ref{theorem:Sinkhorn-RKHS-concentration} and Lemma \ref{lemma:Gaussian-integral-norm-4} immediately leads to the following.

\begin{corollary}
	[\textbf{Convergence of empirical Gaussian measures in Sinkhorn divergence on Hilbert space}]
	\label{corollary:Sinkhorn-Gaussian-concentration}
	Let $\rho = \Ncal(\mu,C)$ on $\H$.
	Let $\Xbf = (x_i)_{i=1}^m$, $m \in \Nbb$, be independently sampled
	from $(\H, \rho)$. For any $0 < \delta <1$, with probability at least $1-\delta$,
	\begin{align}
	&\Srm^{\ep}_{d^2}[\Ncal(\mu_{\Xbf}, C_{\Xbf}), \Ncal(\mu, C)]
	\nonumber
	\\
	&\quad \leq \frac{4\kappa^2}{m \delta^2} + \frac{12 \kappa^4}{\ep \sqrt{m}\delta}\left[2 + \frac{1}{\sqrt{m}\delta}\left(3 + \frac{2}{\sqrt{m}\delta}\right)\right]\left(3 + \frac{2}{\sqrt{m}\delta}\right).
	\end{align}
	Here $\kappa > 0$ is given by the following
	\begin{align}
	\label{equation:kappa-Gaussian-linear-kernel}
	\kappa =  \left(2||C||^2_{\HS} + 4\la \mu, C\mu\ra + (\trace{C} +||\mu||^2)^2\right)^{1/4}.
	\end{align}
\end{corollary}

Corollary \ref{corollary:Sinkhorn-Gaussian-concentration} is valid on any separable Hilbert space $\H$.
Note, however, that for $\H = \R^d$, $\kappa$ may depend on the dimension $d$ if $\trace(C)$ and $||C||_{\HS}$ do.
For example, if $\rho = \Ncal(0, I_d)$, then $\kappa = (d^2 +2d)^{1/4}$.

\begin{theorem}
	[\textbf{Sample complexity for finite sample approximations of Sinkhorn divergence between Gaussian measures on RKHS - general kernels}]
	\label{theorem:Sinkhorn-RKHS-approximation-unbounded}
	Assume Assumptions 1,2, and 5. Let $\Xbf = (x_i)_{i=1}^m$ and $\Ybf = (y_j)_{j=1}^n$ be independently sampled from  $(\X,\rho_1)$ and $(\X, \rho_2)$, respectively. Then for any $0 < \delta < 1$, with probability at least $1-\delta$, 
	\begin{align}
	&\left|\Srm^{\ep}_{d^2}[\Ncal(\mu_{\Phi(\Xbf)}, C_{\Phi(\Xbf)}), \Ncal(\mu_{\Phi(\Ybf)}, C_{\Phi(\Ybf)})]
	- \Srm^{\ep}_{d^2}[\Ncal(\mu_{\Phi,\rho_1}, C_{\Phi, \rho_1}), \Ncal(\mu_{\Phi,\rho_2}, C_{\Phi, \rho_2})]\right|
	\nonumber
	\\
	& \leq \frac{16\kappa^2}{\delta}\left(1 + \frac{1}{\sqrt{m}\delta} + \frac{1}{\sqrt{n}\delta}\right)
	\left(\frac{1}{\sqrt{m}} + \frac{1}{\sqrt{n}}\right)
	\nonumber
	\\
	& \quad + \frac{48\kappa^4}{\ep\sqrt{m}\delta}\left[2 + \frac{1}{\sqrt{m}\delta}\left(3 + \frac{4}{\sqrt{m}\delta}\right)\right]\left(3 + \frac{4}{\sqrt{m}\delta}\right)
	\nonumber
	\\
	&\quad + \frac{48\kappa^4}{\ep\sqrt{n}\delta}\left[\frac{2}{\sqrt{m}\delta}\left(3 + \frac{4}{\sqrt{m}\delta}\right) + \frac{1}{\sqrt{n}\delta}\left(3 + \frac{4}{\sqrt{n}\delta}\right)+2\right]\left(3 + \frac{4}{\sqrt{n}\delta}\right).
	\end{align}
\end{theorem}

When $K(x,y) = \la x,y\ra$ on $\H$, 
Theorem \ref{theorem:Sinkhorn-RKHS-approximation-unbounded} gives the finite sample complexity for the Sinkhorn divergence
between two Gaussian measures on $\H$. Combining Theorem \ref{theorem:Sinkhorn-RKHS-approximation} and Lemma \ref{lemma:Gaussian-integral-norm-4} immediately leads to the following.

\begin{corollary}
	[\textbf{Sample complexity for finite sample approximations of Sinkhorn divergence between Gaussian measures on Hilbert space}]
	\label{corollary:Sinkhorn-Gaussian-approximation-unbounded}
	Let $\rho_1 = \Ncal(\mu_1, C_1)$ and $\rho_2 = \Ncal(\mu_2, C_2)$ on $\H$.
	Let $\Xbf = (x_i)_{i=1}^m$ and $\Ybf = (y_j)_{j=1}^n$ be independently sampled from  $(\H,\rho_1)$ and $(\H, \rho_2)$, respectively. Then for any $0 < \delta < 1$, with probability at least $1-\delta$, 
	\begin{align}
	&\left|\Srm^{\ep}_{d^2}[\Ncal(\mu_{\Xbf}, C_{\Xbf}), \Ncal(\mu_{\Ybf}, C_{\Ybf})]
	- \Srm^{\ep}_{d^2}[\Ncal(\mu_1, C_1), \Ncal(\mu_2, C_2)]\right|
	\nonumber
	\\
	& \leq \frac{16\kappa^2}{\delta}\left(1 + \frac{1}{\sqrt{m}\delta} + \frac{1}{\sqrt{n}\delta}\right)
	\left(\frac{1}{\sqrt{m}} + \frac{1}{\sqrt{n}}\right)
	\nonumber
	\\
	& \quad + \frac{48\kappa^4}{\ep\sqrt{m}\delta}\left[2 + \frac{1}{\sqrt{m}\delta}\left(3 + \frac{4}{\sqrt{m}\delta}\right)\right]\left(3 + \frac{4}{\sqrt{m}\delta}\right)
	\nonumber
	\\
	&\quad + \frac{48\kappa^4}{\ep\sqrt{n}\delta}\left[\frac{2}{\sqrt{m}\delta}\left(3 + \frac{4}{\sqrt{m}\delta}\right) + \frac{1}{\sqrt{n}\delta}\left(3 + \frac{4}{\sqrt{n}\delta}\right)+2\right]\left(3 + \frac{4}{\sqrt{n}\delta}\right).
	\end{align}
	Here $\kappa > 0$ is given by $\kappa = \max(\kappa_1, \kappa_2)$, where
	\begin{align}
	\label{equation:kappa12-Gaussian-linear-kernel}
	\kappa_i =  \left(2||C_i||^2_{\HS} + 4\la \mu_i, C_i\mu_i\ra + (\trace{C_i} +||\mu_i||^2)^2\right)^{1/4}, \;\; i=1,2.
	\end{align}
\end{corollary}

As Theorem \ref{theorem:Sinkhorn-RKHS-concentration-unbounded} and Corollary \ref{corollary:Sinkhorn-Gaussian-concentration},
Theorem \ref{theorem:Sinkhorn-RKHS-approximation-unbounded} and Corollary \ref{corollary:Sinkhorn-Gaussian-approximation-unbounded} are valid
for $\X= \H$, $\H$ being any separable Hilbert space. For $\H = \R^d$, 
with $m=n$, the convergence rates have the form $O((1+\frac{1}{\ep})\frac{1}{\sqrt{n}})$ and are thus {\it dimension-independent}.
The total sample complexity itself may contain the dimension $d$ if $\kappa_i$, $i=1,2$ depend on $d$, as noted before,
e.g. $\kappa_i = (d^2+2d)^{1/4}$ if $\rho_i = \Ncal(0,I_d)$. 

For comparison, we remark that in \cite{genevay18sample} (Theorem 3), a sample complexity bound
was obtained for the Sinkhorn divergence between two probability measures with {\it bounded support} $\X \subset \R^d$ with diameter $|\X|$, with a bounded $L$-Lipschitz cost function $c$, of the form $O\left(\exp(2L|\X|+||c||_{\infty})\left(1+\frac{1}{\ep^{\floor{d/2}}}\right)\frac{1}{\sqrt{n}}\right)$.
Thus, while it is of order $O(\frac{1}{\sqrt{n}})$, it also grows exponentially with the diameter $|\X|$.
In \cite{mena2019samplecomplexityEntropicOT}, for $\sigma^2$-subgaussian measures
on $\R^d$, 
the authors achieved the rate of convergence of the form
$O\left(\ep\left(1+\frac{\sigma^{\ceil{5d/2} + 6}}{\ep^{\ceil{5d/4} + 3}}\right)n^{-1/2}\right)$.

\subsection{Sample complexity for the exact Wasserstein distance}
\label{section:exact-distance}

We now apply the methods above to analyze the sample complexity
of the kernel Wasserstein distance, in particular, of the $2$-Wasserstein distance between two Gaussian measures on
$\R^d$. In contrast to the previous sections, the analysis in this section is strictly for the {\it finite-dimensional} setting.

By Proposition \ref{proposition:2Wasserstein-gaussian-upperbound}, on $\R^d$, convergence of Gaussian measures in the Wasserstein distance $\W_2$ occurs if the 
corresponding covariance matrices converge in the Hilbert-Schmidt norm. Similar to Theorem 
\ref{theorem:Sinkhorn-RKHS-concentration-unbounded}, we obtain the following.

\begin{theorem}
	\label{theorem:samplebound-Wasserstein-Gaussian-finiteRKHS}
	Assume Assumptions 1,2, and 5, and furthermore that $\dim(\H_K) < \infty$.
	Let $\Xbf = (x_i)_{i=1}^m$ be independently sampled from $(\X,\rho)$.
	For any $0 < \delta < 1$, with probability at least $1-\delta$,
	\begin{align}
	\W^2_2[\Ncal(\mu_{\Phi(\Xbf)}, C_{\Phi(\Xbf)}), \Ncal(\mu_{\Phi}, C_{\Phi})] \leq 
	\frac{4\kappa^2}{m\delta^2} + \frac{2\kappa^2 \sqrt{\dim(\H_K)}}{\sqrt{m}\delta}\left(3 + \frac{2}{\sqrt{m}\delta}\right).
	\end{align}
\end{theorem}

For $K(x,y) = \la x, y\ra$ on $\R^d$, $d\in \Nbb$, we then obtain the convergence of $\Ncal(\mu_{\Xbf}, C_{\Xbf})$ to 
the Gaussian measure $\Ncal(\mu,C)$ in the $2$-Wasserstein distance.

\begin{corollary}
	[\textbf{Convergence of empirical Gaussian measures in $2$-Wasserstein distance on $\R^d$}]
	\label{corollary:samplebounds-Wasserstein-GaussianRd}
	Let $\rho = \Ncal(\mu,C)$  on $\R^d$.
	Let $\Xbf = (x_i)_{i=1}^m$ be independently sampled from $(\R^d,\rho)$.
	For any $0 < \delta < 1$, with probability at least $1-\delta$,
	\begin{align}
	\W^2_2[\Ncal(\mu_{\Xbf}, C_{\Xbf}), \Ncal(\mu, C)] \leq 
	\frac{4\kappa^2}{m\delta^2} + \frac{2\kappa^2 \sqrt{d}}{\sqrt{m}\delta}\left(3 + \frac{2}{\sqrt{m}\delta}\right),
	\end{align}
	where 
	$\kappa = (2||C||^2_{\HS} + 4\la \mu, C\mu\ra + (\trace{C} +||\mu||^2)^2)^{1/4}$.
\end{corollary}
In Corollary \ref{corollary:samplebounds-Wasserstein-GaussianRd}, the convergence rate for $\W_2$ thus has the form $O\left(({\frac{d}{m}})^{1/4}\right)$.
This is exponentially faster than the worst case scenario $O(m^{-1/d})$ (\cite{dudley1969speed,weed2019sharp}).
As noted before, $\kappa$ may depend on the dimension $d$ if $\trace(C)$ and $||C||_{\HS}$ do, e.g. if $\rho = \Ncal(0, I_d)$, then $\kappa = (d^2 +2d)^{1/4}$. 

Since $\W_2$ is a metric on $\Pcal_2(\R^d)$, the triangle inequality immediately gives
\begin{corollary}
	[\textbf{Sample complexity for finite sample approximations of $2$-Wasserstein distance between Gaussian measures on $\R^d$}]
	\label{corollary:sample-complexity-Wasserstein-GaussianRd}
	Let $\rho_1 = \Ncal(\mu_1,C_1)$ and $\rho_2 = \Ncal(\mu_2, C_2)$  on $\R^d$.
	Let $\Xbf = (x_i)_{i=1}^m$ and $\Ybf = (y_j)_{j=1}^n$ be independently sampled from $(\R^d,\rho_1)$
	 and $(\R^d, \rho_2)$, respectively.
	For any $0 < \delta < 1$, with probability at least $1-\delta$,
	\begin{align}
	&\left|\W_2[\Ncal(\mu_{\Xbf}, C_{\Xbf}), \Ncal(\mu_{\Ybf}, C_{\Ybf})] - \W_2(\Ncal(\mu_1,C_1), \Ncal(\mu_2,C_2))\right|
	\nonumber
	\\
	& \leq 
	\sqrt{\frac{4\kappa_1^2}{m\delta^2} + \frac{2\kappa_1^2 \sqrt{d}}{\sqrt{m}\delta}\left(3 + \frac{2}{\sqrt{m}\delta}\right)}
	+	\sqrt{\frac{4\kappa_2^2}{n\delta^2} + \frac{2\kappa_2^2 \sqrt{d}}{\sqrt{n}\delta}\left(3 + \frac{2}{\sqrt{n}\delta}\right)}
	\end{align}
	where 
	$\kappa_i = (2||C_i||^2_{\HS} + 4\la \mu_i, C_i\mu_i\ra + (\trace{C_i} +||\mu_i||^2)^2)^{1/4}$, $i=1,2$.
\end{corollary}


{\bf On sample complexity of the infinite-dimensional Wasserstein distance}.
We now briefly discuss whether the analysis above can be applied to 
obtain sample complexity for the finite sample approximation
of $\W_2$ for infinite-dimensional Gaussian measures.
We have $\W_2(\Ncal(0, A_N), \Ncal(0,A)) \approach 0 \equivalent ||A_N - A||_{\tr} \approach 0$.
Here the convergence 
is with respect to the Banach trace class norm. While there are laws of large numbers similar to Proposition \ref{proposition:Pinelis} for Banach space-valued random variables  (see \cite{Pinelis1994optimum}) that hold for {\it $2$-smooth} Banach spaces,
it is not clear if they can be extended to the Banach space of trace class operators $\Tr(\H)$.
 
If we assume furthermore that $A_N^{1/2},A^{1/2} \in \Tr(\H)$, then
\begin{align}
||A_N-A||_{\tr} &= ||A_N^{1/2}(A_N^{1/2}-A^{1/2})||_{\tr} + ||(A_N^{1/2}-A^{1/2})A^{1/2}||_{\tr}
\nonumber
\\
& \leq (||A_N^{1/2}||_{\tr}+||A^{1/2}||_{\tr})||A_N^{1/2}-A^{1/2}||
\nonumber
\\
& \leq (||A_N^{1/2}||_{\tr}+||A^{1/2}||_{\tr})||A_N-A||^{1/2}\;\text{by Corollary \ref{corollary:continuity-norm-square-root}}
\nonumber
\\
& \leq (||A_N^{1/2}||_{\tr}+||A^{1/2}||_{\tr})||A_N-A||^{1/2}_{\HS}.
\end{align}
We can thus express the convergence $||A_N-A||_{\tr}$ in terms of the convergence in the Hilbert-Schmidt norm $||A_N-A||_{\HS}$. 
It is not clear, however, how $||A_N^{1/2}||_{\tr}$ and $||A^{1/2}||_{\tr}$ can bounded.
As an example, for $A= L_K$, the integral operator defined in Eq.\eqref{equation:LK}, 
it is not clear how to bound $\trace(L_K^{1/2})$ in terms of the kernel $K$ (note that the kernel for $L_K^{1/2}$ is {\it not} $K^{1/2}$).

\section{Numerical experiments}
\label{section:experiments}

Let us now illustrate the above theoretical analysis in the RKHS setting with the following numerical experiments.

{\bf Experiment 1}. We first empirically verify the theoretical results in Section 
\ref{section:bounded-kernels}. Specifically, we examine the convergence of the kernel Wasserstein distance and kernel Gaussian-Sinkhorn divergence between
mixtures of Gaussian densities on $\R^d$.
In the first experiment,
we generated two sequences of random data matrices $\Xbf$ and $\Ybf$ according to the following mixture of Gaussian distributions  in $\R^5$,
\begin{align}
\label{equation:Gaussian-mixture-1}
P = \frac{1}{2}\sum_{i=1}^2\Ncal(\mu_i, \Sigma_i),
\end{align}
\begin{equation}
\begin{aligned}
\mu_1 &= 
\begin{pmatrix}
-1.2700  & -0.4852  &  0.5943 &  -0.2765 &  -1.8576
\end{pmatrix},
\\
\Sigma_1 &=
\begin{pmatrix}
6.7558  & -1.2294  & -0.0491  &  0.6407 &  -1.6215\\
-1.2294  &  3.9189  & -2.3799 &  -3.6799 &   0.4207\\
-0.0491  & -2.3799   & 3.4696  &  2.9650  & -0.4756\\
0.6407   & -3.6799  &  2.9650   & 5.7615  &  2.3545\\
-1.6215  &  0.4207  & -0.4756   & 2.3545  &  4.2673
\end{pmatrix},
\\
\mu_2 &= 
\begin{pmatrix}
0.0407  &  0.2830 &   0.0636 &   0.4334  &  0.4229
\end{pmatrix},
\\
\Sigma_2 &= 
\begin{pmatrix}
4.7066  &  1.2666  &  0.0573   & 1.8804  &  1.7179\\
1.2666  &  6.5954  &  2.2573   & 2.2448  &  0.4295\\
0.0573  &  2.2573  &  7.4165   & 0.7357  & -0.1879\\
1.8804  &  2.2448  &  0.7357   & 2.3251  & -1.1107\\
1.7179  &  0.4295  & -0.1879  & -1.1107  &  5.3569
\end{pmatrix}.
\end{aligned}
\end{equation}
Both $\Xbf$ and $\Ybf$ have size $d \times m$, where  $d = 5$ and $m$ is the number of samples, with $m = 10,20,30, \ldots, 500$.
Let $K$ be the Laplacian kernel
$K(x,y) = \exp(-a||x-y||)$ on $\R^5 \times \R^5$, where $a =1$.
We then computed the kernel Wasserstein distance ($\epsilon =0$), kernel Gaussian-Sinkhorn divergence
($\epsilon = 0.01$ and $\epsilon = 0.1$), and MMD ($\epsilon = \infty$)
between the Gaussian measures 
$\Ncal(\mu_{\Phi(\Xbf)}, C_{\Phi(\Xbf)})$ and $\Ncal(\mu_{\Phi(\Ybf)}, C_{\Phi(\Ybf)})$ induced by the kernel $K$ on the RKHS $\H_K$, 
according to Theorem \ref{theorem:RKHS-distance}. The results are plotted in Figure \ref{figure:Gaussian-mixture-1-plot} (top).

{\bf Experiment 2}.
We repeated the experiment above, except that now $\Ybf$ is generated from the following mixture of Gaussian distributions
\begin{align}
\label{equation:Gaussian-mixture-2}
Q = \frac{1}{2}\sum_{i=3,4}\Ncal(\mu_i, \Sigma_i),
\end{align}
\begin{equation}
\begin{aligned}
\mu_3 & = 
\begin{pmatrix}
0.4170  &  0.7203 &   0.0001  &  0.3023 &   0.1468
\end{pmatrix},
\\
\Sigma_3 & =
\begin{pmatrix}
0.6104  &  0.6002  &  0.4951  &  1.1095 &   0.6525\\
0.6002  &  1.4589  &  0.8099  &  1.6952 &   0.6145\\
0.4951  &  0.8099  &  0.9948  &  1.3868 &   0.8549\\
1.1095  &  1.6952  &  1.3868  &  2.9243 &   1.6981\\
0.6525  &  0.6145  &  0.8549  &  1.6981 &   1.6091
\end{pmatrix},
\\
\mu_4 &=
\begin{pmatrix}
0.0983  &  0.4211  &  0.9579 &   0.5332 &   0.6919
\end{pmatrix},
\\
\Sigma_4 &=
\begin{pmatrix}
1.8305  &  1.1516 &   1.1129 &   1.0224  &  1.0691\\
1.1516  &  2.2499  &  1.4456  &  0.8467  &  1.0608\\
1.1129  &  1.4456  &  1.2121  &  0.8279  &  0.8492\\
1.0224  &  0.8467  &  0.8279  &  0.8181  &  0.9223\\
1.0691  &  1.0608  &  0.8492  &  0.9223  &  1.1909
\end{pmatrix}.
\end{aligned}
\end{equation}
The results are plotted in Figure \ref{figure:Gaussian-mixture-1-plot} (bottom).

\begin{figure}
	\begin{center}
		\includegraphics[width = 0.7\textwidth]{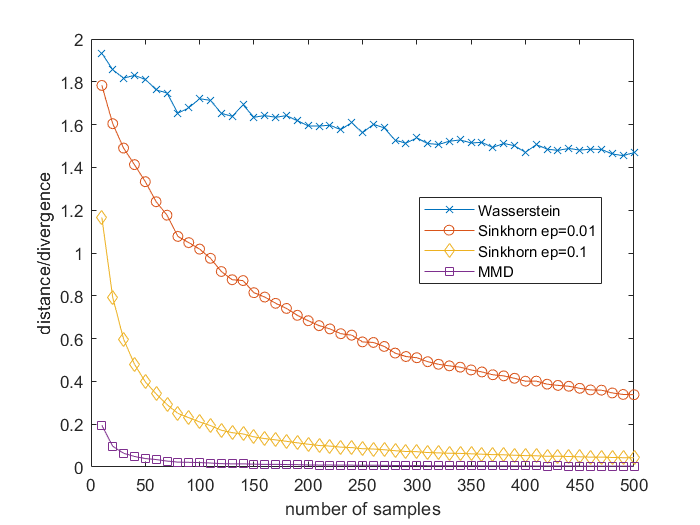}
		\includegraphics[width = 0.7\textwidth]{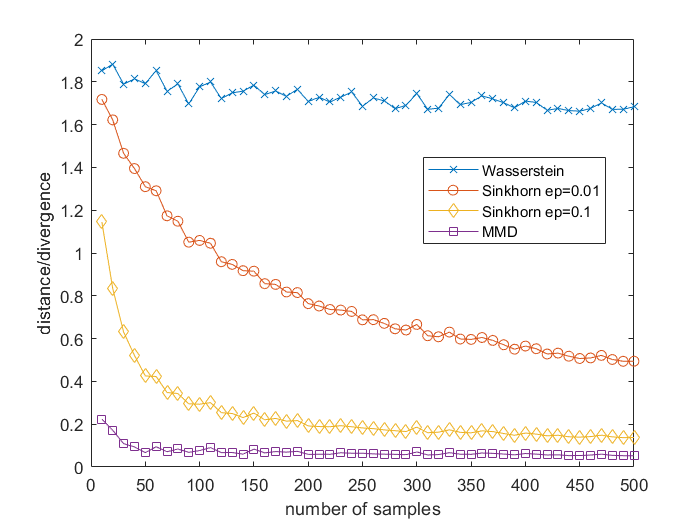}
		\caption{Distances/divergences between $\Ncal(\mu_{\Phi(\Xbf)}, C_{\Phi(\Xbf)})$ and $\Ncal(\mu_{\Phi(\Ybf)}, C_{\Phi(\Ybf)})$. 
			Top: $\Xbf,\Ybf$ are both generated by the mixture of Gaussian distributions in Eq.\eqref{equation:Gaussian-mixture-1}.
			Bottom: $\Xbf$ and $\Ybf$ are generated by two different mixtures of Gaussian distributions in Eq. \eqref{equation:Gaussian-mixture-1} and \eqref{equation:Gaussian-mixture-2}, respectively.
			Both cases use the {\it Laplacian kernel} $K(x,y) = \exp(-||x-y||)$ on $\R^5 \times \R^5$.	
		}
		\label{figure:Gaussian-mixture-1-plot}
	\end{center}
\end{figure}

{\bf Experiments 3}. We repeated Experiments 1 and 2, using the Gaussian kernel
$K(x,y) = \exp(-\frac{||x-y||^2}{\sigma^2})$, where $\sigma = 1$
(Figure \ref{figure:Gaussian-mixture-2-plot}).

\begin{figure}
	\begin{center}
		\includegraphics[width = 0.7\textwidth]{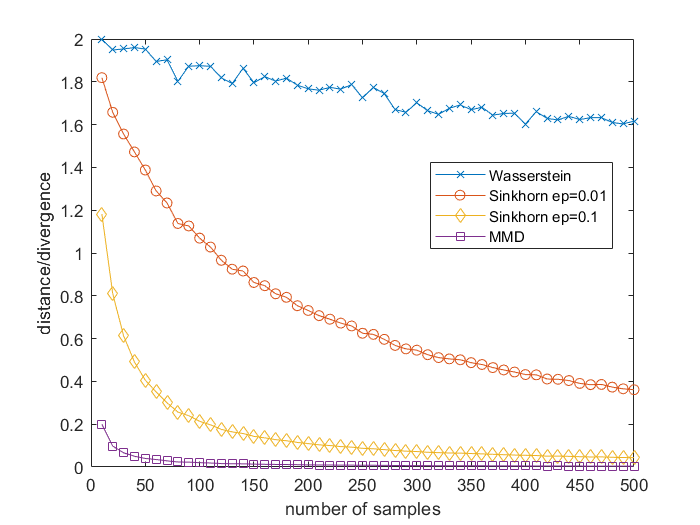}
		\includegraphics[width = 0.7\textwidth]{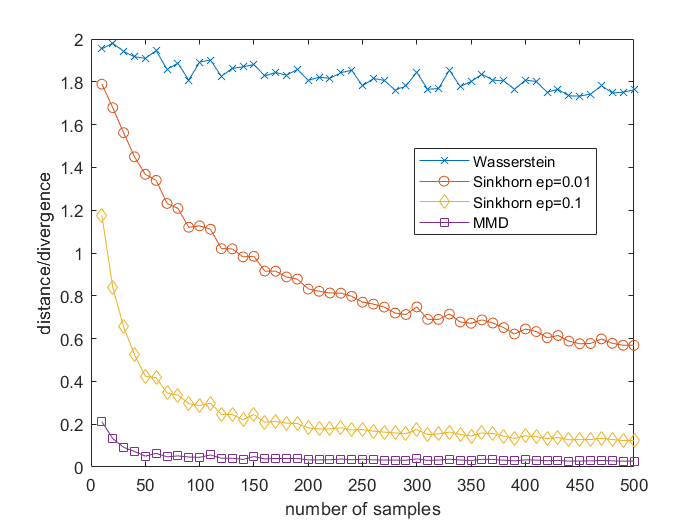}
		\caption{Distances/divergences between $\Ncal(\mu_{\Phi(\Xbf)}, C_{\Phi(\Xbf)})$ and $\Ncal(\mu_{\Phi(\Ybf)}, C_{\Phi(\Ybf)})$. 
			Top: $\Xbf,\Ybf$ are both generated by the mixture of Gaussian distributions in Eq.\eqref{equation:Gaussian-mixture-1}.
			Bottom: $\Xbf$ and $\Ybf$ are generated by two different mixtures of Gaussian distributions in Eq. \eqref{equation:Gaussian-mixture-1} and \eqref{equation:Gaussian-mixture-2}, respectively.
			Both cases use the {\it Gaussian kernel} $K(x,y) = \exp(-||x-y||^2)$ on $\R^5 \times \R^5$.	
		}
		\label{figure:Gaussian-mixture-2-plot}
	\end{center}
\end{figure}

{\bf Discussion of results}. In both Figure \ref{figure:Gaussian-mixture-1-plot} (top) and Figure \ref{figure:Gaussian-mixture-2-plot} (top), $\Xbf$ and $\Ybf$ are generated from the same probability distribution, the empirical distance/divergences are expected to converge to zero
as the sample size $m \approach \infty$. This is confirmed for all the distances/divergences plotted, with faster convergence 
for bigger values of $\epsilon$, in agreement with the theoretical analysis.
In particular, the kernel Gaussian-Sinkhorn divergence converges much more rapidly than the kernel Wasserstein divergence.
In Figure \ref{figure:Gaussian-mixture-1-plot} (bottom) and Figure \ref{figure:Gaussian-mixture-2-plot} (bottom),
since $\Xbf$ and $\Ybf$ are generated from two different probability distributions, the empirical distance/divergences are expected to converge to nonzero numbers
as 
$m \approach \infty$, as can be readily observed. We also see the dependence of the distances/divergences on the choice of kernels, as expected. 

{\bf Experiment 4}. 
We now empirically verify the results in Section \ref{section:general-kernels}. Specifically, 
we examine the convergence of the Wasserstein distance and Sinkhorn divergence between Gaussian measures on
$\R^d$. In the first set of experiments, we 
generated $\Xbf$ and $\Ybf$, each of size $d \times m$, where $d = 100$ and $m = 10,20,\ldots, 1000$,
by the same centered Gaussian distribution $\Ncal(0,C)$, on $\R^{100}$,
with $C = U^TU$, where the entries of $U$ are randomly generated by
$\Ncal(0,1)$
on $\R$. 
We then computed the $2$-Wasserstein distance and Sinkhorn divergence between $\Ncal(\mu_{\Xbf}, C_{\Xbf})$
and $\Ncal(\mu_{\Ybf}, C_{\Ybf})$. 
Since $||C||_{\HS}$ and $\trace(C)$ play an important role in the convergence rate in Corollary \ref{corollary:sample-complexity-Wasserstein-GaussianRd}, we ran a further experiment by normalizing $C$
so that $\trace(C) = 1$. The results are plotted in
Figure \ref{figure:Gaussians-Rd-plot-1}.

{\bf Experiment 5}. We repeated Experiment 4, except that $\Xbf$ and $\Ybf$ are generated by two different centered Gaussian distributions in $\R^{100}$
(Figure \ref{figure:Gaussians-Rd-plot-2}).

{\bf Experiment 6}. We repeated Experiment 4 in two scenarios: (i) $\Xbf,\Ybf$ are generated by the same centered Gaussian distributions in $\R^{100}$; (ii) $\Xbf, \Ybf$ are generated by the same centered Gaussian distributions on $\R^{500}$
(Figure \ref{figure:Gaussians-Rd-plot-3}).

{\bf Discussion of results}. 
Both Figures \ref{figure:Gaussians-Rd-plot-1} and \ref{figure:Gaussians-Rd-plot-2} show
that the Wasserstein distance and Sinkhorn divergence have similar convergence behavior when the sample
size $m$ is large. With $\trace(C)=1$, in Figure \ref{figure:Gaussians-Rd-plot-3} we observe the clear dimension-independent convergence behavior 
of the Sinkhorn divergence, in accordance with Corollary \ref{corollary:Sinkhorn-Gaussian-approximation-unbounded}, whereas the dimension-dependent convergence behavior
of the Wasserstein distance is 
in accordance with Corollary \ref{corollary:sample-complexity-Wasserstein-GaussianRd}.

\begin{figure}
	\begin{center}
		\includegraphics[width = 0.7\textwidth]{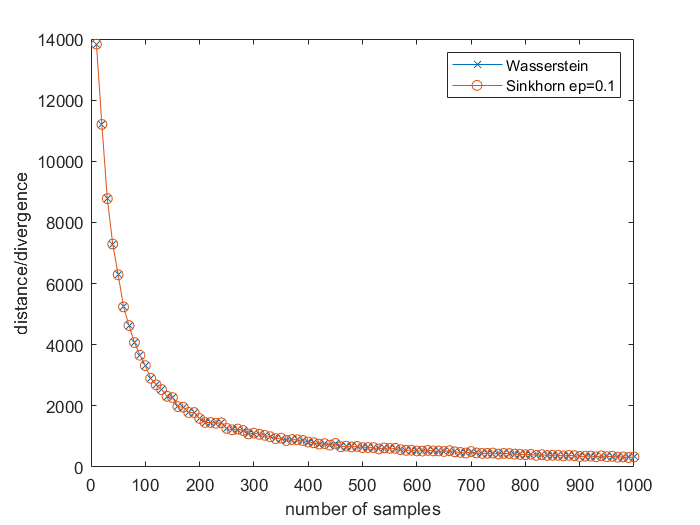}
		\includegraphics[width = 0.7\textwidth]{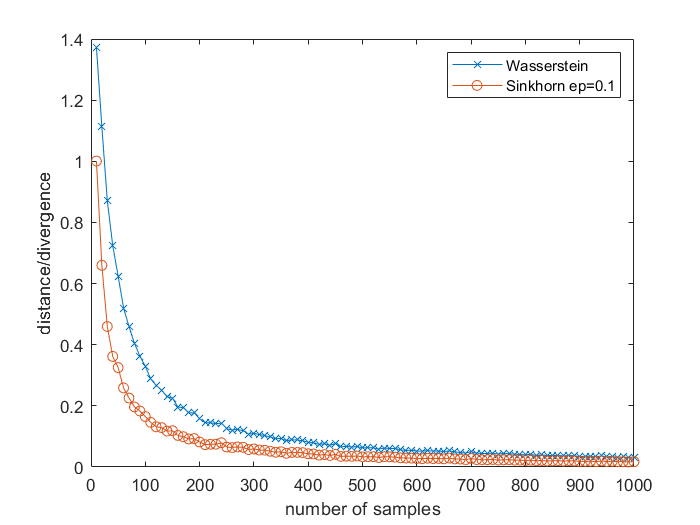}
		\caption{Distances/divergences between $\Ncal(\mu_{\Xbf}, C_{\Xbf})$ and $\Ncal(\mu_{\Ybf}, C_{\Ybf})$. 
			Top: $\Xbf,\Ybf$ are generated by the same centered Gaussian distribution in $\R^{100}$.
			Bottom: $\Xbf,\Ybf$ are generated by the same Gaussian distribution as above, but 
			with the trace of the covariance matrix normalized to $1$.
			Both cases are equivalent to using {\it linear kernel} $K(x,y) = \la x,y\ra$ on $\R^{100} \times \R^{100}$.	
		}
		\label{figure:Gaussians-Rd-plot-1}
	\end{center}
\end{figure}

\begin{figure}
	\begin{center}
		\includegraphics[width = 0.7\textwidth]{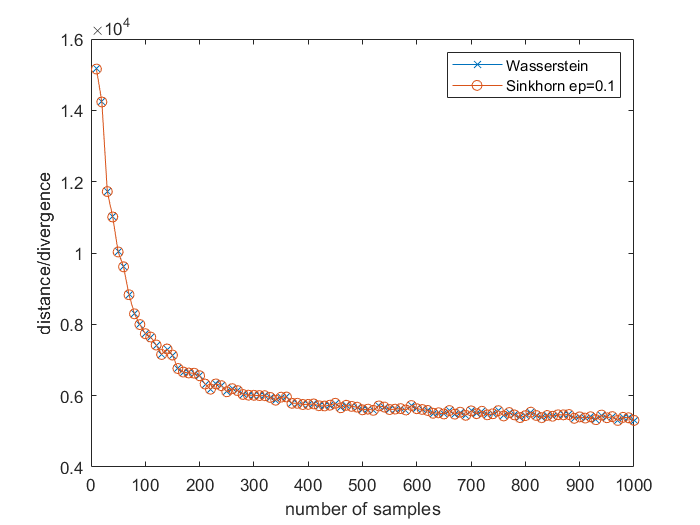}
		\includegraphics[width = 0.7\textwidth]{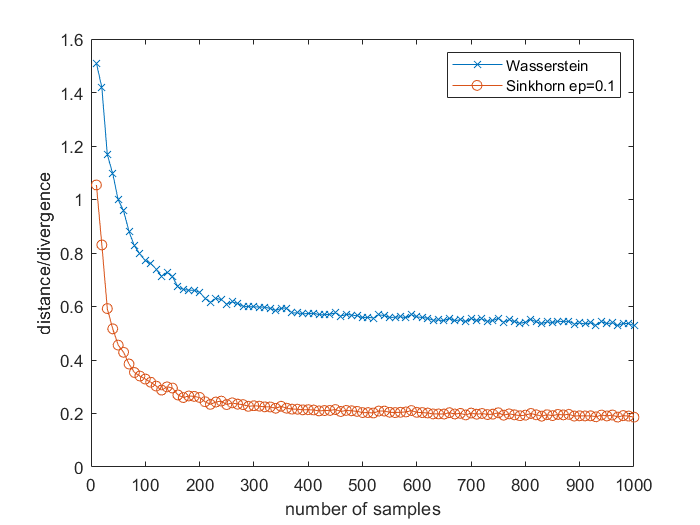}
		\caption{Distances/divergences between $\Ncal(\mu_{\Xbf}, C_{\Xbf})$ and $\Ncal(\mu_{\Ybf}, C_{\Ybf})$. 
			Top: $\Xbf,\Ybf$ are generated by the two centered Gaussian distributions in $\R^{100}$.
			Bottom: $\Xbf$ and $\Ybf$ are generated by the same centered Gaussian distributions in $\R^{100}$ as above, but 
			with the traces of the covariance matrices both normalized to $1$.
			Both cases are equivalent to using {\it linear kernel} $K(x,y) = \la x,y\ra$ on $\R^{100} \times \R^{100}$.	
		}
		\label{figure:Gaussians-Rd-plot-2}
	\end{center}
\end{figure}

\begin{figure}
	\begin{center}
		\includegraphics[width = 0.7\textwidth]{WassersteinSinkhorn_Gaussian_linearkernel_trace_1.png}
		\includegraphics[width = 0.7\textwidth]{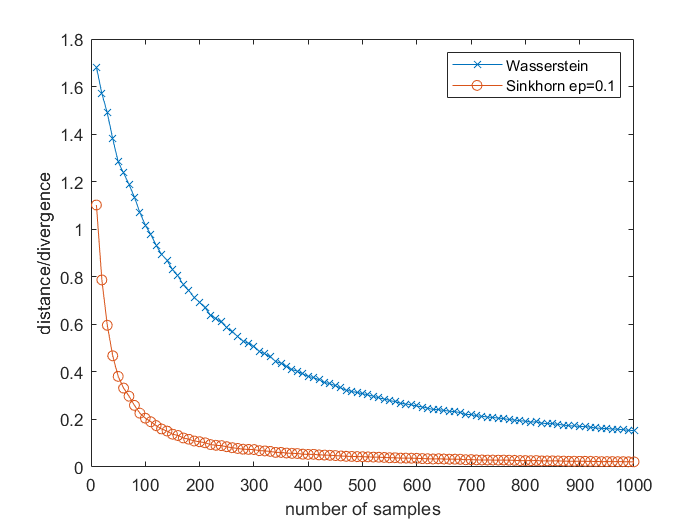}
		\caption{Distances/divergences between $\Ncal(\mu_{\Xbf}, C_{\Xbf})$ and $\Ncal(\mu_{\Ybf}, C_{\Ybf})$. 
			Top: $\Xbf,\Ybf$ are generated by the same centered Gaussian distribution in $\R^{100}$.
			Bottom: $\Xbf$ and $\Ybf$ are generated by the same centered Gaussian distribution in $\R^{500}$. In both case, 
			the traces of the covariance matrices are normalized to $1$.
			These cases are equivalent to using {\it linear kernel} $K(x,y) = \la x,y\ra$ on $\R^{100} \times \R^{100}$
			and $\R^{500} \times \R^{500}$, respectively.	
		}
		\label{figure:Gaussians-Rd-plot-3}
	\end{center}
\end{figure}

\section{Proofs of main results}
\label{section:proofs}

\subsection{Proofs for the general Hilbert space setting}
\label{section:proofs-Hilbert}

In this section, we prove Theorems \ref{theorem:convergence-Sinkhorn}, \ref{theorem:Sinkhorn-vs-Wasserstein},
\ref{theorem:OT-entropic-approx}, and
 \ref{theorem:Sinkhorn-entropic-approx}. We first need the following technical lemmas.

\begin{lemma}
	[\textbf{Corollary 3.2 in \cite{Kitta:InequalitiesV}}]
	\label{lemma:inequality-Schatten}
	For any two positive operators $A,B$ on $\H$ such that $A \geq cI > 0$, $B \geq cI > 0$,
	for any bounded operator $X$ on $\H$,
	\begin{equation}
	||A^rX - XB^r||_p\leq rc^{r-1}||AX-XB||_p, 0 < r \leq 1, 1 \leq p \leq \infty.
	\end{equation}
\end{lemma}
The following result is then immediate.
\begin{corollary}
	\label{corollary:trace-class-square-root}
	Let $\Csc_p(\H)$ denote the set of $p$th Schatten class operators on $\H$, $1 \leq p \leq \infty$.
	For two operators $A, B \in \Sym^{+}(\H) \cap \Csc_p(\H)$, 
	\begin{align}
	||(I+A)^{r} - (I+B)^{r}||_{p} \leq r||A-B||_{p}, 0\leq r \leq 1.
	\end{align}
\end{corollary}

\begin{lemma}
	\label{lemma:trace-norm-bound-ABA}
	Let $A \in \Sym^{+}(\H) \cap \HS(\H)$ and $B \in \Sym(\H) \cap \HS(\H)$. Then
	$A^{1/2}BA^{1/2} \in \Tr(\H)$ and
	\begin{align}
	||A^{1/2}BA^{1/2}||_{\tr} \leq ||A||_{\HS}||B||_{\HS}.
	\end{align}
\end{lemma}
\begin{proof}
	If $B \in \Sym^{+}(\H)$, then $A^{1/2}BA^{1/2} \in \Sym^{+}(\H)$, so that
	\begin{align*}
	0 \leq ||A^{1/2}BA^{1/2}||_{\tr} = \trace(A^{1/2}BA^{1/2}) = \trace(AB) = \la A,B\ra_{\HS} \leq ||A||_{\HS}||B||_{\HS}.
	\end{align*}
	If $B \in \Sym(\H)$, then $B = B_1 - B_2$, where $B_1 = \frac{|B|+B}{2} \geq 0$, $B_2 = \frac{|B|-B}{2}\geq 0$.
	Applying the previous result gives
	\begin{align*}
	&||A^{1/2}BA^{1/2}||_{\tr} = ||A^{1/2}(B_1 - B_2)A^{1/2}||_{\tr} \leq ||A^{1/2}B_1A^{1/2}||_{\tr} + ||A^{1/2}B_2A^{1/2}||_{\tr}
	\\
	& = \trace(A^{1/2}B_1A^{1/2}) + \trace(A^{1/2}B_2A^{1/2}) = \trace(AB_1) + \trace(AB_2) = \trace(A(B_1 + B_2))
	\\
	& = \trace(A|B|) \leq ||A||_{\HS}|||B|||_{\HS} = ||A||_{\HS}||B||_{\HS}.
	\end{align*}
	\qed
\end{proof}

\begin{lemma}
	\label{lemma:trace-bound-square-root-1}
	Let $A_N,B_N, A,B \in \Sym^{+}(\H)\cap \HS(\H)$. Then for $c\in \R$,
	\begin{align}
	&|\trace[-I + (I+c^2A_N^{1/2}B_NA_N^{1/2})^{1/2}] - \trace[-I + (I+c^2 A^{1/2}BA^{1/2})^{1/2}]|
	\nonumber
	\\
	&\quad \leq \frac{c^2}{2}\left(||A_N||_{\HS}||B_N-B||_{\HS} + ||B||_{\HS}||A_N-A||_{\HS}\right).
	\end{align}
\end{lemma}
\begin{proof} By Corollary \ref{corollary:trace-class-square-root}, with $r=1/2$,
	\begin{align*}
	&|\trace[-I + (I+c^2A_N^{1/2}B_NA_N^{1/2})^{1/2}] - \trace[-I + (I+c^2 A^{1/2}BA^{1/2})^{1/2}]|
	\\
	& \leq |\trace[-I + (I+c^2A_N^{1/2}B_NA_N^{1/2})^{1/2}] - \trace[-I + (I+c^2 A_N^{1/2}BA_N^{1/2})^{1/2}]|
	\\
	&\quad + |\trace[-I + (I+c^2A_N^{1/2}BA_N^{1/2})^{1/2}] - \trace[-I + (I+c^2 A^{1/2}BA^{1/2})^{1/2}]|
	\\
	&\leq ||(I+c^2A_N^{1/2}B_NA_N^{1/2})^{1/2} - (I+c^2 A_N^{1/2}BA_N^{1/2})^{1/2}||_{\trace}
	\\
	& \quad + |\trace[-I + (I+c^2B^{1/2}A_NB^{1/2})^{1/2}] - \trace[-I + (I+c^2 B^{1/2}AB^{1/2})^{1/2}]|
	\\
	& \leq ||(I+c^2A_N^{1/2}B_NA_N^{1/2})^{1/2} - (I+c^2 A_N^{1/2}BA_N^{1/2})^{1/2}||_{\trace}
	\\
	& \quad + ||(I+c^2B^{1/2}A_NB^{1/2})^{1/2} - (I+c^2 B^{1/2}AB^{1/2})^{1/2}||_{\trace}
	\\
	& \leq \frac{c^2}{2}||A_N^{1/2}B_NA_N^{1/2} - A_N^{1/2}BA_N^{1/2}||_{\tr} 
	+ \frac{c^2}{2}||B^{1/2}A_NB^{1/2} - B^{1/2}AB^{1/2}||_{\tr}
	\\
	& \leq \frac{c^2}{2}||A_N||_{\HS}||B_N-B||_{\HS} + \frac{c^2}{2}||B||_{\HS}||A_N-A||_{\HS}
	\end{align*}
	where the last inequality follows from Lemma \ref{lemma:trace-norm-bound-ABA}.
	\qed
\end{proof}
\begin{lemma}
	[\textbf{Theorem 3.2 in \cite{Simon:1977}}]
	\label{lemma:Fredholm-det-upperbound}
	For any $A \in \Tr(\H)$,
	\begin{equation}
	|\det(I+A)| \leq \exp(||A||_{\tr}).
	\end{equation}
\end{lemma}

\begin{corollary}
	\label{corollary:logdet-continuity-trace-norm}
	For $A,B  \in \Sym^{+}(\H) \cap \Tr(\H)$. Then
	\begin{align}
	&|\log\det(I+A) - \log\det(I+B)| \leq ||A-B||_{\tr},
	\\
	&\left|\log\det\left(\frac{1}{2}I + \frac{1}{2}(I+A)^{1/2}\right) - \log\det\left(\frac{1}{2}I + \frac{1}{2}(I+B)^{1/2}\right)\right|
	\leq \frac{1}{4}||A-B||_{\tr}.
	\end{align}
\end{corollary}
\begin{proof} For the first part, by Lemma \ref{lemma:Fredholm-det-upperbound} and the fact that the $\log$ function is strictly increasing 
	on $(0, \infty)$, we have
	\begin{align*}
	&\log\det(I+A) - \log\det(I+B) = \log\det[(I+A)(I+B)^{-1}]
	\\
	&= \log\det[I + (A-B)(I+B)^{-1}]
	\leq \log\exp(||(A-B)(I+B)^{-1}||_{\tr} 
	\\
	&=||(A-B)(I+B)^{-1}||_{\tr} \leq ||A-B||_{\tr}||(I+B)^{-1}|| \leq ||A-B||_{\tr}.
	\end{align*}
	Similarly,
	$\log\det(I+B) - \log\det(I+A) \leq ||(B-A)(I+A)^{-1}||_{\tr} \leq ||A-B||_{\tr}$.
	Thus it follows that
	$|\log\det(I+A) - \log\det(I+B)| \leq ||A-B||_{\tr}$.
	
	For the second part, let $A_1 = -\frac{1}{2}I + \frac{1}{2}(I+A)^{1/2}$, $B_1 = -\frac{1}{2}I + \frac{1}{2}(I+B)^{1/2}$,
	then using the first inequality,
	\begin{align*}
	&\left|\log\det(\frac{1}{2}I + \frac{1}{2}(I+A)^{1/2}) - \log\det(\frac{1}{2}I + \frac{1}{2}(I+B)^{1/2})\right|
	\\
	& = |\log\det(I+A_1) - \log\det(I+B_1)| \leq ||A_1 - B_1||_{\tr} 
	\\
	&=\frac{1}{2}||(I+A)^{1/2} - (I+B)^{1/2}||_{\tr} \leq \frac{1}{4}||A-B||_{\tr}.
	\end{align*} 
	by Corollary \ref{corollary:trace-class-square-root}. \qed
\end{proof}

\begin{lemma}
	\label{lemma:logdet-bound-square-root-1}
	Let $A_N,B_N, A,B \in \Sym^{+}(\H)\cap \HS(\H)$. Then
	\begin{align}
	&\left|\log\det\left[\frac{1}{2} + \frac{1}{2}(I+c^2A_N^{1/2}B_NA_N^{1/2})^{1/2}\right] - \log\det\left[\frac{1}{2} + \frac{1}{2}(I+c^2 A^{1/2}BA^{1/2})^{1/2}\right]\right|
	\nonumber
	\\
	&\leq \frac{c^2}{4}\left(||A_N||_{\HS}||B_N-B||_{\HS} + ||B||_{\HS}||A_N-A||_{\HS}\right).
	\end{align}
\end{lemma}
\begin{proof}
	By Corollary \ref{corollary:logdet-continuity-trace-norm},
	\begin{align*}
	&\left|\log\det\left[\frac{1}{2} + \frac{1}{2}(I+c^2A_N^{1/2}B_NA_N^{1/2})^{1/2}\right] - \log\det\left[\frac{1}{2} + \frac{1}{2}(I+c^2 A^{1/2}BA^{1/2})^{1/2}\right]\right|
	\\
	& \leq \left|\log\det\left[\frac{1}{2} + \frac{1}{2}(I+c^2A_N^{1/2}B_NA_N^{1/2})^{1/2}\right] - \log\det\left[\frac{1}{2} + \frac{1}{2}(I+c^2 A_N^{1/2}BA_N^{1/2})^{1/2}\right]\right|
	\\
	&\quad + \left|\log\det\left[\frac{1}{2} + \frac{1}{2}(I+c^2A_N^{1/2}BA_N^{1/2})^{1/2}\right] - \log\det\left[\frac{1}{2} + \frac{1}{2}(I+c^2 A^{1/2}BA^{1/2})^{1/2}\right]\right|
	\\
	& = \left|\log\det\left[\frac{1}{2} + \frac{1}{2}(I+c^2A_N^{1/2}B_NA_N^{1/2})^{1/2}\right] - \log\det\left[\frac{1}{2} + \frac{1}{2}(I+c^2 A_N^{1/2}BA_N^{1/2})^{1/2}\right]\right|
	\\
	&\quad + \left|\log\det\left[\frac{1}{2} + \frac{1}{2}(I+c^2B^{1/2}A_NB^{1/2})^{1/2}\right] - \log\det\left[\frac{1}{2} + \frac{1}{2}(I+c^2 B^{1/2}AB^{1/2})^{1/2}\right]\right|
	\\
	& \leq \frac{c^2}{4}||A_N^{1/2}B_NA_N^{1/2} - A_N^{1/2}BAN^{1/2}||_{\tr} + \frac{c^2}{4}||B^{1/2}A_NB^{1/2}-B^{1/2}AB^{1/2}||_{\tr}
	\\ 
	& \leq \frac{c^2}{4}||A_N||_{\HS}||B_N-B||_{\HS} + \frac{c^2}{4}||B||_{\HS}||A_N-A||_{\HS}
	\end{align*}
	where the last inequality follows from Lemma \ref{lemma:trace-norm-bound-ABA}. 
	\qed
\end{proof}

\begin{lemma}
\label{lemma:norm-trace-HS-square-root}
	Let $A,B \in \Sym^{+}(\H) \cap \Tr(\H)$. Then
	\begin{align}
	||A^{1/2}- B^{1/2}||^2_{\HS} \leq ||A-B||_{\tr} \leq [||A^{1/2}||_{\HS} + ||B^{1/2}||_{\HS}]||A^{1/2}-B^{1/2}||_{\HS}.
	\end{align}
	In particular, for $\{A_N\}_{N \in \Nbb}, A\in \Sym^{+}(\H)\cap \Tr(\H)$,
	\begin{align}
	\lim_{N \approach \infty}||A_N-A||_{\tr} = 0 \equivalent \lim_{N \approach \infty}||A_N^{1/2}-A^{1/2}||_{\HS} = 0.
	\end{align}
\end{lemma}
\begin{proof} The first part of the inequality is from \cite{Powers1970free} (Lemma 4.1), namely
	\begin{align}
	||A^{1/2}- B^{1/2}||^2_{\HS} \leq ||A-B||_{\tr}.
	\end{align}
	The second part follows from the property $||AB||_{\tr}\leq ||A||_{\HS}||B||_{\HS}$ (e.g. \cite{ReedSimon:Functional}),
	\begin{align*}
	||A-B||_{\tr}&\leq ||A^{1/2}(A^{1/2} - B^{1/2})||_{\tr} +||(A^{1/2}-B^{1/2})B^{1/2}||_{\tr} 
	\\
	&\leq [||A^{1/2}||_{\HS} + ||B^{1/2}||_{\HS}]||A^{1/2}-B^{1/2}||_{\HS}.
	\end{align*}
	\qed
\end{proof}
\begin{lemma}
	\label{lemma:trace-norm-A2-B2}
	Let $A,B \in \HS(\H)$. Then
	\begin{align}
	||A^2-B^2||_{\tr} \leq [||A||_{\HS} + ||B||_{\HS}]||A-B||_{\HS}.
	\end{align}
\end{lemma}
\begin{proof} This follows from the property that $A,B \in \HS(\H) \imply AB \in \Tr(\H)$, with
	$||AB||_{\tr} \leq ||A||_{\HS}||B||_{\HS}$ (see e.g. \cite{ReedSimon:Functional}), so that
	\begin{align*}
	&||A^2 - B^2||_{\tr} = ||A^2 - AB + AB - B^2||_{\tr} \leq ||A(A-B)||_{\tr} + ||(A-B)B||_{\tr} 
	\\
	&\leq ||A||_{\HS}||A-B||_{\HS} + ||A-B||_{\HS}||B||_{\HS}.
	\end{align*}
	\qed
\end{proof}
\begin{lemma}
	\label{lemma:trace-bound-square-root-2}
	Let $A,B \in \Sym(\H)\cap \HS(\H)$. Then for $c \in \R$,
	\begin{align}
	&|\trace[-I + (I+c^2A_N^2)^{1/2}] - \trace[-I+(I+c^2A^2)^{1/2}]| 
	\nonumber
	\\
	&\quad \leq \frac{c^2}{2}[||A_N||_{\HS}+||A||_{\HS}]||A_N - A||_{\HS}.
	\end{align}
\end{lemma}
\begin{proof} By Corollary \ref{corollary:trace-class-square-root} and Lemma \ref{lemma:trace-norm-A2-B2},
	\begin{align*}
	&|\trace[-I + (I+c^2A_N^2)^{1/2}] - \trace[-I+(I+c^2A^2)^{1/2}]|
	\\
	& \leq || (I+c^2A_N^2)^{1/2} - (I+c^2A^2)^{1/2}||_{\tr}
	\\
	& \leq \frac{c^2}{2}||A_N^2 - A^2||_{\tr}\leq \frac{c^2}{2}[||A_N||_{\HS}+||A||_{\HS}]||A_N - A||_{\HS}.
	\end{align*}
	\qed
\end{proof}

\begin{lemma}
	\label{lemma:logdet-bound-square-root-2}
	Let $A,B \in \Sym(\H)\cap \HS(\H)$. Then for $c\in \R$,
	\begin{align}
	&\left|\logdet\left[\frac{1}{2}I + \frac{1}{2}(I+c^2A_N^2)^{1/2}\right] - \logdet\left[\frac{1}{2}I+\frac{1}{2}(I+c^2A^2)^{1/2}\right]\right| 
	\nonumber
	\\
	&\quad \leq \frac{c^2}{4}[||A_N||_{\HS}+||A||_{\HS}]||A_N - A||_{\HS}.
	\end{align}
\end{lemma}
\begin{proof}
	By Corollary \ref{corollary:logdet-continuity-trace-norm} and Lemma \ref{lemma:trace-norm-A2-B2},
	\begin{align*}
	&\left|\logdet\left[\frac{1}{2}I + \frac{1}{2}(I+c^2A_N^2)^{1/2}\right] - \logdet\left[\frac{1}{2}I+\frac{1}{2}(I+c^2A^2)^{1/2}\right]\right| 
	\\
	& \leq \frac{c^2}{4}||A_N^2 - A^2||_{\tr} \leq \frac{c^2}{4}[||A_N||_{\HS}+||A||_{\HS}]||A_N - A||_{\HS}.
	\end{align*}
	\qed
\end{proof}

\begin{lemma}
	\label{lemma:trace-bound-square-root-3}
	Let $A_N, A \in \Sym^{+}(\H) \cap \HS(\H)$. Then for $c \in \R$,
	\begin{align}
	&\left|\trace[-I +(I + c^2A_N^{1/2}AA_N^{1/2})^{1/2}] -\trace[-I + (I + c^2A_N^2)^{1/2}]\right|
	\nonumber
	\\
	&\quad \leq \frac{c^2}{2}||A_N||_{\HS}||A_N - A||_{\HS},
	\\
	&\left|\trace[-I +(I + c^2A_N^{1/2}AA_N^{1/2})^{1/2}] -\trace[-I + (I + c^2A^2)^{1/2}]\right| 
	\nonumber
	\\
	&\quad \leq \frac{c^2}{2}||A||_{\HS}||A_N - A||_{\HS}.
	\end{align}
\end{lemma}
\begin{proof}
	By Corollary \ref{corollary:trace-class-square-root} and Lemma \ref{lemma:trace-norm-bound-ABA},
	\begin{align*}
	&\left|\trace[-I +(I + c^2A_N^{1/2}AA_N^{1/2})^{1/2}] -\trace[-I + (I + c^2A_N^2)^{1/2}]\right|
	\\
	& \leq ||(I + c^2A_N^{1/2}AA_N^{1/2})^{1/2} - (I + c^2A_N^2)^{1/2}||_{\tr}
	\leq 
	\frac{c^2}{2}||A_N^{1/2}AA_N^{1/2} - A_N^2||_{\tr} 
	\\
	& = \frac{c^2}{2}||A_N^{1/2}(A-A_N)A_N^{1/2}||_{\tr}
	\leq \frac{c^2}{2}||A_N||_{\HS}||A_N-A||_{\HS}.
	\end{align*}
	The second result follows similarly, by noting that $\trace[-I +(I + c^2A_N^{1/2}AA_N^{1/2})^{1/2}]
	= \trace[-I +(I + c^2A^{1/2}A_NA^{1/2})^{1/2}]$.
	\qed
\end{proof}

\begin{lemma}
	\label{lemma:logdet-bound-square-root-3}
	Let $A_N, A \in \Sym^{+}(\H) \cap \HS(\H)$. Then for $c \in \R$,
	\begin{align}
	&\left|\logdet\left(\frac{1}{2} + \frac{1}{2}(I+c^2 A_N^{1/2}AA_N^{1/2})^{1/2}\right)
	- \logdet\left(\frac{1}{2} + \frac{1}{2}(I+c^2 A_N^2)^{1/2}\right)\right|
	\nonumber
	\\
	&\leq \frac{c^2}{4}||A_N||_{\HS}||A_N - A||_{\HS},
	\\
	&\left|\logdet\left(\frac{1}{2} + \frac{1}{2}(I+c^2 A_N^{1/2}AA_N^{1/2})^{1/2}\right)
	- \logdet\left(\frac{1}{2} + \frac{1}{2}(I+c^2 A^2)^{1/2}\right)\right|
	\nonumber
	\\
	&\leq \frac{c^2}{4}||A||_{\HS}||A_N - A||_{\HS}.
	\end{align}
\end{lemma}
\begin{proof}
	By Corollary \ref{corollary:logdet-continuity-trace-norm} and Lemma \ref{lemma:trace-norm-bound-ABA},
	\begin{align*}
	&\left|\logdet\left(\frac{1}{2} + \frac{1}{2}(I+c^2 A_N^{1/2}AA_N^{1/2})^{1/2}\right)
	- \logdet\left(\frac{1}{2} + \frac{1}{2}(I+c^2 A_N^2)^{1/2}\right)\right|
	\\
	& \leq \frac{c^2}{4}||A_N^{1/2}AA_N^{1/2} - A_N^2||_{\tr} \leq \frac{c^2}{4}||A_N||_{\HS}||A_N - A||_{\HS}.
	\end{align*}
	Since $\logdet\left(\frac{1}{2} + \frac{1}{2}(I+c^2 A_N^{1/2}AA_N^{1/2})^{1/2}\right) = \logdet\left(\frac{1}{2} + \frac{1}{2}(I+c^2 A^{1/2}A_NA{1/2})^{1/2}\right)$, the second inequality follows similarly. 
	\qed
\end{proof}

\begin{theorem}
	[\textbf{Theorem 2.3 in \cite{Kitta:InequalitiesV}}]
	Let $A,B$ be two positive operators on $\H$ and $f$ any operator monotone function with $f(0) = 0$. 
	Then
	\begin{equation}
	||f(A) - f(B)|| \leq f(||A-B||).
	\end{equation}
\end{theorem}
The following result is then immediate.
\begin{corollary}
	\label{corollary:continuity-norm-square-root}
	Let $A,B$ be two positive operators on $\H$. Then
	\begin{equation}
	||A^{r} - B^{r}|| \leq ||A-B||^{r}, \;\;\; 0 < r \leq 1.
	\end{equation}
\end{corollary}

\begin{theorem}
	\label{theorem:FS-AB-HS}
	Define the following function $F_S: \Sym^{+}(\H) \cap \HS(\H) \times \Sym^{+}(\H) \cap \HS(\H) \mapto \R$ by
	\begin{align}
	F_S(A,B) &= \trace[M(A,A) - 2M(A,B) + M(B,B)] 
	\nonumber
	\\
	&\quad +\logdet\left[\frac{(I +\frac{1}{2}M(A,B))^2}{(I+\frac{1}{2}M(A,A))(I+\frac{1}{2}M(B,B))}\right],
	\end{align}
	where $M(A,B) = -I + (I + c^2 A^{1/2}BA^{1/2})^{1/2}$, $c \in \R$.
	Then 
	\begin{align}
	|F_S(A,B)| \leq \frac{3c^2}{4}[||A||_{\HS} + ||B||_{\HS}]||A-B||_{\HS}.
	\end{align}
\end{theorem}
\begin{proof}
	[\textbf{Proof of Theorem \ref{theorem:FS-AB-HS}}]
	Combining Lemmas \ref{lemma:trace-bound-square-root-3} and \ref{lemma:logdet-bound-square-root-3}, we have
	\begin{align*}
	&|F_S(A,B)| \leq |\trace(M(A,A)) - \trace(M(A,B))| + |\trace(M(A,B)) - \trace(M(B,B))|
	\\
	&\quad + |\logdet(I+\frac{1}{2}M(A,B)) - \logdet(I+\frac{1}{2}M(A,A))|
	\\
	&\quad +  |\logdet(I+\frac{1}{2}M(A,B)) - \logdet(I+\frac{1}{2}M(B,B))|
	\\
	& \leq \frac{c^2}{2}||A||_{\HS}||A-B||_{\HS} + \frac{c^2}{2}||B||_{\HS}||A-B||_{\HS}
	\\
	& \quad + \frac{c^2}{4}||A||_{\HS}||A-B||_{\HS} + \frac{c^2}{4}||B||_{\HS}||A-B||_{\HS}
	\\
	& = \frac{3c^2}{4}[||A||_{\HS}+||B||_{\HS}]||A-B||_{\HS}. 
	\end{align*}
\end{proof}

\begin{proof}
	[\textbf{Proof of Theorem \ref{theorem:convergence-Sinkhorn}}]
	This follows from Theorem \ref{theorem:FS-AB-HS}, where $\Srm^{\ep}_{d^2}[\Ncal(0,A_N), \Ncal(0,A)] = 
	\frac{\ep}{4}F_S(A_N,A)$, with $c = c_{\ep} = \frac{4}{\ep}$.
	\qed
\end{proof}

\begin{proof}
	[\textbf{Proof of Theorem \ref{theorem:Sinkhorn-vs-Wasserstein}}]
	It suffices to focus on the case $m_N = m= 0$.
	Then $\lim_{N \approach \infty}\W_2(\Ncal(0,A_N), \Ncal(0,A)) = 0 \equivalent \lim_{N \approach \infty}||A_N - A||_{\tr} = 0
	\imply \lim_{N \approach \infty}||A_N -A||_{\HS} = 0 \imply \lim_{N \approach \infty}\Srm^{\ep}_{d^2}(\Ncal(0,A_N), \Ncal(0,A)) = 0$.
	
	(i) Let us now construct a sequence $\{A_N\}_{N \in \Nbb} , A\in \Sym^{+}(\H) \cap \Tr(\H)$
	such that 
	\begin{align*}
	\lim_{N \approach \infty}||A_N - A||_{\HS} &= 0, \;\; \lim_{N \approach \infty}||A_N - A||_{\tr} \neq 0,
	\\
	\lim_{N \approach \infty}\Srm^{\ep}_{d^2}[\Ncal(0,A_N), \Ncal(0,A)] &= 0, \; \lim_{N \approach \infty}\W_2[\Ncal(0,A_N), |\Ncal(0,A)] \neq 0.
	\end{align*}
	Let $\{e_k\}_{k \in \Nbb}$ be any orthonormal basis in $\H$.
	Let $A_N = \frac{1}{N}\sum_{i=1}^Ne_k \otimes e_k$, then $A_N \in \Tr(\H)$ with $\trace(A_N) = 1$ $\forall N \in \Nbb$.
	Let $A=0$, then
	\begin{align}
	||A_N-A||_{\tr} =1 \; \forall N \in \Nbb, ||A_N - A||_{\HS}^2 = \frac{1}{N} \approach 0 \text{ as $N \approach \infty$}.
	\end{align}
	The exact $2$-Wasserstein distance between $\Ncal(0,A_N)$ and $\Ncal(0,A)$ is, $\forall N \in \Nbb$,
	\begin{align}
	\W^2_2[\Ncal(0,A_N), \Ncal(0,A)] = \trace(A_N)+\trace(A) - 2\trace((A_N^{1/2}AA_N^{1/2})^{1/2}) = 1.
	\end{align}
	With $M(A,B)$ as
	in Theorem \ref{theorem:FS-AB-HS}, we have $M(A_N,A) = -I + (I+c_{\ep}^2A_N^{1/2}AA_N^{1/2})^{1/2} = 0$, $c_{\ep} = \frac{4}{\ep}$, thus
	the entropic $2$-Wasserstein distance is, $\forall N \in \Nbb$,
	\begin{align}
	&\OT^{\ep}_{d^2}[\Ncal(0,A_N), \Ncal(0,A)]
	\\
	& =\trace(A_N) + \trace(A) -\frac{\ep}{2}\trace(M(A_N,A)) + \frac{\ep}{2}\logdet\left(I + \frac{1}{2}M(A_N,A)\right) = 1. 
	\nonumber
	\end{align}
	Since $M(A_N,A) = M(A,A) = 0$, we have
	\begin{align*}
	\Srm^{\ep}_{d^2}[\Ncal(0,A_N), \Ncal(0,A)] = \frac{\ep}{4}\trace[M(A_N,A_N)] - \frac{\ep}{4}\logdet\left(I+\frac{1}{2}M(A_N,A_N)\right).
	\end{align*}
	With $M(A_N,A_N) = -I +(I +c_{\ep}^2A_N^2)^{1/2}$ having $N$ nonzero eigenvalues,
	\begin{align*}
	&\trace(M(A_N,A_N)) = N\left(-1+\left(1+\frac{c_{\ep}^2}{N^2}\right)^{1/2}\right) = \frac{c_{\ep}^2}{N}\left(1+\left(1+\frac{c_{\ep}^2}{N^2}\right)^{1/2}\right)^{-1},
	\\
	&\logdet\left(I + \frac{1}{2}M(A_N,A_N)\right) = N\log\left(\frac{1}{2} + \frac{1}{2}\left(1+\frac{c_{\ep}^2}{N^2}\right)^{1/2}\right).
	\end{align*}
	It is clear that $\lim_{N \approach \infty}\trace(M(A_N,A_N)) = 0$.
	Applying L'Hopital's rule gives $\lim_{N \approach \infty}\logdet\left(I + \frac{1}{2}M(A_N,A_N)\right) = 0$.
	Thus
	$\lim_{N \approach \infty}\Srm^{\ep}_{d^2}(\Ncal(0,A_N), \Ncal(0,A)) = 0$ as we claimed.
	
	(ii) Next, we construct a sequence $\{A_N\}_{N \in \Nbb} , A\in \Sym^{+}(\H) \cap \Tr(\H)$
	such that 
	\begin{align*}
	\lim_{N \approach \infty}||A_N - A||_{\HS} &\neq  0, \;\; \lim_{N \approach \infty}||A_N - A|| = 0,
	\\
	\lim_{N \approach \infty}\Srm^{\ep}_{d^2}[\Ncal(0,A_N), \Ncal(0,A)] &\neq 0.
	\end{align*}
	We proceed as in part (i), with $A=0$, but $A_N = \frac{1}{\sqrt{N}}\sum_{k=1}^Ne_k \otimes e_k$.
	Then $||A_N-A||= \frac{1}{\sqrt{N}} \approach 0$ but $||A_N - A||_{\HS} = 1$ $\forall N \in \Nbb$. As $N \approach \infty$,
	\begin{align*}
	&\trace(M(A_N,A_N)) = N\left(-1+\left(1+\frac{c_{\ep}^2}{N}\right)^{1/2}\right) = {c_{\ep}^2}\left(1+\left(1+\frac{c_{\ep}^2}{N}\right)^{1/2}\right)^{-1} \approach \frac{c_{\ep}^2}{2},
	\\
	&\logdet\left(I + \frac{1}{2}M(A_N,A_N)\right) = N\log\left(\frac{1}{2} + \frac{1}{2}\left(1+\frac{c_{\ep}^2}{N}\right)^{1/2}\right) \approach \frac{c_{\ep}^2}{4},
	\end{align*}
	with the second limit following from L'Hopital's rule. Thus
	\begin{align*}
	\lim_{N \approach \infty}\Srm^{\ep}_{d^2}[\Ncal(0,A_N), \Ncal(0,A)]
	 = \frac{\ep}{4}\frac{c^2_{\ep}}{4} = \frac{4}{\ep} > 0
	\end{align*}
	for all $\ep > 0$. 
	\qed
\end{proof}

\begin{proof}
	[\textbf{Proof of Theorem \ref{theorem:OT-entropic-approx}}]
	We have
	\begin{align*}
	&\left|\OT^{\ep}_{d^2}[\Ncal(m_{A,N}, A_N), \Ncal(m_{B,N},B_N)] - \OT^{\ep}_{d^2}[\Ncal(m_A, A), \Ncal(m_B,B)]\right|
	\\ 
	&\leq \left|||m_{A,N} - m_{B,N}||^2 - ||m_A-m_B||^2\right|
	\\
	& \quad +\left|\OT^{\ep}_{d^2}[\Ncal(0, A_N), \Ncal(0,B_N)] - \OT^{\ep}_{d^2}[\Ncal(0, A), \Ncal(0,B)]\right|.
	\end{align*}
	For the first term involving the means, by the triangle inequality
	\begin{align*}
	&\left|||m_{A,N} - m_{B,N}||^2 - ||m_A-m_B||^2\right|
	\\ 
	&= \left|||m_{A,N}- m_{B,N}|| - ||m_A - m_B||\right|[||m_{A,N} - m_{B,N}|| + ||m_A - m_{B}||]
	\\
	& \leq ||(m_{A,N} - m_{B,N}) - (m_A - m_B)||\;[||m_{A,N}|| + ||m_{B,N}|| + ||m_A|| + ||m_B||]
	\\
	& \leq [||m_{A,N} - m_A|| + ||m_{B,N} - m_{B}||]\;[||m_{A,N}|| + ||m_{B,N}|| + ||m_A|| + ||m_B||].
	\end{align*}
	
	For the second terms with the centered Gaussians,
	let $M(A,B) = -I + (I+ c_{\ep}^2A^{1/2}BA^{1/2})^{1/2}$, $c_{\ep} = \frac{4}{\ep}$, then we have
	\begin{align*}
	&\left|\OT^{\ep}_{d^2}[\Ncal(0, A_N), \Ncal(0,B_N)] - \OT^{\ep}_{d^2}[\Ncal(0, A), \Ncal(0,B)]\right|
	\\
	&\leq |\trace(A_N) - \trace(A)| + |\trace(B_N) - \trace(B)| + \frac{\ep}{2}|\trace(M(A_N,B_N)) - \trace(M(A,B))|
	\\
	&+ \frac{\ep}{2}\left|\logdet\left(I + \frac{1}{2}M(A_N,B_N)\right) - \logdet\left(I + \frac{1}{2}M(A,B)\right)\right|.
	\end{align*}
	For the first two terms, we have $|\trace(A_N) - \trace(A)| = |\trace(A_N-A)| \leq \trace|A_N-A| = ||A_N-A||_{\tr}$
	and similarly $|\trace(B_N)-\trace(B)|\leq ||B_N-B||_{\tr}$.
	For the third term, by Lemma \ref{lemma:trace-bound-square-root-1},
	\begin{align*}
	&|\trace(M(A_N,B_N)) - \trace(M(A,B))|
	\\ 
	&\leq \frac{c_{\ep}^2}{2}
	\left(||A_N||_{\HS}||B_N-B||_{\HS} + ||B||_{\HS}||A_N-A||_{\HS}\right).
	\end{align*}
	For the fourth term, by Lemma \ref{lemma:logdet-bound-square-root-1},
	\begin{align*}
	&\left|\logdet\left(I + \frac{1}{2}M(A_N,B_N)\right) - \logdet\left(I + \frac{1}{2}M(A,B)\right)\right|
	\\
	&\leq \frac{c_{\ep}^2}{4}
	\left(||A_N||_{\HS}||B_N-B||_{\HS} + ||B||_{\HS}||A_N-A||_{\HS}\right).
	\end{align*}
	Combining all these expressions give the desired result.
	\qed
\end{proof}

\begin{proof}
	[\textbf{Proof of Theorem \ref{theorem:Sinkhorn-entropic-approx}}]
	As in the proof of Theorem \ref{theorem:OT-entropic-approx}, 
	we have
	\begin{align*}
	&\left|\Srm^{\ep}_{d^2}[\Ncal(m_{A,N}, A_N), \Ncal(m_{B,N},B_N)] - \Srm^{\ep}_{d^2}[\Ncal(m_A, A), \Ncal(m_B,B)]\right|
	\\ 
	&\leq \left|||m_{A,N} - m_{B,N}||^2 - ||m_A-m_B||^2\right|
	\\
	& \quad +\left|\Srm^{\ep}_{d^2}[\Ncal(0, A_N), \Ncal(0,B_N)] - \Srm^{\ep}_{d^2}[\Ncal(0, A), \Ncal(0,B)]\right|.
	\end{align*}
	The mean terms are the same as in Theorem \ref{theorem:OT-entropic-approx}. For the terms involving the centered Gaussians,
	combining Lemmas \ref{lemma:trace-bound-square-root-1}, \ref{lemma:trace-bound-square-root-2}, \ref{lemma:logdet-bound-square-root-1}, \ref{lemma:logdet-bound-square-root-2}, 
	with $c = c_{\ep} = \frac{4}{\ep}$, $M(A,B) = -I + (I+c^2A^{1/2}BA^{1/2})^{1/2}$, we obtain
	\begin{align*}
	&\left|\Srm^{\ep}_{d^2}[\Ncal(0, A_N), \Ncal(0,B_N)] - \Srm^{\ep}_{d^2}[\Ncal(0, A), \Ncal(0,B)]\right|
	\\
	& \leq \frac{\ep}{4}|\trace(M(A_N,A_N)) - \trace(M(A,A))| + \frac{\ep}{4}|\trace(M(B_N,B_N)) - \trace(M(B,B))|
	\\
	& \quad +\frac{\ep}{4}\left|\logdet\left(I + \frac{1}{2}M(A_N,A_N)\right) - \logdet\left(I + \frac{1}{2}M(A,A)\right)\right| 
	\\
	& \quad + \frac{\ep}{4}\left|\logdet\left(I + \frac{1}{2}M(B_N,B_N)\right) - \logdet\left(I + \frac{1}{2}M(B,B\right)\right|
	\\
	&\quad +\frac{\ep}{2}|\trace(M(A_N,B_N)) - \trace(M(A,B))|
	\\
	& \quad + \frac{\ep}{2}\left|\logdet\left(I + \frac{1}{2}M(A_N,B_N)\right) - \logdet\left(I + \frac{1}{2}M(A,B\right)\right|
	\\
	&\leq \frac{\ep}{4}\frac{3c_{\ep}^2}{4}[||A_N||_{\HS} + ||A||_{\HS}]||A_N-A||_{\HS} + \frac{\ep}{4}\frac{3c_{\ep}^2}{4}
	[||B_N||_{\HS} + ||B||_{\HS}]||B_N-B||_{\HS}
	\\
	& \quad + \frac{\ep}{2}\frac{3c_{\ep}^2}{4}[||A_N||_{\HS}||B_N-B||_{\HS} + ||B||_{\HS}||A_N-A||_{\HS}]
	\\
	& \leq \frac{3}{\ep}[||A_N||_{\HS} + ||A||_{\HS}]||A_N-A||_{\HS} + \frac{3}{\ep}[||B_N||_{\HS} + ||B||_{\HS}]||B_N-B||_{\HS}
	\\
	&\quad +\frac{6}{\ep}[||A_N||_{\HS}||B_N-B||_{\HS} + ||B||_{\HS}||A_N-A||_{\HS}] 
	\\
	& = \frac{3}{\ep}[||A_N||_{\HS} + ||A||_{\HS} + 2||B||_{\HS}||]||A_N-A||_{\HS}
	\\
	&\quad + \frac{3}{\ep}[2||A_N||_{\HS} + ||B_N||_{\HS} + ||B||_{\HS}]||B_N-B||_{\HS}.
	\end{align*}
	\qed
\end{proof}

\subsection{Proofs for the RKHS setting}
\label{section:proofs-RKHS}

\begin{lemma}
\label{lemma:muPhi-norm-HK}
	Assume Assumptions 1-3. Then $\mu_{\Phi} = \int_{\X}\Phi(x)d\rho(x) \in \H_K$, with $||\mu_{\Phi}||_{\H_K} \leq \kappa$.
	\end{lemma}
\begin{proof}
	Define the random variable $\xi:(\X, \rho) \mapto \H_K$ by $\xi(x) = \Phi(x)$. Then $\mu_{\Phi} = \bE{\xi}$ and
	\begin{align*}
	\bE||\xi||^2_{\H_K} = \int_{\X}||\Phi(x)||^2_{\H_K}d\rho(x) = \int_{\X}K(x,x)d\rho(x) \leq \kappa^2.
	\end{align*}
	Thus $||\mu_{\Phi}||_{\H_K} = ||\bE{\xi}||_{\H_K} \leq \sqrt{\bE||\xi||^2_{\H_K}} \leq \kappa$.
	\qed
\end{proof}

\begin{lemma}
	\label{lemma:trace-LK}
	Assume Assumptions 1-3. Then the operator $L_K:\H_K \mapto \H_K$ as defined in Eq.\eqref{equation:LK} is positive, trace class, with
	\begin{align}
	\trace(L_K) &= \int_{\X}K(x,x)d\rho(x) \leq \kappa^2,
	\\
	||L_K||_{\HS(\H_K)} &\leq \kappa^2.
	\end{align}
	\end{lemma}
\begin{proof}
	Let $\{e_k\}_{k\in \Nbb}$ be any orthonormal basis in $\H_K$. 
	By Lebesgue Monotone Convergence Theorem,
	\begin{align*}
	&\trace(L_K) = \sum_{k=1}^{\infty}\la e_k, L_Ke_k\ra_{\H_K} = \sum_{k=1}^{\infty}\int_{\X}|\la \Phi(x), e_k\ra_{\H_K}|^2d\rho(x)
	\\
	& = \int_{\X}\sum_{k=1}^{\infty}|\la \Phi(x), e_k\ra_{\H_K}|^2d\rho(x) = \int_{\X}||\Phi(x)||^2_{\H_K}d\rho(x) 
	= \int_{\X}K(x,x)d\rho(x) \leq \kappa^2.
	\end{align*}
	Since $L_K$ is positive, this implies that $L_K$ is trace class, hence compact.
	Let $\{\lambda_k\}_{k \in \Nbb}$ be the eigenvalues of $L_K$, then $\lambda_k \geq 0$ $\forall k \in \Nbb$ and
	\begin{align*}
	||L_K||^2_{\HS(\H_K)} = \sum_{k=1}^{\infty}\lambda_k^2 \leq (\sum_{k=1}^{\infty}\lambda_k)^2 = [\trace(L_K)]^2 \leq \kappa^4.
	\end{align*}
	\qed
\end{proof}

\begin{lemma}
	\label{lemma:HSnorm-rankone-ab}
	For any pair $a,b\in \H$, $||a\otimes b||_{\HS(\H)} = ||a||\;||b||$.
\end{lemma}
\begin{proof}
	Let $\{e_k\}_{k \in \Nbb}$ be any orthonormal basis in $\H$, then
	\begin{align*}
	||a\otimes b||_{\HS}^2 = \sum_{k=1}^{\infty}||(a\otimes b)e_k||^2 = \sum_{k=1}^{\infty}||\la b, e_k\ra a|||^2 
	= ||a||^2\sum_{k=1}^{\infty}|\la b,e_k\ra|^2 = ||a||^2||b||^2.
	\end{align*}
	\qed
\end{proof}

\begin{proof}
[\textbf{Proof of Theorem \ref{theorem:Sinkhorn-divergence-Borel-probability-measures}}]
We need only to show that 
\begin{align*}
\Srm^{\ep}_{d^2}[\Ncal(\mu_{\Phi,\rho_1}, C_{\Phi, \rho_1}), \Ncal(\mu_{\Phi,\rho_2}, C_{\Phi, \rho_2})] = 0 \equivalent \rho_1 = \rho_2.
\end{align*}
When $\ep = \infty$, $\Srm^{\infty}_{d^2} = \MMD_K^2$ and the desired property is already valid. Assume now $0 \leq \ep < \infty$.
We have the decomposition
\begin{align*}
&\Srm^{\ep}_{d^2}[\Ncal(\mu_{\Phi,\rho_1}, C_{\Phi, \rho_1}), \Ncal(\mu_{\Phi,\rho_2}, C_{\Phi, \rho_2})] 
\\
&= ||\mu_{\Phi,\rho_1}-\mu_{\Phi,\rho_2}||^2_{\H_K} 
+ \Srm^{\ep}_{d^2}[\Ncal(0, C_{\Phi, \rho_1}), \Ncal(0, C_{\Phi, \rho_2})]
\\
& = \MMD_K^2(\rho_1, \rho_2) +  \Srm^{\ep}_{d^2}[\Ncal(0, C_{\Phi, \rho_1}), \Ncal(0, C_{\Phi, \rho_2})].
\end{align*}
Thus $\Srm^{\ep}_{d^2}[\Ncal(\mu_{\Phi,\rho_1}, C_{\Phi, \rho_1}), \Ncal(\mu_{\Phi,\rho_2}, C_{\Phi, \rho_2})] = 0
$ if and only if
\begin{align*}
\MMD_K(\rho_1,\rho_2) = 0
\;\text{ and }\Srm^{\ep}_{d^2}[\Ncal(0, C_{\Phi, \rho_1}), \Ncal(0, C_{\Phi, \rho_2})] = 0.
\end{align*}
Since $K$
is a characteristic kernel on
$\X$, $\MMD_K(\rho_1,\rho_2) = 0 \equivalent \rho_1 = \rho_2$, the latter also 
implying $\Srm^{\ep}_{d^2}[\Ncal(0, C_{\Phi, \rho_1}), \Ncal(0, C_{\Phi, \rho_2})] = 0$.
\qed
\end{proof}

\begin{proof}
	[\textbf{Proof of Theorem \ref{theorem:RKHS-distance}}]
	For the mean terms, we have
	\begin{align*}
	&||\mu_{\Phi(\Xbf)} - \mu_{\Phi(\Ybf)}||^2_{\H_K} = 
	\left\|\frac{1}{m}\sum_{i=1}^m\Phi(x_i) - \frac{1}{n}\sum_{j=1}^n\Phi(y_j) \right\|^2_{\H_K}
	\\
	& = \frac{1}{m^2}\sum_{i,j=1}^m\la \Phi(x_i), \Phi(x_j)\ra_{\H_K} + \frac{1}{n^2}\sum_{i,j=1}^n \la \Phi(y_i), \Phi(y_j)\ra_{\H_K} - \frac{2}{mn}\sum_{i=1}^m\sum_{j=1}^n\la\Phi(x_i), \Phi(y_j)\ra_{\H_K} 
	\\
	& = \frac{1}{m^2}\1_m^TK[\Xbf]\1_m + \frac{1}{n^2}\1_n^TK[\Ybf]\1_n - \frac{2}{mn}\1_m^TK[\Xbf,\Ybf]\1_n.
	\end{align*}
	Let $\lambda(A)$ denote the set of nonzero eigenvalues of a compact operator $A$. Since the nonzero eigenvalues of $AB$ are the same as those of $BA$, we obtain
	\begin{align*}
	\lambda(C_{\Phi(\Xbf)}) & = \frac{1}{m}\lambda([\Phi(\Xbf) J_m \Phi(\Xbf)^{*}]) = \frac{1}{m}\lambda([\Phi(\Xbf)^{*}\Phi(\Xbf) J_m ])
	\\
	& = \frac{1}{m}\lambda(K[\Xbf]J_m) = \frac{1}{m}\lambda(J_mK[\Xbf]J_m), \;\;\text{since $J_m^2 = J_m$}
	\\
	\lambda(C_{\Phi(\Ybf)}) &= \frac{1}{n}\lambda(K[\Ybf]J_n) = \frac{1}{n}\lambda(J_nK[\Ybf]J_n),
	\\
	\lambda(C_{\Phi(\Xbf)}^2) &= \frac{1}{m^2}\lambda([\Phi(\Xbf) J_m \Phi(\Xbf)^{*}][\Phi(\Xbf) J_m \Phi(\Xbf)^{*}])
	\\
	&= \frac{1}{m^2}\lambda([\Phi(\Xbf)^{*}\Phi(\Xbf) J_m \Phi(\Xbf)^{*}\Phi(\Xbf) J_m ]) = \frac{1}{m^2}\lambda[(K[\Xbf]J_m)^2]
	\\
	& = \frac{1}{m^2}\lambda[(J_mK[\Xbf]J_m)^2],
	\\
	\lambda(C_{\Phi(\Ybf)}^2) &=\frac{1}{n^2}\lambda[(K[\Ybf]J_n)^2] = \frac{1}{n^2}\lambda[(J_nK[\Ybf]J_n)^2],
	\\
	\lambda(C_{\Phi(\Ybf)}^{1/2}C_{\Phi(\Xbf)}C_{\Phi(\Ybf)}^{1/2}) &=\lambda(C_{\Phi(\Xbf)}^{1/2}C_{\Phi(\Ybf)}C_{\Phi(\Xbf)}^{1/2}) = \lambda(C_{\Phi(\Ybf)}C_{\Phi(\Xbf)}) 
	\\
	&= \frac{1}{mn}\lambda([\Phi(\Ybf) J_n \Phi(\Ybf)^{*}][\Phi(\Xbf) J_m \Phi(\Xbf)^{*}])
	\\
	&= \frac{1}{mn}\lambda(K[\Xbf,\Ybf]J_nK[\Ybf,\Xbf]J_m)
	\\
	& = \frac{1}{mn}\lambda(J_mK[\Xbf,\Ybf]J_nK[\Ybf,\Xbf]J_m)
	\end{align*}
	Combining these with the expressions for $\OT^{\ep}_{d^2}$ and $\Srm^{\ep}_{d^2}$ in Theorem
	\ref{theorem:OT-regularized-Gaussian}
	gives the desired results. 
	\qed
\end{proof}

\subsection{Proof for the bounded kernel setting}

\begin{lemma}
	\label{lemma:mean-bound-HK}
	Assume Assumptions 1-4, then
	\begin{align}
	||\mu_{\Phi}||_{\H_K} \leq \kappa, \;\;\; ||\mu_{\Phi(\Xbf)}||_{\H_K} &\leq \kappa \;\;\;\forall \Xbf \in \X^m.
	\end{align}
	For any $0 < \delta < 1$, with probability at least $1-\delta$,
	\begin{align}
	||\mu_{\Phi(\Xbf)} - \mu_{\Phi}||_{\H_K} \leq \kappa \left( \frac{2\log\frac{2}{\delta}}{m} + \sqrt{\frac{2\log\frac{2}{\delta}}{m}}\right). 
	\end{align}
\end{lemma}
\begin{proof}
	By Lemma \ref{lemma:muPhi-norm-HK}, $||\mu_{\Phi}||_{\H_K} \leq \kappa$. 
	Define the random variable $\xi: (\X, \rho) \mapto \H_K$ by $\xi(x) = \Phi(x)$. Then
	$\mu_{\Phi(\Xbf)} = \frac{1}{m}\sum_{j=1}^m\xi(x_j)$ and $\mu_{\Phi} = \bE(\xi)$, with 
	\begin{align*}
	&||\xi(x)||_{\H_K} = ||\Phi(x)||_{\H_K} = \sqrt{\la \Phi(x), \Phi(x)\ra_{\H_K}} = \sqrt{K(x,x)} \leq \kappa\;\forall x \in \X, 
	\\
	&\sigma^2(\xi) = \bE||\xi||_{\H_K}^2 = \int_{\X}\la \Phi(x), \Phi(x)\ra_{\H_K} = \int_{\X}K(x,x)d\rho(x) \leq \kappa^2 < \infty,
	\\
	&||\mu_{\Phi(\Xbf)}||_{\H_K} \leq \frac{1}{m}\sum_{j=1}^m ||\Phi(x_j)||_{\H_K} = \frac{1}{m}\sum_{j=1}^m\sqrt{K(x_j,x_j)} \leq \kappa \;\forall \Xbf \in \X^m.
	\end{align*}
	The desired result then follows by invoking Proposition \ref{proposition:Pinelis}.
	\qed
\end{proof}

\begin{lemma}
	\label{lemma:mean-bound-HS(HK)}
	Assume Assumptions 1-4, then
	\begin{align}
	||\mu_{\Phi(\Xbf)} \otimes \mu_{\Phi(\Xbf)}||_{\HS(\H_K)} &\leq \kappa^2 \;\;\;\forall \Xbf \in \X^m,
	\\
	||\mu_{\Phi}\otimes \mu_{\Phi}||_{\HS(\H_K)} &\leq \kappa^2.
	\end{align}
	For any $0< \delta < 1$, with probability at least $1-\delta$,
	\begin{align}
	||\mu_{\Phi(\Xbf)} \otimes \mu_{\Phi(\Xbf)} - \mu_{\Phi}\otimes \mu_{\Phi}||_{\HS(H_K)} \leq 
	2\kappa^2\left(\frac{2\log\frac{2}{\delta}}{m} + \sqrt{\frac{2\log\frac{2}{\delta}}{m}}\right).
	\end{align}
\end{lemma}
\begin{proof}
	By Lemma \ref{lemma:HSnorm-rankone-ab}, $\forall \Xbf \in \X^m$,
	\begin{align*}
	&||\mu_{\Phi(\Xbf)}\otimes \mu_{\Phi(\Xbf)}||_{\HS(\H_K)} = ||\mu_{\Phi(\Xbf)}||^2_{\H_K}
	= \frac{1}{m^2}||\sum_{j=1}^m\Phi(x_j)||^2_{\H_K} 
	\\
	&= \frac{1}{m^2}\sum_{j,k=1}^m\la \Phi(x_j), \Phi(x_k)\ra_{\H_K} = \frac{1}{m^2}\sum_{j,k=1}^mK(x_j,x_k) \leq \kappa^2.
	\\
	&||\mu_{\Phi(\Xbf)}\otimes \mu_{\Phi(\Xbf)} - \mu_{\Phi}\otimes \mu_{\Phi}||_{\HS} 
	\\
	&\leq ||\mu_{\Phi(\Xbf)}\otimes (\mu_{\Phi(\Xbf)}-\mu_{\Phi})||_{\HS} +||(\mu_{\Phi(\Xbf)}-\mu_{\Phi})\otimes \mu_{\Phi}||_{\HS}
	\\
	& =\left( ||\mu_{\Phi(\Xbf)}||_{\H_K} + ||\mu_{\Phi}||_{\H_K}\right)||\mu_{\Phi(\Xbf)}-\mu_{\Phi}||_{\H_K}.
	\end{align*}
	As in the proof of Lemma \ref{lemma:mean-bound-HK},
	define the random variable $\xi: (\X, \rho) \mapto \H_K$ by $\xi(x) = \Phi(x)$. Then
	$\mu_{\Phi(\Xbf)} = \frac{1}{m}\sum_{j=1}^m\xi(x_j)$ and $\mu_{\Phi} = \bE(\xi)$, with
	\begin{align*}
	&||\mu_{\Phi(\Xbf)}||_{\H_K} \leq \frac{1}{m}\sum_{j=1}^m ||\Phi(x_j)||_{\H_K} = \frac{1}{m}\sum_{j=1}^m\sqrt{K(x_j,x_j)} \leq \kappa \;\forall \Xbf \in \X^m,
	\\
	&||\mu_{\Phi}\otimes \mu_{\Phi}||_{\HS(\H_K)}  = ||\mu_{\Phi}||^2_{\H_K} \leq \kappa^2, \;\;\text{by Lemma \ref{lemma:muPhi-norm-HK}}.
	\end{align*} 
	The desired result then follows by invoking Lemma \ref{lemma:mean-bound-HK}.
	\qed
\end{proof}

\begin{proposition}
	\label{proposition:LKbound-HSnorm}
	Assume Assumptions 1-4, then
	\begin{align}
	\left\|\frac{1}{m}\Phi(\Xbf)\Phi(\Xbf)^{*}\right\|_{\HS(\H_K)} &\leq \kappa^2 \;\;\forall \Xbf \in \X^m,
	\\
	||L_K||_{\HS(\H_K)} &\leq \kappa^2.
	\end{align}
	For any $0 < \delta < 1$, with probability at least $1-\delta$,
	\begin{align}
	\left\| \frac{1}{m}\Phi(\Xbf)\Phi(\Xbf)^{*} - L_K\right\|_{\HS(\H_K)} \leq \kappa^2\left(\frac{2\log\frac{2}{\delta}}{m} + \sqrt{\frac{2\log\frac{2}{\delta}}{m}}\right). 
	\end{align}
\end{proposition}
\begin{proof}
	We have $||L_K||_{\HS(\H_K)} \leq \kappa^2$ by Lemma \ref{lemma:trace-LK}.
	Define the random variable $\xi: (\X, \rho) \mapto \HS(\H_K)$ by $\xi(x) = \Phi(x) \otimes \Phi(x)$,
	then,
	\begin{align*}
	\frac{1}{m}\Phi(\Xbf)\Phi(\Xbf)^{*} &= \frac{1}{m}\sum_{j=1}^m(\Phi(x_j) \otimes \Phi(x_j)) = \frac{1}{m}\sum_{j=1}^m\xi(x_j),
	\\
	L_K &= \int_{\X}\Phi(x)\otimes \Phi(x) d\rho(x) = \bE(\xi).
	\end{align*}
	By Lemma \ref{lemma:HSnorm-rankone-ab}, 
	\begin{align*}
	||\xi(x)||_{\HS(\H_K)} &= ||\Phi(x)||^2_{\H_K} = K(x,x) \leq \kappa^2 \;\forall x \in \X,
	\\
	\sigma^2(\xi) &= \bE||\xi||^2_{\HS(\H_K)} = \int_{\X}K(x,x)^2d\rho(x) \leq \kappa^4.  
	\end{align*}
	The desired result then follows by invoking Proposition \ref{proposition:Pinelis}.
	\qed
\end{proof}

\begin{proof}
	[\textbf{Proof of Theorem \ref{theorem:CPhi-concentration}}]
	Combining Lemma \ref{lemma:mean-bound-HS(HK)} and Proposition \ref{proposition:LKbound-HSnorm},
	\begin{align*}
	&||C_{\Phi(\Xbf)}||_{\HS} = \left\|\frac{1}{m}\Phi(\Xbf)\Phi(\Xbf)^{*} - \mu_{\Phi(\Xbf)}\otimes \mu_{\Phi(\Xbf)}\right\|_{\HS} 
	\\
	&\leq  \left\|\frac{1}{m}\Phi(\Xbf)\Phi(\Xbf)^{*}\right\|_{\HS} +||\mu_{\Phi(\Xbf)}\otimes \mu_{\Phi(\Xbf)}||_{\HS}
	\leq \kappa^2 + \kappa^2 = 2 \kappa^2.
	\\
	&||C_{\Phi}||_{\HS} = ||L_K - \mu_{\Phi}\otimes \mu_{\Phi}||_{\HS} \leq ||L_K||_{\HS} + ||\mu_{\Phi}\otimes \mu_{\Phi}||_{\HS}
	\leq \kappa^2 + \kappa^2 = 2\kappa^2.
	\end{align*}
	For any $0 < \delta < 1$, by Lemma \ref{lemma:mean-bound-HK}, let $U_1 \subset \X^m$ be defined by
	\begin{align*}
	U_1 = \left\{\Xbf \in \X^m: ||\mu_{\Phi(\Xbf)} - \mu_{\Phi}||_{\H_K} \leq \kappa\left(\frac{2\log\frac{4}{\delta}}{m} + \sqrt{\frac{2\log\frac{4}{\delta}}{m}}\right) \right\},
	\end{align*}
	then $\rho^m(U_1) \geq 1-\frac{\delta}{2}$. By the proof of Lemma \ref{lemma:mean-bound-HS(HK)}, $\forall \Xbf \in U_1$,
	\begin{align*}
	||\mu_{\Phi(\Xbf)}\otimes \mu_{\Phi(\Xbf)} - \mu_{\Phi}\otimes \mu_{\Phi}||_{\HS} \leq 
	2\kappa^2\left(\frac{2\log\frac{4}{\delta}}{m} + \sqrt{\frac{2\log\frac{4}{\delta}}{m}}\right).
	\end{align*}
	By Proposition \ref{proposition:LKbound-HSnorm}, let $U_2 \subset \X^m$ be defined by
	\begin{align*}
	U_2 = \left\{\Xbf \in \X^m: \left\|\frac{1}{m}\Phi(\Xbf)\Phi(\Xbf)^{*}-L_K\right\|_{\HS} \leq \kappa^2\left(\frac{2\log\frac{4}{\delta}}{m} + \sqrt{\frac{2\log\frac{4}{\delta}}{m}}\right)\right\},
	\end{align*}
	then $\rho^m(U_2) \geq 1 - \frac{\delta}{2}$. The intersection set $U_1 \cap U_2$ satisfies
	\begin{align*}
	\rho^m(U_1 \cap U_2) = \rho^m(U_1) + \rho^m(U_2) - \rho^m(U_1 \cup U_2) \geq 2(1-\frac{\delta}{2}) -1 = 1-\delta.
	\end{align*}
	Thus for $\Xbf \in U_1 \cap U_2$, with measure at least $1-\delta$,
	\begin{align*}
	||C_{\Phi(\Xbf)} - C_{\Phi}||_{\HS} &\leq \left\|\frac{1}{m}\Phi(\Xbf)\Phi(\Xbf)^{*}-L_K\right\|_{\HS} + ||\mu_{\Phi(\Xbf)}\otimes \mu_{\Phi(\Xbf)} - \mu_{\Phi}\otimes \mu_{\Phi}||_{\HS}
	\\
	&\leq 3\kappa^2\left(\frac{2\log\frac{4}{\delta}}{m} + \sqrt{\frac{2\log\frac{4}{\delta}}{m}}\right).
	\end{align*}
	\qed
\end{proof}

\begin{proof}
	[\textbf{Proof of Theorem \ref{theorem:Sinkhorn-RKHS-concentration}}]
	By Theorems \ref{theorem:convergence-Sinkhorn} and \ref{theorem:CPhi-concentration}, for any $0 < \delta < 1$,
	\begin{align*}
	&\Srm^{\ep}_{d^2}[\Ncal(\mu_{\Phi(\Xbf)}, C_{\Phi(\Xbf)}), \Ncal(\mu_{\Phi}, C_{\Phi})]
	\\
	& \leq ||\mu_{\Phi(\Xbf)} - \mu_{\Phi}||^2_{\H_K}
	+  \frac{3}{\ep}[||C_{\Phi(\Xbf)}||_{\HS(\H_K)} + ||C_{\Phi}||_{\HS(\H_K)}]||C_{\Phi(\Xbf)}- C_{\Phi}||_{\HS(\H_K)}
	\\
	& \leq \kappa^2\left(\frac{2\log\frac{4}{\delta}}{m} + \sqrt{\frac{2\log\frac{4}{\delta}}{m}}\right)^2
	+ \frac{3}{\ep}[2\kappa^2 + 2\kappa^2]3\kappa^2\left(\frac{2\log\frac{4}{\delta}}{m} + \sqrt{\frac{2\log\frac{4}{\delta}}{m}}\right)
	\\
	& = \kappa^2\left(\frac{2\log\frac{4}{\delta}}{m} + \sqrt{\frac{2\log\frac{4}{\delta}}{m}}\right)^2
	+ \frac{36\kappa^4}{\ep}\left(\frac{2\log\frac{4}{\delta}}{m} + \sqrt{\frac{2\log\frac{4}{\delta}}{m}}\right),
	\end{align*} 
	with probability at least $1-\delta$. 
	\qed
\end{proof}

\begin{proof}
	[\textbf{Proof of Theorem \ref{theorem:Sinkhorn-RKHS-approximation}}]
	By Theorem \ref{theorem:Sinkhorn-entropic-approx}, Lemma \ref{lemma:mean-bound-HK}, and Theorem \ref{theorem:CPhi-concentration},
	\begin{align*}
	&\left|\Srm^{\ep}_{d^2}[\Ncal(\mu_{\Phi(\Xbf)}, C_{\Phi(\Xbf)}), \Ncal(\mu_{\Phi(\Ybf)}, C_{\Phi(\Ybf)})]
	- \Srm^{\ep}_{d^2}[\Ncal(\mu_{\Phi,\rho_1}, C_{\Phi, \rho_1}), \Ncal(\mu_{\Phi,\rho_2}, C_{\Phi, \rho_2})]\right|
	\\
	& \leq \left[||\mu_{\Phi(\Xbf)}||_{\H_K} + ||\mu_{\Phi(\Ybf)}||_{\H_K} + ||\mu_{\Phi, \rho_1}||_{\H_K} + ||\mu_{\Phi, \rho_2}||_{\H_K} \right]
	\\
	& \quad \times \left[||\mu_{\Phi(\Xbf)} - \mu_{\Phi, \rho_1}||_{\H_K} + ||\mu_{\Phi(\Ybf)} - \mu_{\rho,2}||_{\H_K}\right]
	\\
	& \quad + \frac{3}{\ep}\left[||C_{\Phi(\Xbf)}||_{\HS(\H_K)} + ||C_{\Phi,\rho_1}||_{\HS(\H_K)} + 2||C_{\Phi, \rho_2}||_{\HS(\H_K)}\right]
	||C_{\Phi(\Xbf)} - C_{\Phi, \rho,1}||_{\HS(\H_K)}
	\\
	&\quad + \frac{3}{\ep}\left[2||C_{\Phi(\Xbf)}||_{\HS(\H_K)} + ||C_{\Phi(\Ybf)}||_{\HS(\H_K)} + ||C_{\Phi, \rho_2}||_{\HS(\H_K)}\right]
	||C_{\Phi(\Ybf)} - C_{\Phi, \rho_2}||_{\HS(\H_K)}
	\\
	& \leq 4\kappa\left[||\mu_{\Phi(\Xbf)} - \mu_{\Phi, \rho_1}||_{\H_K} + ||\mu_{\Phi(\Ybf)} - \mu_{\rho,2}||_{\H_K}\right]
	\\
	& \quad + \frac{24\kappa^2}{\ep}||C_{\Phi(\Xbf)} - C_{\Phi, \rho,1}||_{\HS(\H_K)} + \frac{24\kappa^2}{\ep}||C_{\Phi(\Ybf)} - C_{\Phi, \rho_2}||_{\HS(\H_K)}.
	\end{align*}
	By Theorems \ref{theorem:CPhi-concentration}, for any $0 < \delta < 1$, 
	let $U_1 \subset (\X, \rho_1)^m$ be such that $\forall \Xbf \in U_1$,
	\begin{align*}
	||\mu_{\Phi(\Xbf)} - \mu_{\Phi, \rho_1}||_{\H_K} & \leq \kappa\left(\frac{2\log\frac{8}{\delta}}{m} + \sqrt{\frac{2\log\frac{8}{\delta}}{m}}\right),
	\\
	\text{and } ||C_{\Phi(\Xbf)} - C_{\Phi, \rho_1}||_{\HS(\H_K)} &\leq 3\kappa^2\left(\frac{2\log\frac{8}{\delta}}{m} + \sqrt{\frac{2\log\frac{8}{\delta}}{m}}\right),
	\end{align*}
	then $\rho^m(U_1) \geq 1 -\frac{\delta}{2}$.
	Similarly, let $U_2 \subset (\X, \rho_2)^n$ be such that $\forall \Ybf \in U_2$, 
	\begin{align*}
	||\mu_{\Phi(\Ybf)} - \mu_{\Phi, \rho_2}||_{\H_K} & \leq \kappa\left(\frac{2\log\frac{8}{\delta}}{n} + \sqrt{\frac{2\log\frac{8}{\delta}}{n}}\right),
	\\
	\text{and } ||C_{\Phi(\Ybf)} - C_{\Phi, \rho_2}||_{\HS(\H_K)} &\leq 3\kappa^2\left(\frac{2\log\frac{8}{\delta}}{n} + \sqrt{\frac{2\log\frac{8}{\delta}}{n}}\right),
	\end{align*}
	then $\rho^n(U_2) \geq 1-\frac{\delta}{2}$. Let now $Z = (U_1 \times (\X,\rho_2)^n) \cap ((\X, \rho_1)^m \times U_2)
	\subset (\X,\rho_1)^m \times (\X,\rho_2)^n$, then $(\rho_1^m \times \rho_2^n)(Z) \geq 1 -\delta$.
	Then $\forall (\Xbf,\Ybf) \in Z$, we have
	\begin{align*}
	&\left|\Srm^{\ep}_{d^2}[\Ncal(\mu_{\Phi(\Xbf)}, C_{\Phi(\Xbf)}), \Ncal(\mu_{\Phi(\Ybf)}, C_{\Phi(\Ybf)})]
	- \Srm^{\ep}_{d^2}[\Ncal(\mu_{\Phi,\rho_1}, C_{\Phi, \rho_1}), \Ncal(\mu_{\Phi,\rho_2}, C_{\Phi, \rho_2})]\right|
	\\
	& 
	\leq 4\kappa^2\left(\frac{2\log\frac{8}{\delta}}{m} + \sqrt{\frac{2\log\frac{8}{\delta}}{m}}+\frac{2\log\frac{8}{\delta}}{n} + \sqrt{\frac{2\log\frac{8}{\delta}}{n}}\right)
	\\
	& \quad +\frac{72\kappa^4}{\ep}\left(\frac{2\log\frac{8}{\delta}}{m} + \sqrt{\frac{2\log\frac{8}{\delta}}{m}}\right)
	+
	\frac{72\kappa^4}{\ep}\left(\frac{2\log\frac{8}{\delta}}{n} + \sqrt{\frac{2\log\frac{8}{\delta}}{n}}\right).
	\end{align*}
	\qed
\end{proof}

\subsection{Proofs for the general kernel case}

\begin{lemma}
	\label{lemma:mean-norm-HK-unbounded}
	Assume Assumptions 1-3. Let $\Xbf = (x_i)_{i=1}^m$ be independently sampled from $(\X, \rho)$. 
	For any $0 < \delta < 1$, with probability at least $1-\delta$,
	\begin{align}
	||\mu_{\Phi(\Xbf)} - \mu_{\Phi}||_{\H_K} &\leq \frac{\kappa}{\sqrt{m}\delta},
	\\
	||\mu_{\Phi(\Xbf)}||_{\H_K} &\leq \kappa \left(1 + \frac{1}{\sqrt{m}\delta}\right),
	\\
	||\mu_{\Phi(\Xbf)} \otimes \mu_{\Phi(\Xbf)} - \mu_{\Phi}\otimes \mu_{\Phi}||_{\HS(\H_K)} &\leq \frac{\kappa^2}{\sqrt{m}\delta}\left(2 + \frac{1}{\sqrt{m}\delta}\right).
	\end{align}
\end{lemma}
\begin{proof}
	Let $Y_j$, $1 \leq j \leq m$, be IID $\H_K$-valued random variables defined on $(\X^m, \rho^m)$ by
	$Y_j(\Xbf) = \Phi(x_j)$, where $\Xbf = (x_1, \ldots, x_m)$ is sampled independently from $(\X, \rho)$. 
	Then $\bE{Y_j} = \int_{\X}\Phi(x)d\rho(x) = \mu_{\Phi}$ and
	\begin{align*} 
	\bE{||Y_j||^2_{\H_K}} = \int_{\X}||\Phi(x)||^2_{\H_K}d\rho(x) = \int_{\X}K(x,x)d\rho(x)  \leq \kappa^2.
	\end{align*}
	Let $\xi:(\X^m, \rho^m) \mapto \R$ be the random variable defined by
	\begin{align*}
	\xi(\Xbf) = \left\|\frac{1}{m}\sum_{j=1}^mY_j(\Xbf) - \mu_{\Phi}\right\|_{\H_K} = \left\|\frac{1}{m}\sum_{j=1}^m\Phi(x_j) - \mu_{\Phi}\right\|_{\H_K} = ||\mu_{\Phi(\Xbf)} - \mu_{\Phi}||_{\H_K}.
	\end{align*}
	Since the $Y_j$'s are independent, identically distributed, with $\bE{Y_j} = \mu_{\Phi}$,
	\begin{align*}
	\bE{\xi^2} &= \frac{1}{m^2}\bE\left\|\sum_{j=1}^m(Y_j -\mu_{\Phi})\right\|^2_{\H_K} = \frac{1}{m^2}\sum_{j=1}^m\bE||Y_j -\mu_{\Phi}||^2_{\H_K}
	\\
	&=
	\frac{1}{m^2}\sum_{j=1}^m\left(\bE||Y_j||^2_{\H_K} - 2\la \bE[Y_j], \mu_{\Phi}\ra_{\H_K} + ||\mu_{\Phi}||^2_{\H_K}\right)
	\\
	& = \frac{1}{m^2}\sum_{j=1}^m\bE\left(||Y_j||^2_{\H_K} - ||\mu_{\Phi}||^2_{\H_K}\right)
	\leq \frac{1}{m}\kappa^2.
	\end{align*}
	By the Chebyshev inequality, for any $t > 0$,
	\begin{align*}
	\bP(\xi \geq t) \leq \frac{\bE\xi}{t} \leq \frac{\sqrt{\bE{\xi^2}}}{t}\leq \frac{\kappa}{\sqrt{m}t}.
	\end{align*}
	Let $\delta = \frac{\kappa}{\sqrt{m}t } \equivalent t = \frac{\kappa}{\sqrt{m}\delta}$, then with probability at least $1-\delta$,
	\begin{align*}
	||\mu_{\Phi(\Xbf)} - \mu_{\Phi}||_{\H_K} = \xi(\Xbf) \leq t = \frac{\kappa}{\sqrt{m}\delta}.
	\end{align*}
	This gives the first inequality. Then, with probability at least $1-\delta$,
	\begin{align*}
	||\mu_{\Phi(\Xbf)}||_{\H_K} \leq ||\mu_{\Phi(\Xbf)} - \mu_{\Phi}||_{\H_K} + ||\mu_{\Phi}||_{\H_K} \leq \kappa (1 + \frac{1}{\sqrt{m}\delta}). 
	\end{align*}
	For the third inequality, as in the proof of Lemma \ref{lemma:mean-bound-HS(HK)},
	\begin{align*}
	||\mu_{\Phi(\Xbf)}\otimes \mu_{\Phi(\Xbf)} - \mu_{\Phi}\otimes \mu_{\Phi}||_{\HS(\H_K)} &\leq [||\mu_{\Phi(\Xbf)}||_{\H_K} + ||\mu_{\Phi}||_{\H_K}]||\mu_{\Phi(\Xbf)} - \mu_{\Phi}||_{\H_K}
	\\
	&\leq \frac{\kappa^2}{\sqrt{m}\delta}(2 + \frac{1}{\sqrt{m}\delta}).
	\end{align*}
	\qed
\end{proof}

\begin{lemma}
	\label{lemma:LK-HSnorm-unbounded}
	Assume Assumptions 1,2 and 5. 
	Let $\Xbf =(x_i)_{i=1}^m$ be sampled independently from $(\X, \rho)$.
	For any $0 <\delta < 1$, with probability at least $1-\delta$,
	\begin{align}
	\left\|\frac{1}{m}\Phi(\Xbf)\Phi(\Xbf)^{*} - L_K\right\|_{\HS(\H_K)} \leq \frac{\kappa^2}{\sqrt{m}\delta}.
	\end{align}
\end{lemma}
\begin{proof}
	Let $Y_j$, $1 \leq j \leq m$, be IID $\HS(\H_K)$-valued random variables defined on $(\X^m, \rho^m)$ by
	$Y_j(\Xbf) = \Phi(x_j)\otimes \Phi(x_j)$, where $\Xbf = (x_1, \ldots, x_m)$ is any IID random sample from $(\X, \rho)$. 
	Then $\bE{Y_j} = \int_{\X}\Phi(x)\otimes \Phi(x) d\rho(x) = L_K$ and
	\begin{align*} 
	\bE{||Y_j||^2_{\HS(\H_K)}} &= \int_{\X}||\Phi(x) \otimes \Phi(x)||^2_{\HS(\H_K)} = \int_{\X}||\Phi(x)||^4_{\H_K}d\rho(x)
	\\
	& = \int_{\X}K(x,x)^2d\rho(x)  \leq \kappa^4.
	\end{align*}
	Let $\xi:(\X^m, \rho^m) \mapto \R$ be the random variable defined by
	\begin{align*}
	\xi(\Xbf) = \left\|\frac{1}{m}\sum_{j=1}^mY_j - L_K\right\|_{\HS(\H_K)} &= \left\|\frac{1}{m}\sum_{j=1}^m\Phi(x_j) \otimes \Phi(x_j) - L_K\right\|_{\HS(\H_K)} 
	\\
	&= 
	\left\|\frac{1}{m}\Phi(\Xbf)\Phi(\Xbf)^{*} - L_K\right\|_{\HS(\H_K)}.
	\end{align*}
	Since the $Y_j$'s are independent, identically distributed, with $\bE{Y_j} = L_K$,
	\begin{align*}
	\bE{\xi^2} &= 
	\frac{1}{m^2}\bE\left\|\sum_{j=1}^m(Y_j - L_K)\right\|_{\HS(\H_K)}^2 = \frac{1}{m^2}\sum_{j=1}^m\bE||Y_j-L_K||_{\HS(\H_K)}^2
	\\
	& = \frac{1}{m^2}\sum_{j=1}^m(\bE||Y_j||^2_{\HS(\H_K)} - ||L_K||^2_{\HS(\H_K)})
	\leq \frac{1}{m}\kappa^4.
	\end{align*}
	By the Chebyshev inequality, for any $t > 0$,
	\begin{align*}
	\bP(\xi \geq t) \leq \frac{\bE\xi}{t} \leq \frac{\sqrt{\bE{\xi^2}}}{t}\leq \frac{\kappa^2}{\sqrt{m}t}.
	\end{align*}
	Let $\delta = \frac{\kappa^2}{\sqrt{m}t } \equivalent t = \frac{\kappa^2}{\sqrt{m}\delta}$, then with probability at least $1-\delta$,
	\begin{align*}
	\left\|\frac{1}{m}\Phi(\Xbf)\Phi(\Xbf)^{*} - L_K\right\|_{\HS(\H_K)} = \xi(\Xbf) \leq t = \frac{\kappa^2}{\sqrt{m}\delta}.
	\end{align*}
	\qed
\end{proof}

\begin{proof}
	[\textbf{Proof of  Proposition \ref{proposition:CPhi-concentration-unbounded}}]
	Let $U_1 \subset \X^m$ be defined by
	\begin{align*}
	U_1 &= \left\{\Xbf \in \X^m:||\mu_{\Phi(\Xbf)} - \mu_{\Phi}||_{\H_K} \leq \frac{2\kappa}{\sqrt{m}\delta},
	||\mu_{\Phi(\Xbf)}||  \leq \kappa\left(1+ \frac{2}{\sqrt{m}\delta}\right)\right\}.
	\end{align*}
	Then by Lemma \ref{lemma:mean-norm-HK-unbounded}, $\rho^m(U_1)\geq 1-\frac{\delta}{2}$ and for all $\Xbf \in \X^m$, 
	\begin{align*}
	||\mu_{\Phi(\Xbf)}\otimes \mu_{\Phi(\Xbf)} - \mu_{\Phi}\otimes \mu_{\Phi}||_{\HS} 
	\leq \frac{4\kappa^2}{\sqrt{m}\delta}\left(1+\frac{1}{\sqrt{m}\delta}\right).
	\end{align*}
	Similarly, let $U_2 \subset \X^m$ be defined by
	\begin{align*}
	U_2 &= \left\{\Xbf \in \X^m: \left\|\frac{1}{m}\Phi(\Xbf)\Phi(\Xbf)^{*} - L_K\right\|_{\HS} 
	\leq \frac{2\kappa^2}{\sqrt{m}\delta}\right\}.
	\end{align*}
	By Lemma \ref{lemma:LK-HSnorm-unbounded}, $\rho^m(U_2) \geq - \frac{\delta}{2}$. For $\Xbf \in U_1 \cap U_2$,
	with measure at least $1-\delta$,
	\begin{align*}
	&||C_{\Phi(\Xbf)} - C_{\Phi}||_{\HS} 
	\leq ||\mu_{\Phi(\Xbf)}\otimes \mu_{\Phi(\Xbf)} - \mu_{\Phi}\otimes \mu_{\Phi}||_{\HS}
	+ \left\|\frac{1}{m}\Phi(\Xbf)\Phi(\Xbf)^{*} - L_K\right\|_{\HS}
	\\
	& \leq \frac{2\kappa^2}{\sqrt{m}\delta}\left(3 + \frac{2}{\sqrt{m}\delta}\right).
	\end{align*}
	Since $||C_{\Phi}||_{\HS} \leq ||L_K||_{\HS} + ||\mu_{\Phi} \otimes \mu_{\Phi}||_{\HS} \leq \kappa^2 + \kappa^2 = 2 \kappa^2$,
	we then have
	\begin{align*}
	||C_{\Phi(\Xbf)}||_{\HS} \leq ||C_{\Phi(\Xbf)} - C_{\Phi}||_{\HS} + ||C_{\Phi}||_{\HS} \leq 2\kappa^2 + \frac{2\kappa^2}{\sqrt{m}\delta}\left(3 + \frac{2}{\sqrt{m}\delta}\right)
	\end{align*}
	for all $\Xbf \in U_1 \cap U_2$.
	\qed
\end{proof}

\begin{proof}
	[\textbf{Proof of Theorem \ref{theorem:Sinkhorn-RKHS-concentration-unbounded}}]
	As in the proof of Theorem \ref{theorem:Sinkhorn-RKHS-concentration},
	by Theorem \ref{theorem:convergence-Sinkhorn} and Proposition \ref{proposition:CPhi-concentration-unbounded}, for any $0 < \delta < 1$,
	\begin{align*}
	&\Srm^{\ep}_{d^2}[\Ncal(\mu_{\Phi(\Xbf)}, C_{\Phi(\Xbf)}), \Ncal(\mu_{\Phi}, C_{\Phi})]
	\\
	& \leq ||\mu_{\Phi(\Xbf)} - \mu_{\Phi}||^2_{\H_K}
	+  \frac{3}{\ep}[||C_{\Phi(\Xbf)}||_{\HS(\H_K)} + ||C_{\Phi}||_{\HS(\H_K)}]||C_{\Phi(\Xbf)}- C_{\Phi}||_{\HS(\H_K)}
	\\
	&\leq \frac{4\kappa^2}{m \delta^2} + \frac{3}{\ep}\left[4\kappa^2 + \frac{2\kappa^2}{\sqrt{m}\delta}\left(3 + \frac{2}{\sqrt{m}\delta}\right)\right]\frac{2\kappa^2}{\sqrt{m}\delta}\left(3 + \frac{2}{\sqrt{m}\delta}\right)
	\\
	& = \frac{4\kappa^2}{m \delta^2} + \frac{12 \kappa^4}{\ep \sqrt{m}\delta}\left[2 + \frac{1}{\sqrt{m}\delta}\left(3 + \frac{2}{\sqrt{m}\delta}\right)\right]\left(3 + \frac{2}{\sqrt{m}\delta}\right)
	\end{align*} 
	with probability at least $1-\delta$. 
	\qed
\end{proof}

\begin{proof}
	[\textbf{Proof of Theorem \ref{theorem:Sinkhorn-RKHS-approximation-unbounded}}]
	As in the proof of Theorem \ref{theorem:Sinkhorn-RKHS-approximation},
	by Theorem \ref{theorem:Sinkhorn-entropic-approx}, Lemma \ref{lemma:mean-norm-HK-unbounded}, and Proposition \ref{proposition:CPhi-concentration-unbounded},
	\begin{align*}
	&\left|\Srm^{\ep}_{d^2}[\Ncal(\mu_{\Phi(\Xbf)}, C_{\Phi(\Xbf)}), \Ncal(\mu_{\Phi(\Ybf)}, C_{\Phi(\Ybf)})]
	- \Srm^{\ep}_{d^2}[\Ncal(\mu_{\Phi,\rho_1}, C_{\Phi, \rho_1}), \Ncal(\mu_{\Phi,\rho_2}, C_{\Phi, \rho_2})]\right|
	\\
	& \leq \left[||\mu_{\Phi(\Xbf)}||_{\H_K} + ||\mu_{\Phi(\Ybf)}||_{\H_K} + ||\mu_{\Phi, \rho_1}||_{\H_K} + ||\mu_{\Phi, \rho_2}||_{\H_K} \right]
	\\
	& \quad \times \left[||\mu_{\Phi(\Xbf)} - \mu_{\Phi, \rho_1}||_{\H_K} + ||\mu_{\Phi(\Ybf)} - \mu_{\rho,2}||_{\H_K}\right]
	\\
	& \quad + \frac{3}{\ep}\left[||C_{\Phi(\Xbf)}||_{\HS(\H_K)} + ||C_{\Phi,\rho_1}||_{\HS(\H_K)} + 2||C_{\Phi, \rho_2}||_{\HS(\H_K)}\right]
	||C_{\Phi(\Xbf)} - C_{\Phi, \rho,1}||_{\HS(\H_K)}
	\\
	&\quad + \frac{3}{\ep}\left[2||C_{\Phi(\Xbf)}||_{\HS(\H_K)} + ||C_{\Phi(\Ybf)}||_{\HS(\H_K)} + ||C_{\Phi, \rho_2}||_{\HS(\H_K)}\right]
	||C_{\Phi(\Ybf)} - C_{\Phi, \rho_2}||_{\HS(\H_K)}.
	\end{align*}
	By Proposition \ref{proposition:CPhi-concentration-unbounded}, for any $0 < \delta < 1$, 
	let $U_1 \subset (\X, \rho_1)^m$ be such that $\forall \Xbf \in U_1$,
	\begin{align*}
	||\mu_{\Phi(\Xbf)} - \mu_{\Phi}||_{\H_K} &\leq \frac{4\kappa}{\sqrt{m}\delta}, \; ||\mu_{\Phi(\Xbf)}||  \leq \kappa\left(1+ \frac{4}{\sqrt{m}\delta}\right),
	\\
	||C_{\Phi(\Xbf)} - C_{\Phi}||_{\HS(\H_K)} &\leq 
	\frac{4\kappa^2}{\sqrt{m}\delta}\left(3 + \frac{4}{\sqrt{m}\delta}\right),
	\\
	||C_{\Phi(\Xbf)}||_{\HS(\H_K)}  &\leq 2\kappa^2 + \frac{4\kappa^2}{\sqrt{m}\delta}\left(3 + \frac{4}{\sqrt{m}\delta}\right),
	\end{align*}
	then $\rho^m(U_1) \geq 1 -\frac{\delta}{2}$.
	Similarly, let $U_2 \subset (\X, \rho_2)^n$ be such that $\forall \Ybf \in U_2$, 
	\begin{align*}
	||\mu_{\Phi(\Ybf)} - \mu_{\Phi}||_{\H_K} &\leq \frac{4\kappa}{\sqrt{n}\delta}, \; ||\mu_{\Phi(\Ybf)}||  \leq \kappa\left(1+ \frac{4}{\sqrt{n}\delta}\right),
	\\
	||C_{\Phi(\Ybf)} - C_{\Phi}||_{\HS(\H_K)} &\leq 
	\frac{4\kappa^2}{\sqrt{n}\delta}\left(3 + \frac{4}{\sqrt{n}\delta}\right),
	\\
	||C_{\Phi(\Ybf)}||_{\HS(\H_K)}  &\leq 2\kappa^2 + \frac{4\kappa^2}{\sqrt{n}\delta}\left(3 + \frac{4}{\sqrt{n}\delta}\right),
	\end{align*}
	then $\rho^n(U_2) \geq 1-\frac{\delta}{2}$. Let now $Z = (U_1 \times (\X,\rho_2)^n) \cap ((\X, \rho_1)^m \times U_2)
	\subset (\X,\rho_1)^m \times (\X,\rho_2)^n$, then $(\rho_1^m \times \rho_2^n)(Z) \geq 1 -\delta$.
	Then $\forall (\Xbf,\Ybf) \in Z$, we have
	\begin{align*}
	&\left|\Srm^{\ep}_{d^2}[\Ncal(\mu_{\Phi(\Xbf)}, C_{\Phi(\Xbf)}), \Ncal(\mu_{\Phi(\Ybf)}, C_{\Phi(\Ybf)})]
	- \Srm^{\ep}_{d^2}[\Ncal(\mu_{\Phi,\rho_1}, C_{\Phi, \rho_1}), \Ncal(\mu_{\Phi,\rho_2}, C_{\Phi, \rho_2})]\right|
	\\
	& \leq \left[2\kappa + \kappa\left(1+ \frac{4}{\sqrt{m}\delta}\right) +\kappa\left(1+ \frac{4}{\sqrt{n}\delta}\right)\right]\left[\frac{4\kappa}{\sqrt{m}\delta}+\frac{4\kappa}{\sqrt{n}\delta}\right]
	\\
	& \quad + \frac{3}{\ep}\left[2\kappa^2 + \frac{4\kappa^2}{\sqrt{m}\delta}\left(3 + \frac{4}{\sqrt{m}\delta}\right) + 6\kappa^2\right]\frac{4\kappa^2}{\sqrt{m}\delta}\left(3 + \frac{4}{\sqrt{m}\delta}\right)
	\\
	&\quad + \frac{3}{\ep}\left[4\kappa^2+\frac{8\kappa^2}{\sqrt{m}\delta}\left(3 + \frac{4}{\sqrt{m}\delta}\right) +2\kappa^2 + \frac{4\kappa^2}{\sqrt{n}\delta}\left(3 + \frac{4}{\sqrt{n}\delta}\right) + 2\kappa^2\right]\frac{4\kappa^2}{\sqrt{n}\delta}\left(3 + \frac{4}{\sqrt{n}\delta}\right)
	\\
	& \leq \frac{16\kappa^2}{\delta}\left(1 + \frac{1}{\sqrt{m}\delta} + \frac{1}{\sqrt{n}\delta}\right)
	\left(\frac{1}{\sqrt{m}} + \frac{1}{\sqrt{n}}\right)
	\\
	& \quad + \frac{48\kappa^4}{\ep\sqrt{m}\delta}\left[2 + \frac{1}{\sqrt{m}\delta}\left(3 + \frac{4}{\sqrt{m}\delta}\right)\right]\left(3 + \frac{4}{\sqrt{m}\delta}\right)
	\\
	&\quad + \frac{48\kappa^4}{\ep\sqrt{n}\delta}\left[\frac{2}{\sqrt{m}\delta}\left(3 + \frac{4}{\sqrt{m}\delta}\right) + \frac{1}{\sqrt{n}\delta}\left(3 + \frac{4}{\sqrt{n}\delta}\right)+2\right]\left(3 + \frac{4}{\sqrt{n}\delta}\right).
	\end{align*}
	\qed
\end{proof}

\begin{proof}
	[\textbf{Proof of Proposition \ref{proposition:2Wasserstein-gaussian-upperbound}}]
	(i) Consider first the case $\H= \R^d$, $d \in \Nbb$.
	By the Araki-Lieb-Thirring inequality \cite{Wang1995trace},
	for any pair $A,B \in \Sym^{+}(d)$,
	\begin{align}
	\trace[(A^{1/2}BA^{1/2})^{1/2}] \geq \trace(A^{1/2}B^{1/2}),
	\end{align}
	with equality if and only if $AB = BA$.
	Thus it follows that
	\begin{align*}
	&\W_2(\Ncal(0, C_1), \Ncal(0, C_2)) = \trace(C_1) + \trace(C_2) - 2\trace[(C_1^{1/2}C_2C_1^{1/2})^{1/2}]
	\\
	&\leq \trace(C_1) + \trace(C_2) - 2 \trace(C_1^{1/2}C_2^{1/2}) = ||C_1^{1/2}-C_2^{1/2}||^2_{\HS}
	\\
	& \leq ||C_1 - C_2||_{\tr} \;\;\;\text{by Lemma \ref{lemma:norm-trace-HS-square-root}}
	\\
	& \leq \sqrt{d}||C_1-C_2||_{\HS}.
	\end{align*}
	(ii) Consider now the general separable Hilbert space $\H$.
	Let $\{e_k\}_{k \in \Nbb}$ be any orthonormal basis in $\H$. For $N \in \Nbb$, let $\H_N = \myspan\{e_k\}_{k=1}^N$.
	Let $P_N = \sum_{k=1}^Ne_k \otimes e_k$ be the orthogonal projection operator onto $\H_N$.
	Let $C_{i,N} = P_NC_iP_N$, $i=1,2$, then $C_{i,N}: \H_N \mapto \H_N$ and $C_{i,N}|_{\H_N^{\perp}} = 0$.
	Let $\Cbf_{i,N}$
	be the matrix representation of the operator $C_{i,N}|_{\H_N}$
	on the $N$-dimensional subspace $\H_N$ in the basis $\{e_k\}_{k=1}^N$. Then
	\begin{align*}
	\W_2^2[\Ncal(0, C_{1,N}), \Ncal(0,C_{2,N})] &= \trace(C_{1,N}) + \trace(C_{2,N}) - \trace[(C_{1,N}^{1/2}C_{2,N}C_{1,N}^{1/2})^{1/2}]
	\\
	& = \trace(\Cbf_{1,N}) + \trace(\Cbf_{2,N}) - \trace[(\Cbf_{1,N}^{1/2}\Cbf_{2,N} \Cbf_{1,N}^{1/2})^{1/2}]
	\\
	& \leq ||\Cbf_{1,N}^{1/2} - \Cbf_{2,N}^{1/2}||_{\HS}^2 \leq ||\Cbf_{1,N} - \Cbf_{2,N}||_{\tr}
	\end{align*}
	by part (i).
	It thus follows that
	\begin{align*}
	\W_2^2[\Ncal(0, C_{1,N}), \Ncal(0,C_{2,N})]  \leq ||C_{1,N}^{1/2}-C_{2,N}^{1/2}||^2_{\HS} \leq ||C_{1,N} - C_{2,N}||_{\tr}.
	\end{align*}
	Since $\lim_{N \approach \infty}||C_{i,_N}-C_i||_{\tr} = \lim_{N \approach 0}||C_{i,N}^{1/2} - C_i^{1/2}||_{\HS} = 0$
	and the Wasserstein distance is continuous in trace norm, letting $N \approach \infty$ on both sides gives
	\begin{align*}
	\W_2^2[\Ncal(0, C_1), \Ncal(0,C_2)] \leq ||C_1^{1/2}-C_2^{1/2}||^2_{\HS} \leq ||C_1-C_2||_{\tr}.
	\end{align*}
	\qed
\end{proof}

\begin{proof}
	[\textbf{of Theorem \ref{theorem:samplebound-Wasserstein-Gaussian-finiteRKHS}}]
	Since $\dim(\H_K) < \infty$, by Propositions \ref{proposition:2Wasserstein-gaussian-upperbound}
	and \ref{proposition:CPhi-concentration-unbounded},
	\begin{align*}
	&\W^2_2[\Ncal(\mu_{\Phi(\Xbf)}, C_{\Phi(\Xbf)}), \Ncal(\mu_{\Phi}, C_{\Phi})]
	\\
	& \leq ||\mu_{\Phi(\Xbf)} - \mu_{\Phi}||^2_{\H_K} + \sqrt{\dim(\H_K)}||C_{\Phi(\Xbf)} - C_{\Phi}||_{\HS}
	\\
	& \leq \frac{4\kappa^2}{m\delta^2} + \frac{2\kappa^2 \sqrt{\dim(\H_K)}}{\sqrt{m}\delta}\left(3 + \frac{2}{\sqrt{m}\delta}\right),
	\end{align*} 
	with probability at least $1-\delta$, for any $0 < \delta < 1$.
	\qed
\end{proof}

The proofs for Lemmas \ref{lemma:integral-Gaussian-norm-4-meanzero} and \ref{lemma:integral-Gaussian-norm-3-meanzero} below are included for completeness.

\begin{lemma}
	\label{lemma:integral-Gaussian-norm-4-meanzero}
	For the Gaussian measure $\Ncal(0,C)$ on $\H$,
	\begin{align}
	\int_{\H}||x||^4d\Ncal(0,C)(x) = 2||C||^2_{\HS} + (\trace{C})^2.
	\end{align}
\end{lemma}
\begin{proof}
	Let $\{\lambda_j\}_{j\in \Nbb}$ be the eigenvalues of $C$, with corresponding orthonormal eigenvectors
	$\{e_j\}_{j \in \Nbb}$. Write $x_j = \la x, e_j\ra$, $j \in \Nbb$, then
	$x = \sum_{j=1}^{\infty}x_je_j$. By Lebesgue Monotone Convergence Theorem,
	\begin{align*}
	&\int_{\H}||x||^4d\Ncal(0,C)(x) =\int_{\H}(\sum_{j=1}^{\infty}x_j^2)^2d\Ncal(0,C)(x)
	\\
	& = \int_{\H}\sum_{j=1}^{\infty}x_j^4d\Ncal(0,C)(x) + 2\int_{\H}\sum_{j \neq k}x_j^2x_k^2d\Ncal(0,C)(x) 
	\\
	& = \sum_{j=1}^{\infty}\int_{\R}x_j^4d\Ncal(0,\lambda_j)(x) + 2\sum_{j\neq k}\int_{\R}x_j^2d\Ncal(0,\lambda_j)(x)\int_{\R}x_k^2d\Ncal(0,\lambda_k)(x)
	\\
	& = 3\sum_{j=1}^{\infty}\lambda_j^2 + 2\sum_{j\neq k}\lambda_j\lambda_k = 2\sum_{j=1}^{\infty}\lambda_j^2 + (\sum_{j=1}^{\infty}\lambda_j)^2 = 2||C||^2_{\HS} + (\trace{C})^2.
	\end{align*}
	Here we have used the formulas $\int_{\R}t^2d\Ncal(0, \lambda)(t) = \lambda$ and 
	$\int_{\R}t^4d\Ncal(0,\lambda)(t) = 3\lambda^2$, see e.g. (\cite{Handbook:1972}, Formula 7.4.4).
	\qed
\end{proof}

\begin{lemma}
	\label{lemma:integral-Gaussian-norm-3-meanzero}
	For the Gaussian measure $\Ncal(0,C)$ on $\H$,
	\begin{align}
	\int_{\H}||x||^2\la x, \mu\ra d\Ncal(0,C)(x) = 0, \;\;\;\forall \mu \in \H.
	\end{align}
\end{lemma}
\begin{proof}
	We proceed as in Lemma \ref{lemma:integral-Gaussian-norm-4-meanzero}. Write
	$\mu = \sum_{j=1}^{\infty}\mu_je_j$,$x = \sum_{j=1}^{\infty}x_je_j$, then
	\begin{align*}
	&\int_{\H}||x||^2\la x,\mu\ra d\Ncal(0,C)(x) = \int_{\H}(\sum_{j=1}^{\infty}x_j^2)(\sum_{k=1}^{\infty}\mu_k x_k)d\Ncal(0,C)(x) = 0
	\end{align*}
	by symmetry, since each term in the integral is of either form $x_j^2x_k$, $j\neq k$, or $x_j^3$ $\forall j,k \in \Nbb$.
	\qed
\end{proof}

\begin{proof}
	[\textbf{Proof of Lemma \ref{lemma:Gaussian-integral-norm-4}}]
	We apply the formula $\int_{\H}||x||^2d\Ncal(0,C)(x) = \trace(C)$, along with Lemmas 
	\ref{lemma:integral-Gaussian-norm-4-meanzero} and \ref{lemma:integral-Gaussian-norm-3-meanzero} to obtain
	\begin{align*}
	&\int_{\H}||x||^4d\Ncal(\mu,C)(x) = \int_{\H}||(x-\mu) + \mu||^4d\Ncal(\mu, C)(x)
	\\
	& = \int_{\H}||x+\mu||^4d\Ncal(0,C)(x)= \int_{\H}[||x||^2 + 2 \la x, \mu\ra + ||\mu||^2]^2d\Ncal(0, C)(x)
	\\
	& = \int_{\H}[||x||^4 + 4\la x, \mu\ra^2 + ||\mu||^4 + 4||x||^2\la x, \mu\ra + 4\la x, \mu\ra ||\mu||^2 + 2 ||x||^2||\mu||^2]d\Ncal(0, C)(x)
	\\
	& = \int_{\H}[||x||^4 +4||x||^2\la x, \mu\ra] d\Ncal(0,C)(x) + 4\la \mu, C\mu\ra + ||\mu||^4  + 2 ||\mu||^2\trace(C)
	\\
	& = 2||C||^2_{\HS}  + (\trace{C})^2 + 4\la \mu, C\mu\ra + ||\mu||^4 + 2 ||\mu||^2\trace(C)
	\\
	& = 2||C||^2_{\HS} + 4\la \mu, C\mu\ra + (\trace{C} +||\mu||^2)^2.
	\end{align*}
	\qed
\end{proof}

\bibliographystyle{plain}
\bibliography{bib_infinite}
\end{document}